\documentclass[accepted]{uai2023} 

\usepackage[american]{babel}
\usepackage{subcaption}
\usepackage{natbib} 
\usepackage{mathtools} 
\usepackage{booktabs} 
\usepackage{tikz} 

\usepackage{graphicx}
\usepackage{amsmath}
\usepackage{amssymb}
\usepackage{multirow}
\usepackage{amsthm}
\usepackage{flexisym}
\usepackage{wrapfig}

\DeclareMathOperator*{\argmin}{argmin}
\newtheorem{theorem}{Theorem}[section]

\newtheorem{lemma}[theorem]{Lemma}

\newtheorem{definition}[theorem]{Definition}

\newcommand{\minisection}[1]{\noindent{\textbf{#1}}}
\usepackage{xr}

\title{Simple Transferability Estimation for Regression Tasks}

\author{\href{mailto:<cnguy049@cs.fiu.edu>?Subject=Your UAI 2023 paper}{Cuong N.~Nguyen}$^1$ \qquad Phong Tran$^{2,3}$ \qquad Lam Si Tung Ho$^4$ \qquad Vu Dinh$^5$ {\vskip 0.1cm} Anh T.~Tran$^2$ \qquad Tal Hassner$^6$ \qquad Cuong V.~Nguyen}
\affil{Florida International University, USA \qquad $^2$VinAI Research, Vietnam \qquad $^3$MBZUAI, UAE {\vskip 0.1cm} $^4$Dalhousie University, Canada \qquad
$^5$University of Delaware, USA \qquad $^6$Meta AI, USA
}

\begin{document}
\maketitle

\begin{abstract}
We consider transferability estimation, the problem of estimating how well deep learning models transfer from a source to a target task. We focus on regression tasks, which received little previous attention, and propose two simple and computationally efficient approaches that estimate transferability based on the negative regularized mean squared error of a linear regression model. We prove novel theoretical results connecting our approaches to the actual transferability of the optimal target models obtained from the transfer learning process. Despite their simplicity, our approaches significantly outperform existing state-of-the-art regression transferability estimators in both accuracy and efficiency. On two large-scale keypoint regression benchmarks, our approaches yield 12\% to 36\% better results on average while being at least 27\% faster than previous state-of-the-art methods.
\end{abstract}

\section{Introduction}

Transferability estimation~\citep{bao2019information, tran2019transferability, nguyen2020leep} aims to develop computationally efficient metrics to predict the effectiveness of transferring a deep learning model from a source to a target task. This problem has recently gained attention as a means for model and task selection~\citep{bao2019information, tran2019transferability, nguyen2020leep, bolya2021scalable, you2021logme} that can potentially improve the performance and reduce the cost of transfer learning, especially for expensive deep learning models. In recent years, new transferability estimators were also developed and used in applications such as checkpoint ranking~\citep{huang2021exploiting, li2021ranking} and few-shot learning~\citep{tong2021mathematical}.

Nearly all existing methods consider only the transferability between classification tasks~\citep{bao2019information, tran2019transferability, nguyen2020leep, deshpande2021linearized, li2021ranking, tan2021otce, huang2022frustratingly}, with very few designed for regression~\citep{you2021logme, huang2022frustratingly}, despite the importance of regression problems in a wide range of applications such as landmark detection~\citep{fard2021asmnet, poster2021visible}, object detection and localization~\citep{cai2020deep, bu2021gaia}, pose estimation~\citep{schwarz2015rgb, doersch2019sim2real}, or image generation~\citep{ramesh2021zero, razavi2019generating}. Moreover, those few methods are often a byproduct of a classification transferability estimator and were never tested against regression transferability estimation baselines.

In this paper, we explicitly consider transferability estimation for regression tasks and formulate a novel definition for this problem. Our formulation is based on the practical usage of transferability estimation: to compare the actual transferability between different tasks~\citep{bao2019information, tran2019transferability, nguyen2020leep, you2021logme}. We then propose two \emph{simple, efficient, and theoretically grounded} approaches for this problem that estimate transferability using the negative regularized mean squared error (MSE) of a linear regression model computed from the source and target training sets. The first approach, \emph{Linear MSE}, uses the linear regression model between features extracted from the source model (a model trained on the source task) and true labels of the target training set. The second approach, \emph{Label MSE}, estimates transferability by regressing between the \emph{dummy} labels, obtained from the source model, and true labels of the target data. In special cases where the source and target data share the inputs, the Label MSE estimators can be computed even more efficiently from the true labels without a source model.

In addition to their simplicity, we show our transferability estimators to have theoretical properties relating them to the actual transferability of the transferred target model. In particular, we prove that the transferability of the target model obtained from transfer learning is lower bounded by the Label MSE minus a complexity term, which depends on the target dataset size and the model architecture. Similar theoretical results can also be proven for the case where the source and target tasks share the inputs.

We conduct extensive experiments on two real-world keypoint detection datasets, CUB-200-2011~\citep{wah2011caltech} and OpenMonkey~\citep{yao2021openmonkeychallenge}, as well as the dSprites shape regression dataset~\citep{matthey2017dsprites} to show the advantages of our approaches. The results clearly demonstrate that despite their simplicity, our approaches outperform recently published, state-of-the-art (SotA) regression transferability estimators, such as LogME~\citep{you2021logme} and TransRate~\citep{huang2022frustratingly}, in both effectiveness and efficiency. In particular, our approaches can improve SotA results from 12\% to 36\% on average, while being at least 27\% faster.

\minisection{Summary of contributions.} (1) We formulate a new definition for the transferability estimation problem that can be used for comparing the actual transferability (\S\ref{settings}). (2) We propose Linear MSE and Label MSE, two simple yet effective  transferability estimators for regression tasks (\S\ref{sec:estimators}). (3) We prove novel theoretical results for these estimators to connect them with the actual task transferability (\S\ref{sec:theory}). (4) We rigorously test our approaches in various settings and challenging benchmarks, showing their advantages compared to SotA regression transferability methods (\S\ref{sec:experiment}).\footnote{Implementations of our methods are available at: \url{https://github.com/CuongNN218/regression_transferability}.}

\section{Related work}

Our paper is one of the recent attempts to develop efficient and effective transferability estimators for deep transfer learning~\citep{bao2019information, tran2019transferability, nguyen2020leep, deshpande2021linearized, li2021ranking, tan2021otce, you2021logme, huang2022frustratingly, nguyen2022generalization}, which is closely related to the generalization estimation problem~\citep{chuang2020estimating, deng2021labels}. Most of the existing work for transferability estimation focuses on classification~\citep{bao2019information, tran2019transferability, nguyen2020leep, deshpande2021linearized, li2021ranking, tan2021otce, nguyen2022generalization}, while we are only aware of two methods developed for regression~\citep{you2021logme, huang2022frustratingly}.

One regression transferability method, called LogME~\citep{you2021logme}, takes a Bayesian approach and uses the maximum log evidence of the target data as the transferability estimator. While this method can be sped up using matrix decomposition, its scalability is still limited since the required memory is large. In contrast, our proposed approaches are simpler, faster, and more effective. We also provide novel theoretical properties for our methods that were not available for LogME. Another approach for transferability estimation between regression tasks, called TransRate~\citep{huang2022frustratingly}, is to divide the real-valued outputs into different bins and apply a classification transferability estimator. In our experiments, we will show that this approach is less accurate than both LogME and our approaches.

Transferability can also be inferred from a task taxonomy~\citep{zamir2018taskonomy, dwivedi2019representation, dwivedi2020duality} or a task space representation~\citep{achille2019task2vec}, which embeds tasks as vectors on a vector space. A popular task taxonomy, Taskonomy~\citep{zamir2018taskonomy}, exploits the underlying structure of visual tasks by computing a task affinity matrix that can be used for estimating transferability. Constructing the Taskonomy requires training a small classification head, which resembles the training of the regularized linear regression models in our approaches. However, they investigate the global taxonomy of classification tasks, while our paper studies regression tasks with a focus on estimating their transferability efficiently.

Our paper is also related to transfer learning with kernel methods~\citep{radhakrishnan2022transfer} and with deep models~\citep{tan2018survey}, which has been successful in real-world regression problems such as object detection and localization~\citep{cai2020deep, bu2021gaia}, landmark detection~\citep{fard2021asmnet, poster2021visible}, or pose estimation~\citep{schwarz2015rgb, doersch2019sim2real}. Several previous works have investigated theoretical bounds for transfer learning~\citep{ben2003exploiting, blitzer2007learning, mansour2009domain, azizzadenesheli2018regularized, wang2019transfer, tripuraneni2020theory}; however, these bounds are hard to compute in practice and thus unsuitable for transferability estimation. Some previous transferability estimators have theoretical bounds on the empirical loss of the transferred model~\citep{tran2019transferability, nguyen2020leep}, but these bounds were for classification and did not relate directly to transferability. Our bounds, on the other hand, focus on regression and connect our approaches directly to the notion of transferability.

\section{Transferability between regression tasks}
\label{settings}

In this section, we describe the transfer learning setting that will be used in our subsequent analysis. We then propose a definition of transferability for regression tasks and a new formulation for the transferability estimation problem.

\subsection{Transfer learning for regression}
\label{sec:transfer_learning}

Consider a source training set $\mathcal{D}_s = \{ (x^s_i, y^s_i) \}_{i=1}^{n_s}$ and a target training set $\mathcal{D}_t = \{ (x^t_i, y^t_i) \}_{i=1}^{n_t}$ consisting of $n_s$ and $n_t$ examples respectively, where $x^s_i, x^t_i \in \mathbb{R}^d$ are $d$-dimensional input vectors, $y^s_i \in \mathbb{R}^{d_s}$ is a $d_s$-dimensional source label vector, and $y^t_i \in \mathbb{R}^{d_t}$ is a $d_t$-dimensional target label vector. Here we allow multi-output regression tasks (with $d_s, d_t \ge 1$) where the source and target labels may have different dimensions ($d_s \neq d_t$). In the simplest case, the source and target tasks are both single-output regression tasks where $d_s = d_t = 1$.

In this paper, we will refer to a model (such as $w$, $w^*$, $h$, $h^*$, $k$, or $k^*$) and its parameters interchangeably. Using the source dataset $\mathcal{D}_s$, we train a deep learning model $(w^*, h^*)$ consisting of an optimal feature extractor $w^*$ and an optimal regression head $h^*$ that minimizes the empirical MSE loss:\footnote{Here we assume $(w^*, h^*)$ is a global minimum of Eq.~\eqref{eq:source}. However, practical optimization algorithms often only return a local minimum for this problem. The same is also true for Eq.~\eqref{eq:target}.}
\begin{equation}
\textstyle w^*, h^* = \argmin_{w, h} \mathcal{L} (w, h; \mathcal{D}_s),
\label{eq:source}
\end{equation}
where $w : \mathbb{R}^d \rightarrow \mathbb{R}^{d_r}$ is a feature extractor network that transforms a $d$-dimensional input vector into a $d_r$-dimensional feature vector, $h : \mathbb{R}^{d_r} \rightarrow \mathbb{R}^{d_s}$ is a source regression head network that transforms a $d_r$-dimensional feature vector into a $d_s$-dimensional output vector, and $\mathcal{L} (w, h; \mathcal{D}_s)$ is the empirical MSE loss of the whole model $(w, h)$ on the dataset $\mathcal{D}_s$:
\begin{equation}
\mathcal{L} (w, h; \mathcal{D}_s) = \frac{1}{n_s} \sum_{i=1}^{n_s} \| y^s_i - h(w(x^s_i)) \|^2,
\end{equation}
with $\| \cdot \|$ being the $\ell_2$ norm. In practice, we usually consider a source model (e.g.,~a ResNet~\citep{he2016deep}) as a whole and use its first $l$ layers from the input (for some chosen number $l$) as the feature extractor $w$. The regression head $h$ is the remaining part of the model from the $l$-th layer to the output layer, and the prediction for any input $x$ is $h(w(x))$.

After training the optimal source model $(w^*, h^*)$, we perform transfer learning to the target task by freezing the optimal feature extractor $w^*$ and re-training a new regression head $k^*$ using the target dataset $\mathcal{D}_t$, also by minimizing the empirical MSE loss:
\begin{align}
    k^* &= \textstyle \argmin_k \mathcal{L} (w^*, k; \mathcal{D}_t) \notag\\
    &= {\textstyle \argmin_k} \Big\{ \frac{1}{n_t} \sum_{i=1}^{n_t} \| y^t_i - k(w^*(x^t_i)) \|^2 \Big\}
    \label{eq:target},
\end{align}
where $k : \mathbb{R}^{d_r} \rightarrow \mathbb{R}^{d_t}$ is a target regression head network that may have a different architecture than that of $h$. In general, the regression heads $h$ and $k$ may contain multiple layers and are not necessarily linear. 

This transfer learning algorithm, usually called \emph{head re-training}, has been widely used for deep learning models~\citep{donahue2014decaf, oquab2014learning, sharif2014cnn, whatmough2019fixynn} and will be used for our theoretical analysis. In practice and in our experiments, we also consider another transfer learning algorithm, widely known as \emph{fine-tuning}, where we fine-tune the trained feature extractor $w^*$ on the target set, and then train a new target regression head $k^*$ with this fine-tuned feature extractor~\citep{agrawal2014analyzing, girshick2014rich, chatfield2014return, dhillon2019baseline}.

\subsection{Transferability estimation}

As our first contribution, we propose a definition of transferability for regression tasks and a new formulation for the transferability estimation problem. For this purpose, we make the standard assumption that the target data $\mathcal{D}_t$ are drawn iid from the true but unknown distribution $\mathbb{P}_t := \mathbb{P}(X^t, Y^t)$; that is, $(x^t_i, y^t_i) \stackrel{\mathrm{iid}}{\sim} \mathbb{P}_t$. We do not make any assumption on the distribution of the source data $\mathcal{D}_s$, but we assume a source model $(w^*, h^*)$ is pre-trained on $\mathcal{D}_s$ and then transferred to a target model $(w^*, k^*)$ using the procedure in Section~\ref{sec:transfer_learning}. 

We now define the transferability between the source dataset $\mathcal{D}_s$ and the target task represented by $\mathbb{P}_t$. In our Definition~\ref{def:transferability} below, the transferability is the expected negative $\ell_2$ loss of the target model $(w^*, k^*)$ on a random example drawn from $\mathbb{P}_t$. From this definition, the lower the loss of $(w^*, k^*)$, the higher the transferability.

\begin{definition}
The \textbf{transferability} between a source dataset $\mathcal{D}_s$ and a target task $\mathbb{P}_t$ is defined as: $\mathrm{Tr}(\mathcal{D}_s, \mathbb{P}_t) := \mathbb{E}_{(x^t, y^t) \sim \mathbb{P}_t} \left \{ - \|y^t - k^*(w^*(x^t)) \|^2 \right \}$.
\label{def:transferability}
\end{definition}

In the above definition, transferability is also equivalent to the negative expected (true) risk of $(w^*, k^*)$. Next, we formulate the transferability estimation problem. Previous work~\citep{tran2019transferability, huang2022frustratingly} defined this problem as estimating $\mathrm{Tr}(\mathcal{D}_s, \mathbb{P}_t)$ from the training sets $(\mathcal{D}_s, \mathcal{D}_t)$, i.e., to derive a real-valued metric $\mathcal{T}(\mathcal{D}_s, \mathcal{D}_t) \in \mathbb{R}$ such that ${ \mathcal{T}(\mathcal{D}_s, \mathcal{D}_t) \approx \mathrm{Tr}(\mathcal{D}_s, \mathbb{P}_t) }$. However, in most applications of transferability estimation such as task selection~\citep{tran2019transferability, huang2022frustratingly, you2021logme} or model ranking~\citep{huang2021exploiting, li2021ranking}, an accurate approximation of $\mathrm{Tr}(\mathcal{D}_s, \mathbb{P}_t)$ is usually not required since $\mathcal{T}(\mathcal{D}_s, \mathcal{D}_t)$ is only used for comparing tasks or models. Thus, we propose below an \emph{alternative definition} for this problem that better aligns with its practical usage.

\begin{definition}
\textbf{Transferability estimation} aims to find a computationally efficient real-valued metric ${ \mathcal{T}(\mathcal{D}_s, \mathcal{D}_t) \in \mathbb{R} }$ for any pair of training datasets $(\mathcal{D}_s, \mathcal{D}_t)$ such that:
$\mathcal{T}(\mathcal{D}_s, \mathcal{D}_t) \le \mathcal{T}(\mathcal{D}'_s, \mathcal{D}'_t)$ if and only if $\mathrm{Tr}(\mathcal{D}_s, \mathbb{P}_t) \leq \mathrm{Tr}(\mathcal{D}'_s, \mathbb{P}'_t)$, where $\mathbb{P}_t$ and $\mathbb{P}'_t$ are the tasks corresponding with the datasets $\mathcal{D}_t$ and $\mathcal{D}'_t$ respectively.
\label{def:trans_est}
\end{definition}

In our new definition, a transferability estimator $\mathcal{T}(\mathcal{D}_s, \mathcal{D}_t)$ is a function of $(\mathcal{D}_s, \mathcal{D}_t)$ that can be used for comparing or ranking transferability. It does \emph{not} need to be an approximation of $\mathrm{Tr}(\mathcal{D}_s, \mathbb{P}_t)$. This is a generalization of previous definitions~\citep{nguyen2020leep, huang2022frustratingly} and can be used for source task selection (when $\mathbb{P}_t = \mathbb{P}'_t$ and ${\mathcal{D}_t = \mathcal{D}'_t}$) as well as target task selection (when $\mathcal{D}_s = \mathcal{D}'_s$). It is consistent with the usage of transferability estimators and the way they are evaluated in the literature by correlation analysis~\citep{tran2019transferability, nguyen2020leep, you2021logme, huang2022frustratingly}.

\section{Simple transferability estimators for regression}
\label{sec:estimators}

In theory, we can use $-\mathcal{L} (w^*, k^*; \mathcal{D}_t)$, the negative MSE of the transferred target model $(w^*, k^*)$, as a transferability estimator, since it is an empirical estimation of $\mathrm{Tr}(\mathcal{D}_s, \mathbb{P}_t)$ using the dataset $\mathcal{D}_t$. However, this method requires us to run the actual transfer learning process, which could be expensive if the network architecture of the target regression heads (e.g., $k$ and $k^*$) is deep and complex. This violates a crucial requirement for a transferability estimator in Definition~\ref{def:trans_est}: the estimator must be \emph{computationally efficient} since it will be computed several times for task comparison. In this section, we propose two simple regression transferability estimators to address this problem.

\subsection{Linear MSE estimator}

To reduce the cost of computing $\mathcal{L} (w^*, k^*; \mathcal{D}_t)$, a simple idea is to approximate it with an $\ell_2$-regularized linear regression (Ridge regression) head. This leads to our first simple transferability estimator, Linear MSE, which is defined as the negative regularized MSE of this Ridge regression head. In this definition, $\| \cdot \|_F$ is the Frobenius norm.

\begin{definition}
The \textbf{Linear MSE} transferability estimator with a regularization parameter $\lambda \ge 0$ is: $\mathcal{T}^{\mathrm{lin}}_{\lambda}(\mathcal{D}_s, \mathcal{D}_t) := - \min_{A, b} \big\{ \frac{1}{n_t} \sum_{i=1}^{n_t} {\| y^t_i - A w^*(x^t_i) - b \|^2} + \lambda \|A \|_F^2 \big\}$, where $A \in \mathbb{R}^{d_r \times d_t}$ is a $d_r \times d_t$ real-valued matrix and $b \in \mathbb{R}^{d_t}$ is a $d_t$-dimensional real-valued vector.
\label{def:linmse}
\end{definition}

Here we add a regularizer to avoid overfitting when the target dataset $\mathcal{D}_t$ is small. Previous work such as LogME~\citep{you2021logme} proposed to prevent overfitting by taking a Bayesian approach, which is more complicated and expensive. We will show empirically in our experiments (Section~\ref{sec:exp_small}) that our simple regularization approach can tackle the issue more effectively and efficiently.

Given a pre-trained feature extractor $w^*$ and a target set $\mathcal{D}_t$, we can compute $\mathcal{T}^{\mathrm{lin}}_{\lambda}(\mathcal{D}_s, \mathcal{D}_t)$ efficiently using the closed form solution for Ridge regression or using second-order optimization~\citep{bishop2006pattern}. If the target regression head $k^*$ is a linear regression model, $\mathcal{T}^{\mathrm{lin}}_0(\mathcal{D}_s, \mathcal{D}_t)$ with $\lambda = 0$ is the negative MSE of the transferred target model $(w^*, k^*)$ on $\mathcal{D}_t$. If $k^*$ has more than one layer with a non-linear activation, $\mathcal{T}^{\mathrm{lin}}_{\lambda}(\mathcal{D}_s, \mathcal{D}_t)$ can be regarded as using a regularized linear model to approximate this non-linear head.

\subsection{Label MSE estimator}
\label{sec:labelmse}

Although the Linear MSE transferability score above can be computed efficiently, this computation may still be relatively expensive if the feature vectors $w^*(x^t_i)$ are high-dimensional. To further reduce the costs, we propose another transferability estimator, Label MSE, which replaces $w^*(x^t_i)$ by the ``dummy'' source label $z_i = h^*(w^*(x^t_i))$. Using dummy labels from the pre-trained source model $(w^*, h^*)$ is a technique previously used to compute the LEEP transferability score for classification~\citep{nguyen2020leep}. We define our Label MSE estimator below.

\begin{definition}
The \textbf{Label MSE} transferability estimator with a regularization parameter $\lambda \ge 0$ is: $\mathcal{T}^{\mathrm{lab}}_{\lambda}(\mathcal{D}_s, \mathcal{D}_t) := - \min_{A, b} \big\{ \frac{1}{n_t} \sum_{i=1}^{n_t} \| y^t_i - A z_i - b \|^2 + \lambda \|A \|_F^2 \big\}$, where ${ A \in \mathbb{R}^{d_s \times d_t} }$ is a $d_s \times d_t$ real-valued matrix, $b \in \mathbb{R}^{d_t}$ is a $d_t$-dimensional real-valued vector, and $z_i = h^*(w^*(x^t_i))$.
\label{def:labmse}
\end{definition}

In practice, since the size of $z_i$ is usually much smaller than that of $w^*(x^t_i)$ (i.e., $d_s \ll d_r$), computing the Label MSE is usually faster than computing the Linear MSE.

\minisection{$\bullet$ Special case with shared inputs.}
When the source and target datasets have the same inputs, i.e., ${ \mathcal{D}_s = \{ (x_i, y^s_i) \}_{i=1}^n }$ and $\mathcal{D}_t = \{ (x_i, y^t_i) \}_{i=1}^n$, we can compute the Label MSE even faster using only the true labels. Particularly, we can consider the following version of the Label MSE.

\begin{definition}
The \textbf{Shared Inputs Label MSE} transferability estimator with a regularization parameter $\lambda \ge 0$ is: $\widehat{\mathcal{T}}^{\mathrm{lab}}_{\lambda}(\mathcal{D}_s, \mathcal{D}_t) := - \min_{A, b} \Big\{ { \frac{1}{n} \sum_{i=1}^{n} \| y^t_i - A y^s_i - b \|^2 } +  \lambda \|A \|_F^2 \Big\}$, where $A \in \mathbb{R}^{d_s \times d_t}$ and $b \in \mathbb{R}^{d_t}$.
\label{def:shared_labmse}
\end{definition}

In this definition, the Shared Inputs Label MSE is computed by training a Ridge regression model directly from the true label pairs $(y^s_i, y^t_i)$, which is \emph{less expensive} than the original Label MSE since we do not need to train the source model $(w^*, h^*)$ or compute the dummy labels.

Intuitively, our estimators use a weaker version of the actual target model that helps trade off the estimators’ accuracy for computational speed. Our estimators can also be viewed as instances of the kernel Ridge regression approach~\citep{smale2007learning, hastie2009elements}. While the Linear MSE can be interpreted as a linear approximation to $- \mathcal{L} (w^*, k^*; \mathcal{D}_t)$, properties of the Label MSE and Shared Inputs Label MSE are not well understood. In the next section, we shall prove novel theoretical properties for these estimators.

\section{Theoretical properties}
\label{sec:theory}

We now prove some theoretical properties for the Label MSE with ReLU feed-forward neural networks. These properties are in the form of generalization bounds relating $\mathcal{T}^{\mathrm{lab}}_{\lambda}(\mathcal{D}_s, \mathcal{D}_t)$ with the transferability $\mathrm{Tr}(\mathcal{D}_s, \mathbb{P}_t)$. Throughout this section, we assume the space of all target regression heads $k$, which may have more than one layer, is a superset of all the linear regression models. This assumption is generally true for ReLU networks~\citep{arora2018understanding}.

First, we show in Lemma~\ref{lemma:empirical} below a relationship between the negative MSE loss $- \mathcal{L} (w^*, k^*; \mathcal{D}_t)$ of $(w^*, k^*)$ and the Label MSE. This lemma states that the negative MSE loss $- \mathcal{L} (w^*, k^*; \mathcal{D}_t)$ upper bounds the Label MSE. The proof for this lemma is in the Appendix~\ref{proof:lemma:empirical}.

\begin{lemma}
For any $\lambda \ge 0$, we have:
$\mathcal{T}^{\mathrm{lab}}_{\lambda}(\mathcal{D}_s, \mathcal{D}_t) \le - \mathcal{L} (w^*, k^*; \mathcal{D}_t)$.
\label{lemma:empirical}
\end{lemma}

Using this lemma, we can prove our main theoretical result in Theorem~\ref{thm:generalization} below. In this theorem, $L$ is the number of layers of the ReLU feed-forward neural network $(w^*, k^*)$, and we assume the number of hidden nodes and parameters in each layer are upper bounded by $H$ and $M \ge 1$ respectively. Without loss of generality, we also assume all input and output data are upper bounded by $1$ in $\ell_\infty$-norm. This assumption can easily be satisfied by a pre-processing step that scales them to $[0, 1]$ in $\ell_\infty$-norm.

\begin{theorem}
For any source dataset $\mathcal{D}_s$, $\lambda \ge 0$ and $\delta > 0$, with probability at least $1 - \delta$ over the randomness of $\mathcal{D}_t$, we have:
$\mathrm{Tr}(\mathcal{D}_s, \mathbb{P}_t) \ge \mathcal{T}^{\mathrm{lab}}_{\lambda}(\mathcal{D}_s, \mathcal{D}_t) - C(d, d_t, M, H, L, \delta)/\sqrt{n_t}$, where
$C(d, d_t, M, H, L, \delta) = 16 M^{2L+2} H^{2L} [ d_t^2 d \sqrt{L+1+ \ln d} + d_t d^2 \sqrt{2 \ln(4/\delta)} ]$.
\label{thm:generalization}
\end{theorem}

The proof for this theorem is in the Appendix~\ref{proof:thm:generalization} The theorem shows that the transferability $\mathrm{Tr}(\mathcal{D}_s, \mathbb{P}_t)$ is lower bounded by the Label MSE $\mathcal{T}^{\mathrm{lab}}_{\lambda}(\mathcal{D}_s, \mathcal{D}_t)$ minus a complexity term $C(d, d_t, M, H, L, \delta)/\sqrt{n_t}$ that depends on the target dataset (specifically, the input and output dimensions, as well as the dataset size) and the architecture of the target network. When this complexity term is small (e.g., when $n_t$ is large enough), the bound in Theorem~\ref{thm:generalization} will be tighter. In this case, a higher Label MSE score will likely lead to better transferability.

\minisection{$\bullet$ Shared inputs case.}
We can also derive similar bounds for the Shared Inputs Label MSE $\widehat{\mathcal{T}}^{\mathrm{lab}}_{\lambda}(\mathcal{D}_s, \mathcal{D}_t)$.
Denote
${ A^*_\lambda, b^*_\lambda := \argmin_{A, b} \big\{ \frac{1}{n} \sum_i \| y^t_i - A y^s_i - b \|^2 + \lambda \|A \|_F^2 \big\} }$.
We first show the following lemma relating $\widehat{\mathcal{T}}^{\mathrm{lab}}_{\lambda}(\mathcal{D}_s, \mathcal{D}_t)$ and the losses of the source and target models.

\begin{lemma}
For any $\lambda \geq 0$, we have:
$
\widehat{\mathcal{T}}^{\mathrm{lab}}_{\lambda}(\mathcal{D}_s, \mathcal{D}_t) \leq - \mathcal{L} (w^*, k^*; \mathcal{D}_t)/2 + \|A^*_\lambda\|_F^2 \mathcal{L} (w^*, h^*; \mathcal{D}_s).
$
\label{lemma:empirical_same_input}
\end{lemma}

Using this lemma, we can prove the following theorem for this shared inputs setting. The proofs for these results are in the Appendix~\ref{proof:lemma:empirical_same_input}.

\begin{theorem}
For any source dataset $\mathcal{D}_s$, $\lambda \geq 0$ and $\delta > 0$, with probability at least $1 - \delta$ over the randomness of $\mathcal{D}_t$, we have:
$
\mathrm{Tr}(\mathcal{D}_s, \mathbb{P}_t) \geq 2 \widehat{\mathcal{T}}^{\mathrm{lab}}_{\lambda}(\mathcal{D}_s, \mathcal{D}_t) - 2 \|A^*_\lambda\|_F^2 \mathcal{L} (w^*, h^*; \mathcal{D}_s) - C(d, d_t, M, H, L, \delta)/\sqrt{n}  
$.
\label{thm:generalization_same_input}
\end{theorem}

From the theorem, $\widehat{\mathcal{T}}^{\mathrm{lab}}_{\lambda}(\mathcal{D}_s, \mathcal{D}_t)$ can indirectly tell us information about the transferability $\mathrm{Tr}(\mathcal{D}_s, \mathbb{P}_t)$ without actually training $w^*$, $h^*$, and $k^*$. This bound becomes tighter when $n$ is large or $\mathcal{L} (w^*, h^*; \mathcal{D}_s)$ is small (e.g., when the source model is expressive enough to fit the source data).
An experiment to investigate the usefulness of our theoretical bounds in this section is available in the Appendix~\ref{proof:thm:generalization_same_input}.

\section{Experiments}
\label{sec:experiment}

In this section, we conduct experiments to evaluate our approaches on the keypoint (or landmark) regression tasks using the following two large-scale public datasets:

$\bullet$ \textbf{CUB-200-2011}~\citep{wah2011caltech}. This dataset contains 11,788 bird images with 15 labeled keypoints indicating 15 different parts of a bird body. We use 9,788 images for training and 2,000 images for testing. Since the annotations for occluded keypoints are highly inaccurate, we remove all occluded keypoints during the training for both source and target tasks.

$\bullet$ \textbf{OpenMonkey}~\citep{yao2021openmonkeychallenge}. This is a benchmark for the non-human pose tracking problem. It offers over 100,000 monkey images in natural contexts, annotated with 17 body landmarks. We use the original train-test split, which contains 66,917 training images and 22,306 testing images.

In our experiments, we use ResNet34~\citep{he2016deep} as the backbone since it provides good performance as a source model. Following previous work~\citep{tran2019transferability, nguyen2020leep, huang2022frustratingly, nguyen2022generalization}, we investigate how well our transferability estimators correlate (using Pearson correlation) with the \emph{negative test MSE} of the target model obtained from actual transfer learning. This correlation analysis is a good method to measure how well transferability estimators satisfy our Definition~\ref{def:trans_est}. In the Table~\ref{tab:kendall} and~\ref{tab:spearman} in the Appendix~\ref{appendix:high_dim_exp_1}, we provide additional results for other non-linear correlation measures, including Kendall’s $\tau$ and Spearman correlations. The conclusions in our paper remain the same when comparing these correlations.

We consider three standard transfer learning algorithms: (1) \textbf{head re-training}~\citep{donahue2014decaf, sharif2014cnn}: We fix all layers of the source model up until the penultimate layer and re-train the last fully-connected (FC) layer using the target training set; (2) \textbf{half fine-tuning}~\citep{donahue2014decaf, sharif2014cnn}: We fine-tune the last convolutional block and all the FC layers of the source model, while keeping all other layers fixed; and (3) \textbf{full fine-tuning}~\citep{agrawal2014analyzing, girshick2014rich}: We fine-tune the whole source model using the target training set. Among these settings, head re-training resembles the transfer scenario in Section~\ref{sec:transfer_learning}, while half and full fine-tuning are more commonly used in practice. For half fine-tuning, around half of the parameters in the network will be fine-tuned ($\sim$13M parameters). More details of our experiment settings are in the Appendix~\ref{appendix:experiment_settings_1}.

We compare our transferability estimators, Linear MSE and Label MSE, with two recent SotA baselines for regression: LogME~\citep{you2021logme} and TransRate~\citep{huang2022frustratingly}. For our methods, we consider $\lambda = 0$ (named \textbf{LinMSE0} and \textbf{LabMSE0}) for the estimators without regularization, and $\lambda = 1$ (named \textbf{LinMSE1} and \textbf{LabMSE1}) for the estimators with the default $\lambda$ value. The effects of $\lambda$ on our algorithms are investigated in Section~\ref{sec:lambda_exp}. 

For the baselines, besides the usual versions (\textbf{LogME} and \textbf{TransRate}) that are computed from the extracted features and the target labels, we also consider the versions where they are computed from the dummy labels and the target labels (named \textbf{LabLogME} and \textbf{LabTransRate}). As in previous work~\citep{huang2022frustratingly}, we divide the target label values into equal-sized bins (five bins in our case) to compute TransRate and LabTransRate.

\subsection{General transfer between two different domains}
\label{exp:different_input}

\begin{table*}[t]
\caption{{\bf Correlation coefficients when transferring from OpenMonkey to CUB-200-2011}. Bold numbers indicate best results in each row. Asterisks (*) indicate best results among the corresponding label-based or feature-based methods. Detailed correlation plots are in the Appendix~\ref{appendix:high_dim_exp_1}. Our estimators improve up to 25.9\% in comparison with SotA (LogME) while being 12.9\% better on average.}
{\vskip -0.2cm}
\centering
\resizebox{\textwidth}{!}{%
\begin{tabular}{ccccccccc}
\toprule
\multirow{2}{*}{Transfer setting} & \multicolumn{4}{c}{Label-based method} & \multicolumn{4}{c}{Feature-based method} \\
\cmidrule(lr){2-5} \cmidrule(lr){6-9}
& LabLogME & LabTransRate & LabMSE0 & LabMSE1 & LogME & TransRate & LinMSE0 & LinMSE1 \\
\midrule
Head re-training & 0.824 & 0.165 & 0.991 & \textbf{0.995}* & 0.969 & 0.121 & 0.982 & \textbf{0.995*}\\
Half fine-tuning & 0.706 & 0.392 & 0.881 &  \textbf{0.885}* & 0.870 & 0.304 & 0.866 & \textbf{0.885*}\\
Full fine-tuning & 0.691 & 0.410 & ~~{\bf 0.870}* & 0.869~~ & 0.861 & 0.311 & 0.855 & 0.869* \\
\bottomrule
\end{tabular}
}
\label{tab:different_input}
\end{table*}

This experiment considers the general case where source models are trained on one dataset (OpenMonkey) and then transferred to another (CUB-200-2011). Specifically, we train a source model for each of the 17 keypoints of the OpenMonkey dataset and transfer them to each of the 15 keypoints of the CUB-200-2011 dataset, resulting in a total of 255 final models. Since each keypoint consists of x and y positions, all source and target tasks in this experiment have two dimensional labels. The actual MSEs of these models are computed on the respective test sets and then used to calculate the Pearson correlation coefficients with the transferability estimators. In this experiment, LabMSE0, LabMSE1, LabLogME, and LabTransRate are computed from the dummy source labels and the actual target labels.

Results for this experiment are in Table~\ref{tab:different_input}. In this setting, TransRate and LabTransRate perform poorly, while our methods are equal or better than LogME and LabLogME in most cases, especially when using $\lambda=1$ (LinMSE1) or dummy labels (LabMSE0 and LabMSE1). The results show our approaches improve up to 25.9\% in comparison with SotA (LogME) while being 12.9\% better on average.

It is interesting to observe that LabMSE0 and LabMSE1 provide competitive or even better correlations than LinMSE0 and LinMSE1 in this experiment. This shows that the dummy labels (i.e., body parts of monkeys) can provide as much information about the target labels (i.e., body parts of birds) as the extracted features.

In the Appendix~\ref{appendix:high_dim_exp_1}, we also report additional results where both source and target tasks have 10-dimensional labels (i.e., each task predicts 5 keypoints simultaneously). We also achieve better correlations than the baselines in this case.

\subsection{Transfer with shared-inputs tasks}
\label{sec:exp_shared_inputs}

\begin{table*}[t]
\caption{{\bf Correlation coefficients when transferring between tasks with shared inputs}. Bold numbers indicate best results in each row. Asterisks (*) indicate best results among the corresponding label-based or feature-based methods. Detailed correlation plots are in the Appendix~\ref{appendix:high_dim_exp_2}. Our estimators improve up to 113\% in comparison with SotA (LogME) while being 36.6\% better on average.}
{\vskip -0.2cm}
\centering
\resizebox{\textwidth}{!}{%
\begin{tabular}{cccccccccc}
\toprule
\multirow{3}{*}{Dataset} & \multirow{3}{*}{Transfer setting} & \multicolumn{4}{c}{Label-based method} & \multicolumn{4}{c}{Feature-based method} \\
\cmidrule(lr){3-6} \cmidrule(lr){7-10}
& & LabLogME & LabTransRate & LabMSE0 & LabMSE1 & LogME & TransRate & LinMSE0 & LinMSE1 \\
\midrule
\multirow{3}{*}{\parbox{0.9cm}{CUB-200-2011}} & Head re-training & 0.547 & 0.019 & 0.916 & 0.946* & 0.890 & 0.029 & 0.921 & ~~{\bf 0.960}*\\
& Half fine-tuning & 0.401 & 0.006 & 0.536 & 0.565* & 0.560 & 0.064 & ~~{\bf 0.628}* & 0.619\\
& Full fine-tuning & ~~{\bf 0.128}* & 0.041 & 0.056 & 0.057~~ & 0.100 & ~~0.109* & 0.097 & 0.082\\
\midrule
\multirow{3}{*}{\parbox{0.9cm}{Open Monkey}} & Head re-training & 0.890 & 0.666 & ~~0.973* & 0.773~~ & 0.695 & 0.711 & 0.946 & ~~{\bf 0.975}*\\
& Half fine-tuning & 0.615 & 0.340 & 0.754 & 0.890* & 0.446 & 0.488 & ~~{\bf 0.899}* & 0.801\\
& Full fine-tuning & 0.569 & 0.269 & 0.705 & \textbf{0.882*}  & 0.403 & 0.439 & ~~0.859* & 0.761\\
\bottomrule
\end{tabular}
}
\label{tab:shared_input}
\end{table*}

In this experiment, we consider the setting where the source and target tasks have the same inputs (the special setting in Section~\ref{sec:labelmse}). Since images in our datasets contain multiple labels (15 keypoints for CUB-200-2011 and 17 keypoints for OpenMonkey), we can use any two different keypoints on the same dataset as source and target tasks. In total, we construct 210 source-target pairs for CUB-200-2011 and 272 pairs for OpenMonkey that all have the same source and target inputs but different labels. The labels for all tasks are also two dimensional real values.

We repeat the experiment in Section~\ref{exp:different_input} with these source-target pairs for CUB-200-2011 and OpenMonkey separately. The main difference in this experiment is that we use the \emph{true} source labels (instead of dummy labels) when computing LabLogME, LabTransRate, LabMSE0, and LabMSE1. Under this setting, the LabMSE estimators here are the Shared Inputs Label MSE estimators in Definition~\ref{def:shared_labmse}. These estimators can be computed without any source models, and thus incurring very low computational costs in this setting.

Results for these experiments are in Table~\ref{tab:shared_input}. In the results, both versions of TransRate perform poorly on CUB-200-2011, while TransRate is slightly better than LogME on OpenMonkey. In most settings, LabMSE0 and LabMSE1 both outperform LabLogME and LabTransRate, while LinMSE0 and LinMSE1 both outperform LogME and TransRate. In the setting where we transfer by full fine-tuning on the CUB-200-2011 dataset, all methods perform poorly. From these results, our approaches improve up to 113\% in comparison with SotA (LogME) while being 36.6\% better on average.

We also report in the Appendix~\ref{appendix:high_dim_exp_2} additional results for each individual source task. The results show that our methods are consistently better than LogME, LabLogME, TransRate, and LabTransRate for most source tasks on both datasets. Furthermore, our methods are also better than these baselines when transferring to higher dimensional target tasks (tasks that predict 5 keypoints simultaneously and have 10-dimensional labels). These additional results further confirm the effectiveness of our approaches.

\subsection{Evaluations on small target sets}
\label{sec:exp_small}

\begin{figure*}[t]
\centering
    \begin{subfigure}[b]{0.27\textwidth}
    \includegraphics[width=\textwidth]{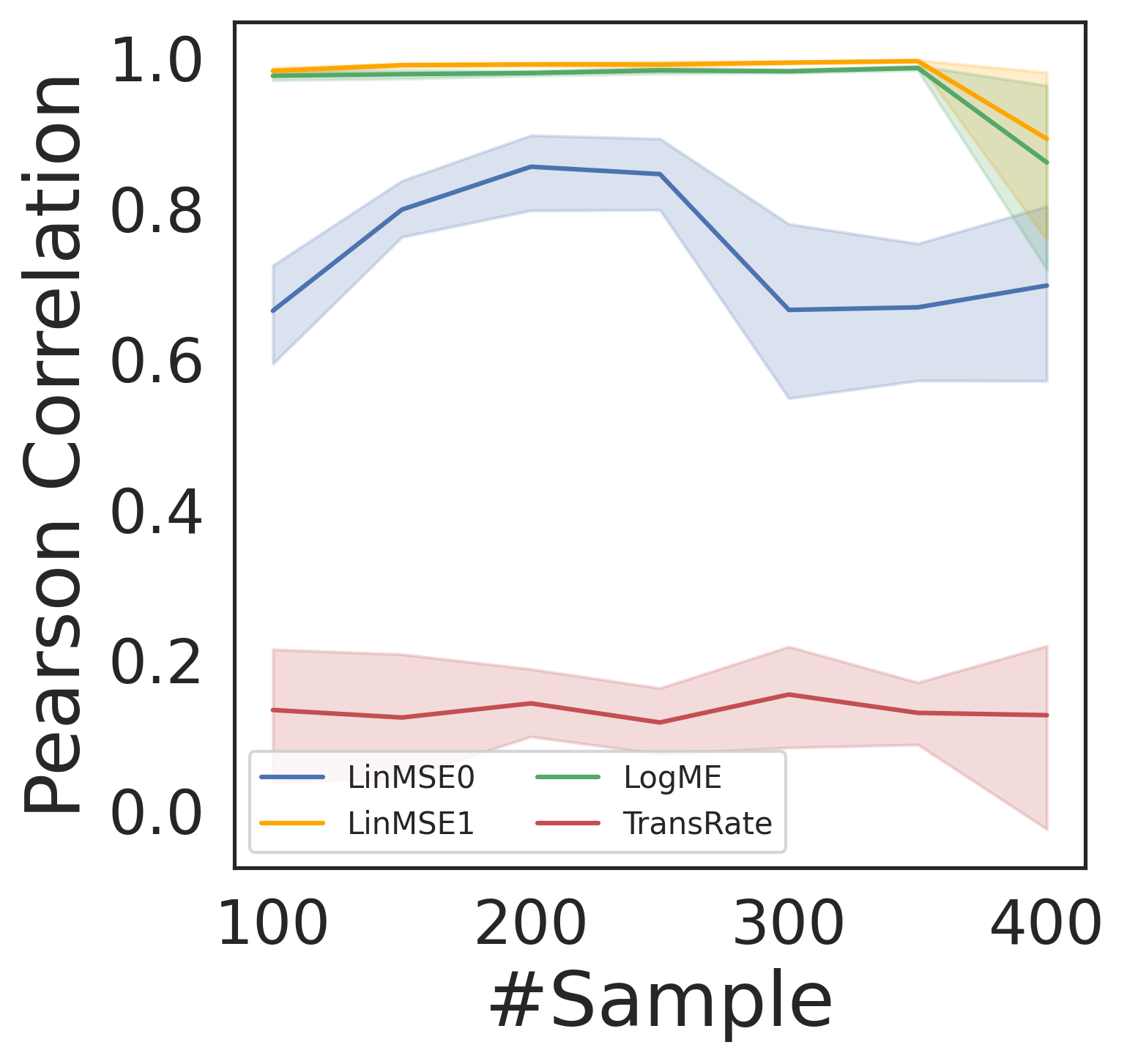}
    \end{subfigure}
    {\hskip 0.7cm}
    \begin{subfigure}[b]{0.27\textwidth}
    \includegraphics[width=\textwidth]{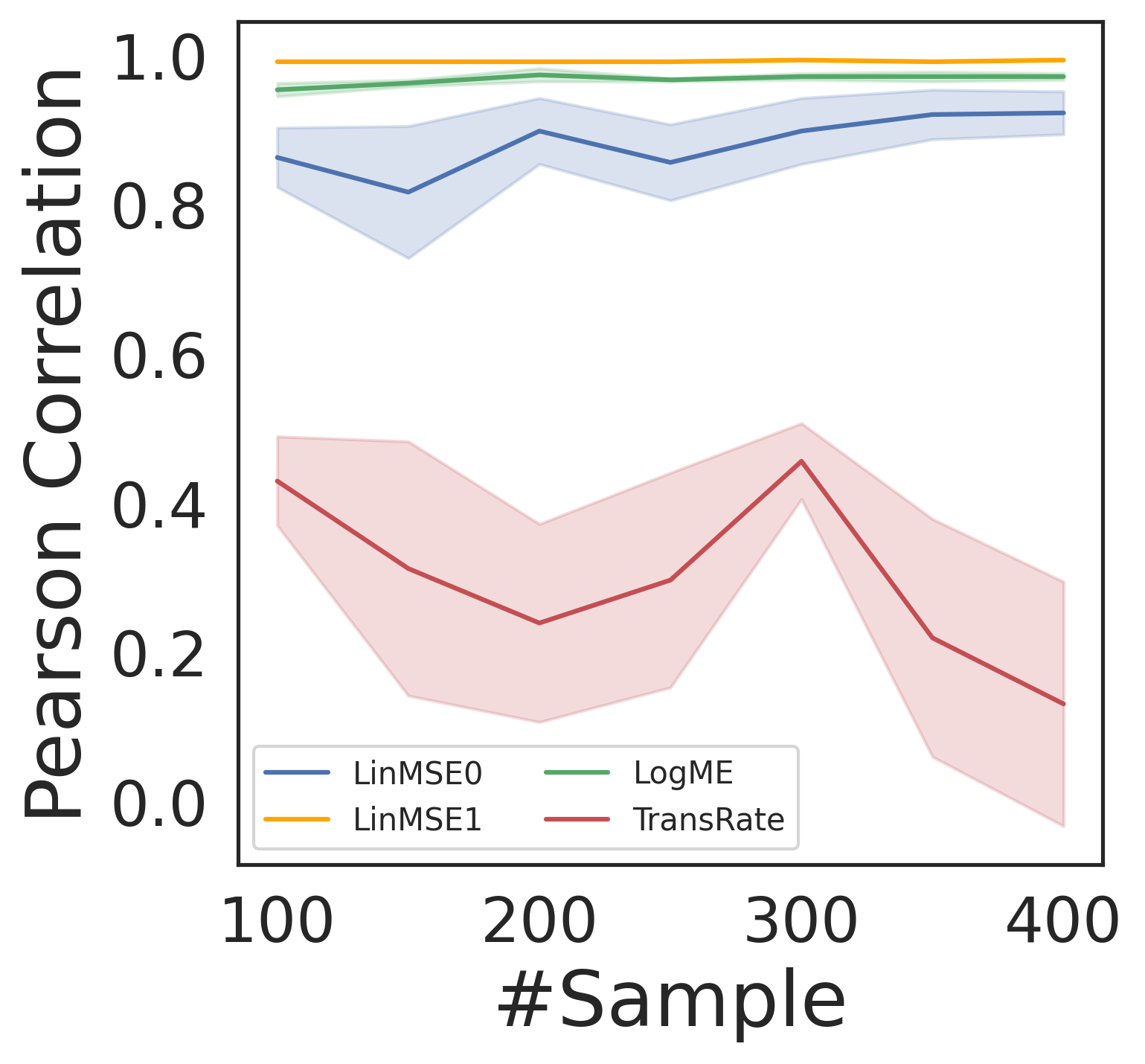}
    \end{subfigure}
    {\vskip -0.2cm}
    \caption{\textbf{Correlation coefficients with small target training sets} on CUB-200-2011 (left) and OpenMonkey (right). LinMSE1 and LogME are designed to avoid overfitting, but LinMSE1 is better than LogME in both datasets.}
    \label{fig:small_data}

    {\vskip 0.5cm}

    \begin{subfigure}[b]{0.27\textwidth}
    \includegraphics[width=\textwidth]{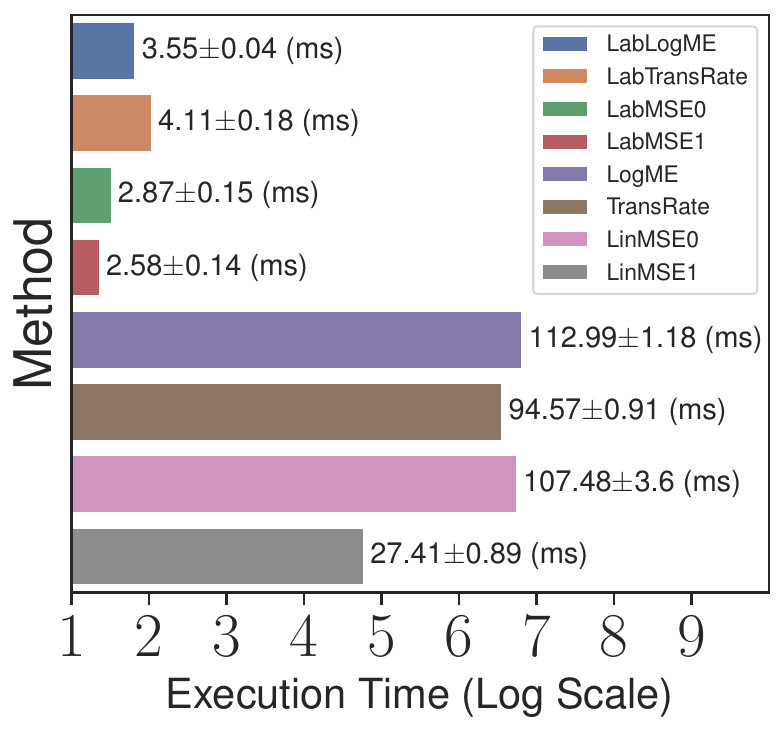}
    \end{subfigure}
    {\hskip 0.7cm}
    \begin{subfigure}[b]{0.27\textwidth}
    \includegraphics[width=\textwidth]{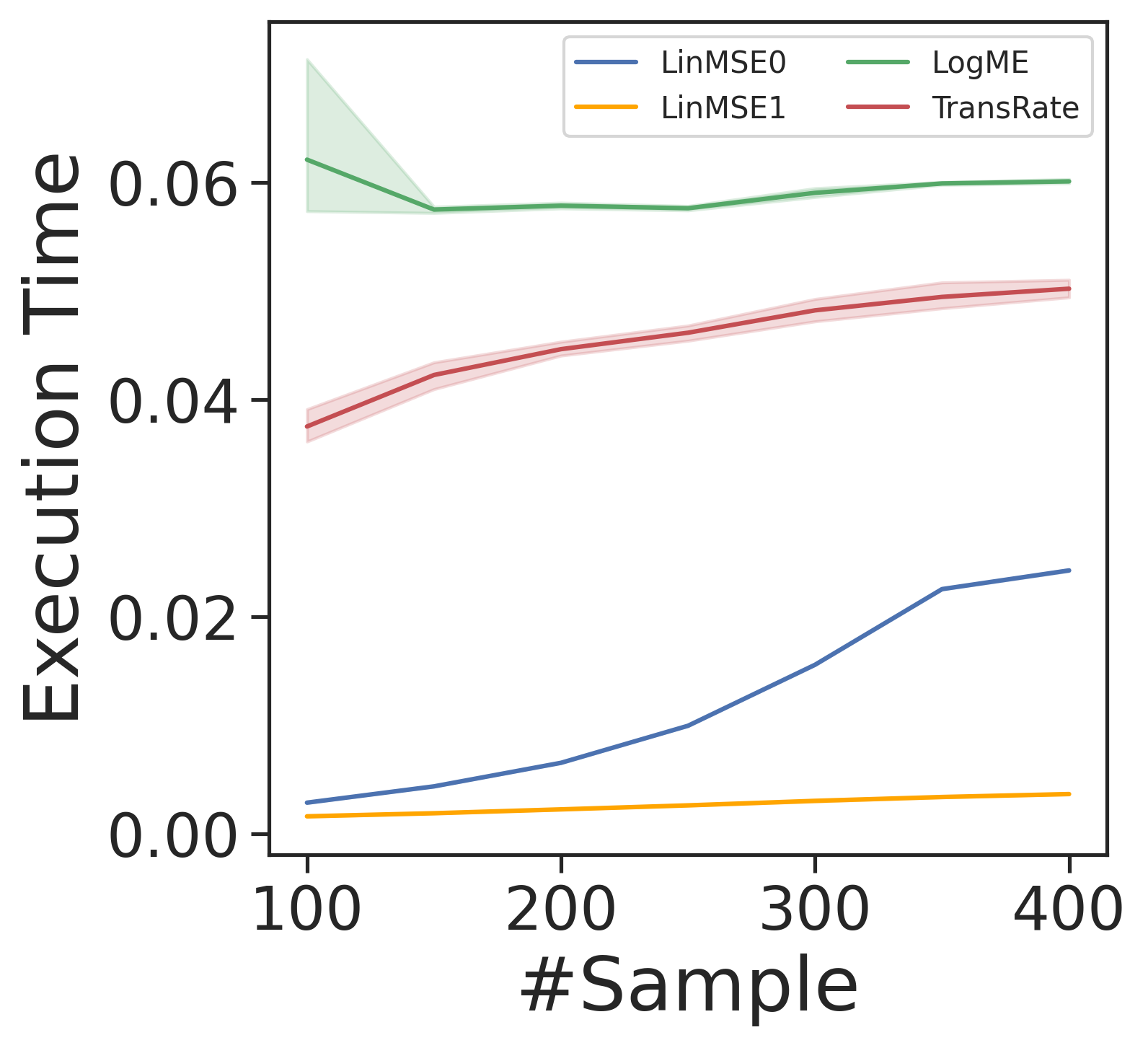}
    \end{subfigure}
    {\vskip -0.2cm}
    \caption{\textbf{Average running time} (in milliseconds) for the experiments in Sections~\ref{sec:exp_shared_inputs} (left) and~\ref{sec:exp_small} (right).}
    \label{fig:processing_time}

    {\vskip 0.5cm}
    
    \begin{subfigure}[b]{0.21\textwidth}
    \includegraphics[width=\textwidth]{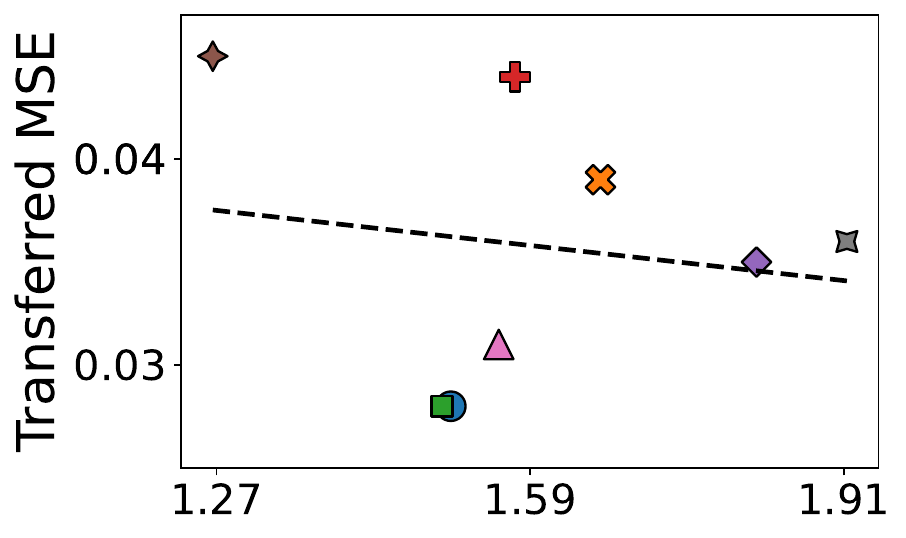}
    \caption{LogME}
    \end{subfigure}
    \begin{subfigure}[b]{0.21\textwidth}
    \includegraphics[width=\textwidth]{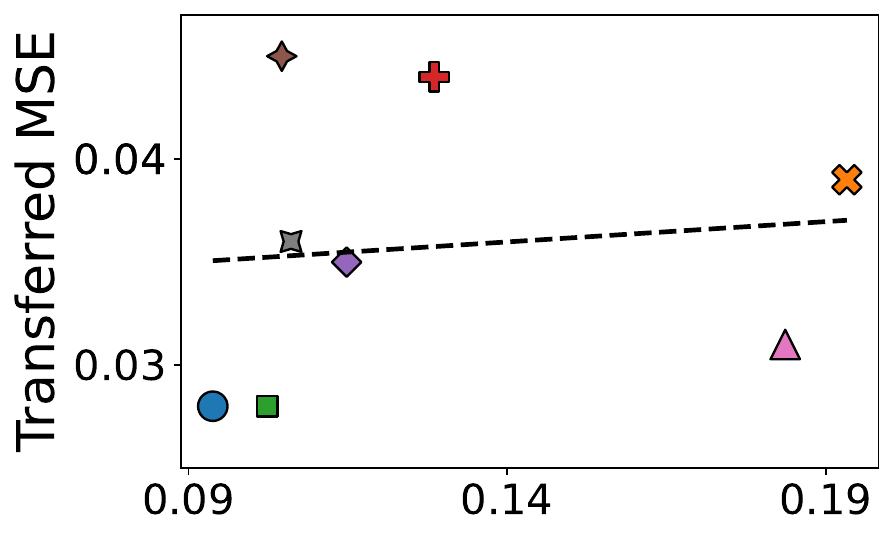}
    \caption{TransRate}
    \end{subfigure}
    \begin{subfigure}[b]{0.22\textwidth}
    \includegraphics[width=\textwidth]{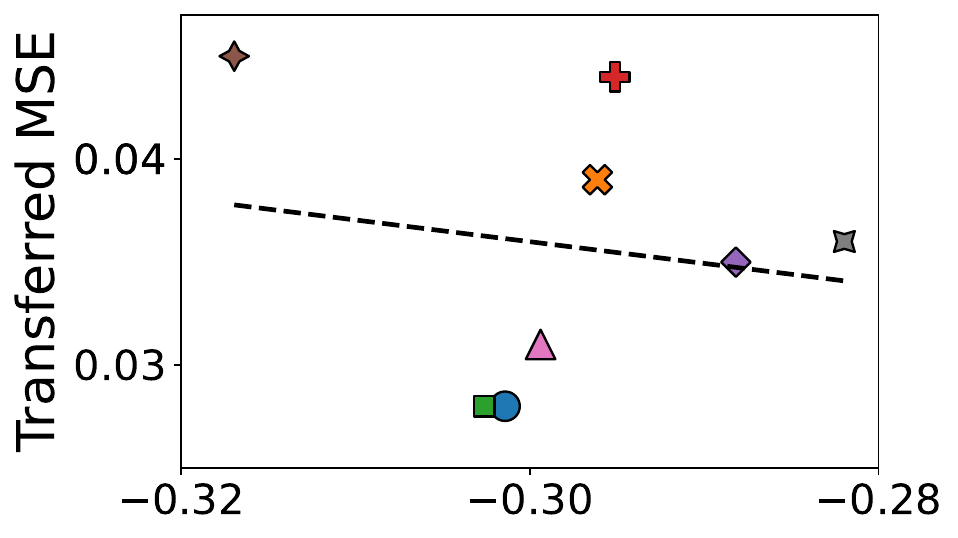}
    \caption{LinMSE0}
    \end{subfigure}
    \begin{subfigure}[b]{0.31\textwidth}
    \includegraphics[width=\textwidth]{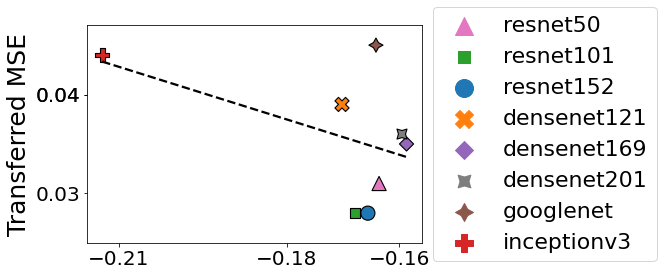}
    \caption{LinMSE1 \qquad \qquad \qquad}
    \end{subfigure}
    {\vskip -0.1cm}
    \caption{\textbf{Test MSEs vs.~transferability scores} when transferring from pre-trained classification models to a target regression task. The $x$-axis represents the transferability scores. A linear regression model (dashed line) is fitted to the points in each plot. Our methods give better fits than the baselines.}
    \label{fig:from_pretrained}
\end{figure*}

In many real-world transfer learning scenarios, the target set is usually small. This experiment will evaluate the effectiveness of the feature-based transferability estimators (LogME, TransRate, LinMSE0, and LinMSE1) in this small data regime where the number of samples is smaller than the feature dimension. For this experiment, we fix a source task (\emph{Belly} for CUB-200-2011 and \emph{Right eye} for OpenMonkey) and transfer to all other tasks in the corresponding dataset using head re-training. These source tasks are chosen since they have fewer missing labels and thus can be used to train reasonably good source models for transfer learning. For each target task, instead of using the full data, we randomly select a small subset of 100 to 400 images to perform transfer learning and to compute the transferability scores. The actual MSEs of the transferred models are still computed using the full target test sets.

Figure~\ref{fig:small_data} compares the correlations of the 4 methods on different target set sizes between 100 and 400. The results are averaged over 10 runs with 10 different random seeds. From the figure, LogME and LinMSE1 are better than TransRate and LinMSE0. This is expected since LogME and LinMSE1 are designed to avoid overfitting on small data. Both LogME and LinMSE1 are also more stable, but LinMSE1 is slightly better than LogME on all dataset sizes.

\subsection{Efficiency of our estimators}
\label{sec:efficiency}

One of the main strengths of our methods is their efficiency due to the simplicity of training the Ridge regression head. In this experiment, we first use the settings in Section~\ref{sec:exp_shared_inputs} to compare the running time of our methods with that of the baselines on the CUB-200-2011 dataset. Figure~\ref{fig:processing_time} (left) reports the results (averaged over 5 runs with different random seeds) for this experiment. From these results, our methods, LabMSE0, LabMSE1, LinMSE0, and LinMSE1, are all faster than the corresponding label-based or feature-based baselines. The figure also shows that LabMSE1 and LinMSE1 achieve the best running time among the label-based and feature-based methods respectively.

In Figure~\ref{fig:processing_time} (right), we also compare the average running time of the 4 transferability estimators using the CUB-200-2011 experiment in Section~\ref{sec:exp_small}. This figure clearly shows that our methods, LinMSE0 and LinMSE1, are more computationally efficient than LogME and TransRate. Both results in Figure~\ref{fig:processing_time} show that LinMSE1 and LabMSE1 are significantly faster than other corresponding feature-based and label-based methods. In these experiments, LinMSE1 and LabMSE1 converge faster than LinMSE0 and LabMSE0 respectively, and thus are more efficient.

\subsection{Source task selection}
\label{sec:source_selection}

Source task selection is important for applying transfer learning since the right source task can improve transfer learning performance~\citep{nguyen2020leep}. In this experiment, we examine the application of our transferability estimation methods for selecting source tasks on the CUB-200-2011 dataset. We use the head re-training setting similar to Section~\ref{sec:exp_shared_inputs}, but fix one of the tasks as the target and choose the best source task from the rest of the task pool. We repeat this process for all 15 target tasks and measure the top-$k$ matching rate of each transferability estimator. 

The top-$k$ matching rate is defined as $m_{\text{match}}/m_{\text{target}}$, where $m_{\text{target}}$ is the total number of target tasks (15 in our case), and $m_{\text{match}}$ is the number of times the selected source task gives a target model within the best $k$ models. Here the best $k$ models are determined by the actual test MSE on the target task.

Results for this experiment are in Table~\ref{tab:source_selection}. From the results, our methods are better than the baselines in terms of top-$3$ and top-$5$ matching rates. When comparing top-$1$ matching rates, our methods are competitive with LogME and LabLogME for the feature-based and label-based approaches respectively. This experiment shows that our transferability estimators are useful for source task selection.

\begin{table*}[t]
\caption{{\bf Top-$k$ matching rates for source task selection} on CUB-200-2011. Bold numbers indicate best results in each column. Asterisks (*) indicate best results among the corresponding label-based or feature-based methods.}
\centering
\small
\begin{tabular}{ccccccccc}
\toprule
\multirow{2}{*}{$k$} & \multicolumn{4}{c}{Label-based method} & \multicolumn{4}{c}{Feature-based method} \\
\cmidrule(lr){2-5} \cmidrule(lr){6-9} \cmidrule(lr){6-7}
& LabLogME & LabTransRate & LabMSE0 & LabMSE1 & LogME & TransRate & LinMSE0 & LinMSE1\\
\midrule
1 & 6/15* & 4/15  & 6/15* & 2/15 & \textbf{11/15}* & 2/15 & 9/15 & 10/15 \\
3 & 9/15  & 9/15  & 10/15* & 9/15 & 12/15 & 6/15 & 12/15 & \textbf{13/15}* \\
5 &	10/15 & 12/15 & \textbf{14/15}* & \textbf{14/15}* & 12/15 & 6/15 & 12/15 & 13/15* \\
\bottomrule
\end{tabular}
\label{tab:source_selection}
\end{table*}

{\vskip -0.2cm}
\subsection{Effects of $\lambda$}
\label{sec:lambda_exp}
{\vskip -0.2cm}

\begin{table}[t]
\setlength{\tabcolsep}{3pt}
\caption{{\bf Correlation coefficients for different values of $\lambda$} on CUB-200-2011. Bold numbers indicate best results in each column. Results of the baselines are given in the last 2 rows for comparison. When there are meaningful correlations (head re-training and half fine-tuning), our methods are better than the corresponding baselines for all $\lambda$ values.}
\resizebox{\linewidth}{!}{%
\centering
\begin{tabular}{@{\hskip -1pt}ccccccc@{\hskip -1pt}}
\toprule
\multirow{2}{*}{$\lambda$} & \multicolumn{2}{c}{Head re-training} & \multicolumn{2}{c}{Half fine-tuning} & \multicolumn{2}{c}{Full fine-tuning}\\
\cmidrule(lr){2-3} \cmidrule(lr){4-5} \cmidrule(lr){6-7}
& LabMSE & LinMSE & LabMSE & LinMSE & LabMSE & LinMSE\\
\midrule
0 &	0.916 & 0.921 & 0.536 & 0.628 & 0.056 & \textbf{0.097} \\
0.001 & 0.921 & 0.933 & 0.562 & \textbf{0.645} & 0.051 & 0.091 \\ 
0.01 & 0.922 & 0.943 & 0.560 & 0.643 & 0.048 & 0.089 \\
0.1	& 0.935	& 0.954	& 0.552	& 0.639	& 0.043	& 0.089 \\ 
0.5	& 0.945	& \textbf{0.960} & 0.562 & 0.629	& 0.053	& 0.085 \\ 
1 & \textbf{0.946} & \textbf{0.960} & 0.565 & 0.619 & 0.057	& 0.082 \\ 
2 & 0.945 & 0.958 & 0.567 & 0.607 & 0.059 & 0.077 \\ 
5 & 0.945 & 0.954 & \textbf{0.568} & 0.594 & \textbf{0.061} & 0.072 \\ 
10 & 0.945 & 0.951 & \textbf{0.568} & 0.586 & \textbf{0.061} & 0.069 \\
15 & 0.945 & 0.950 & \textbf{0.568} & 0.582 & \textbf{0.061} & 0.067 \\ 
20 & 0.945 & 0.949 & \textbf{0.568} & 0.580 & \textbf{0.061} & 0.066 \\
\midrule
(Lab)LogME & 0.547 & 0.889 & 0.400 & 0.560  & 0.120 & 0.099 \\
(Lab)TransRate & 0.008 & 0.029 & 0.006 & 0.006 & 0.001 & 0.100 \\
\bottomrule
\end{tabular}
}
\label{tab:all_lambdas}
\end{table}

In this experiment, we investigate the effects of $\lambda$ on our proposed transferability estimators. We use the setting in Section~\ref{sec:exp_shared_inputs} with the CUB-200-2011 dataset and vary the value of $\lambda$ in [0, 20] for both LabMSE and LinMSE. Table~\ref{tab:all_lambdas} reports the results for all three transfer learning settings.

For head re-training, we observe that the best correlations are achieved at $\lambda = 1$ for both LabMSE and LinMSE. For half fine-tuning, $\lambda \ge 5$ gives the best result for LabMSE, while $\lambda = 0.001$ gives the best result for LinMSE. For full fine-tuning, we do not observe significant correlations for both transferability estimators.

Notably, from the results in Table~\ref{tab:all_lambdas} for the head re-training and half fine-tuning settings (where we have significant correlations for at least one transferability estimator), LabMSE with any tested $\lambda$ value in [0, 20] is better than LabLogME and LabTransRate, while LinMSE with any tested $\lambda$ value in this range is better than LogME and TransRate. These results show that our methods are better than the baselines for a wide range of $\lambda$ values.

\subsection{Beyond regression}
\label{sec:beyond_regression}

Although our paper mainly focuses on regression tasks, the main idea of using the negative regularized MSE of a Ridge regression model for transferability estimation goes beyond regression. In principle, this idea can be applied for transferring between classification tasks (in this case, we should train a linear classifier and use its regularized log-likelihood as the transferability estimator) or between a classification and a regression task.

In this section, we demonstrate that our idea can be applied for transferability estimation between a classification and a regression task. Particularly, we use 8 source models pre-trained on ImageNet~\citep{deng2009imagenet} and transfer to a target regression task on the \emph{dSprite} dataset~\citep{matthey2017dsprites} using full fine-tuning. This setting is similar to~\citet{you2021logme} where the target is a regression task with 4-dimensional labels: x and y positions, scale, and orientation. We compute the transferability scores from the extracted features and the labels of the target training set. More details about this experiment are in the Appendix~\ref{appendix:experiment_settings_2}.

From the results in Figure~\ref{fig:from_pretrained}, the trends for LogME, LinMSE0, and LinMSE1 are correct (i.e., transferability scores have negative correlations with actual MSEs), while that of TransRate is incorrect. Note that there is a discrepancy between the ranges of the transferability and the transferred MSE because of two reasons: (1) The transferability estimators are computed from the target training set, while the transferred MSEs are computed from the target test set, and (2) there is a mismatch between the source task (ImageNet classification) and the target task (dSprite shape regression). 

To compare the transferability estimation methods, we fit a linear regression to the points in each plot and compute its RMSE to these points, where we obtain: $6.12 \times 10^{-3}$ (LogME), $6.16 \times 10^{-3}$ (TransRate), $6.10 \times 10^{-3}$ (LinMSE0), and $\textbf{5.46} \times 10^{-3}$ (LinMSE1). These results show that LinMSE0 and LinMSE1 are better than LogME and TransRate.

\section{Conclusion}

We formulated transferability estimation for regression tasks and proposed the Linear MSE and Label MSE estimators, two simple but effective approaches for this problem. We proved novel theoretical results for these estimators, showing their relationship with the actual task transferability. Our extensive experiments demonstrated that the proposed approaches are superior to recent, relevant SotA methods in terms of efficiency and effectiveness. Our proposed ideas can also be extended to mixed cases where one of the tasks is a classification problem.

\begin{acknowledgements}
LSTH was supported by the Canada Research Chairs program, the NSERC Discovery Grant RGPIN-2018-05447, and the NSERC Discovery Launch Supplement DGECR-2018-00181.
VD was supported by the University of Delaware Research Foundation (UDRF) Strategic Initiatives Grant, and the National Science Foundation Grant DMS-1951474.
\end{acknowledgements}

\bibliography{nguyen_247}

\newpage
\appendix
\onecolumn
\counterwithin{table}{section}
\counterwithin{figure}{section}

\title{Simple Transferability Estimation for Regression Tasks\\(Supplementary Material)}
\maketitle

\vspace{0.5cm}

The contents of this supplementary include:
\begin{enumerate}
    \item \textbf{Appendix~\ref{proof:lemma:empirical}}: Proof of Lemma~\ref{lemma:empirical} in the main paper.
    \item \textbf{Appendix~\ref{proof:thm:generalization}}:  Proof of Theorem~\ref{thm:generalization} in the main paper.
    \item \textbf{Appendix~\ref{proof:lemma:empirical_same_input}}: Proof of Lemma~\ref{lemma:empirical_same_input} in the main paper.
    \item \textbf{Appendix~\ref{proof:thm:generalization_same_input}}: Proof of Theorem~\ref{thm:generalization_same_input} in the main paper.
    \item \textbf{Appendix~\ref{appendix:experiment_settings_1}}: More details for the experiment settings in Sections~\ref{exp:different_input}--\ref{sec:lambda_exp} of the main paper.
    \item \textbf{Appendix~\ref{appendix:experiment_settings_2}}: More details for the experiment setting in Section~\ref{sec:beyond_regression} of the main paper.
    \item \textbf{Appendix~\ref{appendix:tightness_bounds}}: An additional experiment to show the usefulness of our theoretical bounds.
    \item \textbf{Appendix~\ref{appendix:high_dim_exp_1}}: Additional experiment results for Section~\ref{exp:different_input} of the main paper.
    \item \textbf{Appendix~\ref{appendix:high_dim_exp_2}}: Additional experiment results for Section~\ref{sec:exp_shared_inputs} of the main paper.
\end{enumerate}

\section{Mathematical proofs}

\subsection{Proof of Lemma~\ref{lemma:empirical}}
\label{proof:lemma:empirical}

Denote 
$\displaystyle A^*, b^* = \argmin_{A, b} \left\{ \frac{1}{n_t} \sum_{i=1}^{n_t} {\| y^t_i - A z_i - b \|^2} + \lambda \|A \|_F^2 \right\}.$

For all $k$, we have:
\begin{align*}
  \sqrt{\mathcal{L} (w^*, k^*; \mathcal{D}_t)} &\le \sqrt{\mathcal{L} (w^*, k; \mathcal{D}_t)} \tag{definition of $k^*$} \\
  &= \left[ \frac{1}{n_t} \sum_{i=1}^{n_t} \| y^t_i - k(w^*(x^t_i))\|^2 \right]^{1/2} \tag{definition of $\mathcal{L}$} \\
  &\le \left[ \frac{1}{n_t} \sum_{i=1}^{n_t} \| y^t_i - A^* z_i - b^*\|^2 \right]^{1/2} + \left[ \frac{1}{n_t} \sum_{i=1}^{n_t} \| A^* z_i + b^* - k(w^*(x^t_i))\|^2 \right]^{1/2} \tag{triangle inequality} \\
  &\le \sqrt{- \mathcal{T}^{\mathrm{lab}}_{\lambda}(\mathcal{D}_s, \mathcal{D}_t)} + \left[ \frac{1}{n_t} \sum_{i=1}^{n_t} \| A^* z_i + b^* - k(w^*(x^t_i)) \|^2 \right]^{1/2} \\
  &= \sqrt{- \mathcal{T}^{\mathrm{lab}}_{\lambda}(\mathcal{D}_s, \mathcal{D}_t)} + \left[ \frac{1}{n_t}\sum_{i=1}^{n_t} \| A^* h^*(w^*(x^t_i)) + b^* - k(w^*(x^t_i))\|^2 \right]^{1/2} \tag{definition of $z_i$}.\\
\end{align*}

By choosing $k (\cdot) = A^*h^*(\cdot) + b^*$, the second term in the above inequality becomes 0. This implies $\sqrt{\mathcal{L} (w^*, k^*; \mathcal{D}_t)} \le \sqrt{- \mathcal{T}^{\mathrm{lab}}_{\lambda}(\mathcal{D}_s, \mathcal{D}_t)}$ and thus the lemma.

\subsection{Proof of Theorem~\ref{thm:generalization}}
\label{proof:thm:generalization}

First, we need to define the notion of expected (true) risk. Given any model $(w, k)$ for the target task, the expected risk of $(w, k)$ is defined as:
\begin{equation}
  \mathcal{R}(w, k) := \mathbb{E}_{(x^t, y^t) \sim \mathbb{P}_t} \left \{ \|y^t - k(w(x^t))\|^2 \right \}.
\end{equation}

Note that $\mathrm{Tr}(\mathcal{D}_s, \mathbb{P}_t) = - \mathcal{R}(w^*, k^*)$. We prove the uniform bound in Lemma~\ref{lemma:uniform} below that can help us prove Theorem~\ref{thm:generalization}.

\begin{lemma}
For any $\delta >0$, with probability at least ${ 1-\delta }$, for all ReLU feed-forward neural network $(w, k)$ of the target task, we have:
\[ |\mathcal{R} (w, k) - \mathcal{L} (w, k; \mathcal{D}_t)| ~\le~ C(d, d_t, M, H, L, \delta)/\sqrt{n_t}. \]
\label{lemma:uniform}
\end{lemma}

\begin{proof}
We recall the definition of Rademacher complexity. Given a real-valued function class $\mathcal{G}$ and a set of data points $\mathcal{D} = \{ u_i \}_{i=1}^n$, the (empirical) Rademacher complexity $\widehat R_{\mathcal{D}}(\mathcal{G})$ is defined as:
\[
\widehat R_{\mathcal{D}}(\mathcal{G}) = \mathbb{E}_{\epsilon} \left[ \sup_{g \in \mathcal{G}} \frac{1}{n} \sum_{i=1}^n{\epsilon_i g(u_i)} \right],
\]
where $\epsilon = (\epsilon_1, \epsilon_2, \ldots, \epsilon_n)$ is a vector uniformly distributed in $\{ - 1, +1\}^n$ .

In our setting, the hypothesis space $\Phi$ is the class of $L$-layer ReLU feed-forward neural networks whose number of hidden nodes and parameters in each layer are bounded from above by $H$ and $M \ge 1$ respectively.
For all $(w, k) \in \Phi$ and $x$ such that $\|x \|_\infty \leq 1$,  we have:
\[ 
\| k(w(x)) \|_{\infty} \le d M^{L+1} H^L. 
\]
Define $f_{w, k}(x, y) = y - k(w(x))$ and note that $f_{w, k}(x, y) \in \mathbb{R}^{d_t}$. For any $j = 1, 2, \ldots, d_t$, let $[\cdot]_j$ be the projection map to the $j$-th coordinate. We consider the following real-valued function classes:
\begin{align*}
\mathcal{F} &= \{\|f_{w, k}\|^2: (w, k) \in \Phi\}, \\   
\mathcal{F}_j &= \{[f_{w, k}]_j: (w, k) \in \Phi\}, \\
\Phi_j &= \{[k(w(\cdot)]_j: (w, k) \in \Phi\},
\end{align*}
where each element of $\mathcal{F}$ or $\mathcal{F}_j$ is a function with variables $(x, y)$, and each element of $\Phi_j$ is a function with variable $x$.
Let $\mathcal{D}^x_t = \{ x^t_i \}_{i=1}^{n_t}$ be the set of target inputs.
By Theorem 2 of~\cite{golowich2018size}, for all $j = 1, 2, \ldots, d_t$, we have:
\[
\widehat R_{\mathcal{D}^x_t}(\Phi_j) \leq 2 d_t M^{L+1} H^L \sqrt{\frac{L+1+ \ln d}{n_t}}.
\]
We note that for any $i = 1, 2, \ldots, n_t$, the function $r_i(a) = (a - y_i^t)^2$ mapping from ${ a \in [- d M^{L+1} H^L, d M^{L+1} H^L]}$ to $\mathbb{R}$ is Lipschitz with constant $4 d M^{L+1} H^L$. Thus, applying the Contraction Lemma (Lemma 26.9 in~\cite{shalev2014understanding}), we obtain:
\[
\widehat R_{\mathcal{D}_t}(\mathcal{F}_j) \le 4 d M^{L+1} H^L \widehat R_{\mathcal{D}^x_t}(\Phi_j) \le 8 d d_t M^{2L+2} H^{2L} \sqrt{\frac{L+1+ \ln d}{n_t}}.
\]

Therefore,
\[
\widehat R_{\mathcal{D}_t}(\mathcal{F}) \le \sum_{j=1}^{d_t} \widehat R_{\mathcal{D}_t}(\mathcal{F}_j) \le 8 d d_t^2 M^{2L+2} H^{2L} \sqrt{\frac{L+1+ \ln d}{n_t}}.
\]

Using this inequality, the result of Lemma~\ref{lemma:uniform} follows from Theorem 26.5 in~\cite{shalev2014understanding}.
\end{proof}

To prove Theorem~\ref{thm:generalization}, we apply Lemma~\ref{lemma:empirical} in the main paper and Lemma~\ref{lemma:uniform} above for the transferred target model $(w^*, k^*)$. Thus, for any $\lambda \ge 0$ and $\delta > 0$, with probability at least $1 - \delta$, we have:
\begin{align*}
\mathcal{T}^{\mathrm{lab}}_{\lambda}(\mathcal{D}_s, \mathcal{D}_t) &\le - \mathcal{L} (w^*, k^*; \mathcal{D}_t)\\
&\le - \mathcal{R}(w^*, k^*) + C(d, d_t, M, H, L, \delta)/\sqrt{n_t}\\
&= \mathrm{Tr}(\mathcal{D}_s, \mathbb{P}_t) + C(d, d_t, M, H, L, \delta)/\sqrt{n_t}.
\end{align*}

Therefore, Theorem~\ref{thm:generalization} holds.

\subsection{Proof of Lemma~\ref{lemma:empirical_same_input}}
\label{proof:lemma:empirical_same_input}

Note that
$\displaystyle A^*_\lambda, b^*_\lambda = \argmin_{A, b} \left\{ \frac{1}{n} \sum_{i=1}^n \| y^t_i - A y^s_i - b \|^2 + \lambda \|A \|_F^2 \right\}.$

For all $k$, we have:
\begin{align*}
  \sqrt{\mathcal{L} (w^*, k^*; \mathcal{D}_t)} 
  &\le \sqrt{\mathcal{L} (w^*, k; \mathcal{D}_t)} \tag{definition of $k^*$} \\
  &= \left[ \frac{1}{n} \sum_{i=1}^n \| y^t_i - k(w^*(x_i)) \|^2 \right]^{1/2} \tag{definition of $\mathcal{L}$} \\
  &\le \left[ \frac{1}{n} \sum_{i=1}^n \| y^t_i - A^*_\lambda y^s_i - b^*_\lambda \|^2 \right]^{1/2} + \left[ \frac{1}{n}\sum_{i=1}^n \| A^*_\lambda y^s_i + b^*_\lambda - k(w^*(x_i))\|^2 \right ]^{1/2} \tag{triangle inequality} \\
  &\le \sqrt{- \widehat{\mathcal{T}}^{\mathrm{lab}}_{\lambda}(\mathcal{D}_s, \mathcal{D}_t)} + \left[ \frac{1}{n} \sum_{i=1}^n \| A^*_\lambda y^s_i + b^*_\lambda - k(w^*(x_i)) \|^2 \right ]^{1/2}. \tag{definition of $\widehat{\mathcal{T}}^{\mathrm{lab}}_{\lambda}$}
\end{align*}

Picking $k(\cdot) = A^*_\lambda h^*(\cdot) + b^*_\lambda$, this inequality becomes:
\begin{align*}
  \sqrt{\mathcal{L} (w^*, k^*; \mathcal{D}_t)} 
  &\le \sqrt{- \widehat{\mathcal{T}}^{\mathrm{lab}}_{\lambda}(\mathcal{D}_s, \mathcal{D}_t)} + \left[ \frac{1}{n} \sum_{i=1}^n \| A^*_\lambda [y^s_i - h^*(w^*(x_i))]\|^2\right]^{1/2} \\
  &\le \sqrt{- \widehat{\mathcal{T}}^{\mathrm{lab}}_{\lambda}(\mathcal{D}_s, \mathcal{D}_t)} + \| A^*_\lambda\|_F \left[ \frac{1}{n}\sum_{i=1}^n \|y^s_i - h^*(w^*(x_i))\|^2\right]^{1/2} \\
  &= \sqrt{- \widehat{\mathcal{T}}^{\mathrm{lab}}_{\lambda}(\mathcal{D}_s, \mathcal{D}_t)} + \|A^*_\lambda\|_F \sqrt{\mathcal{L} (w^*, h^*; \mathcal{D}_s)}.
\end{align*}

Note that if $a \le b + c$, then $a^2 \le 2b^2 + 2c^2$. Applying this fact to the above inequaility, we have:
\[ \mathcal{L} (w^*, k^*; \mathcal{D}_t) \le - 2 \widehat{\mathcal{T}}^{\mathrm{lab}}_{\lambda}(\mathcal{D}_s, \mathcal{D}_t) + 2 \|A^*_\lambda\|^2_F \mathcal{L} (w^*, h^*; \mathcal{D}_s). \]

Thus, Lemma~\ref{lemma:empirical_same_input} holds.

\subsection{Proof of Theorem~\ref{thm:generalization_same_input}}
\label{proof:thm:generalization_same_input}

For any $\lambda \geq 0$ and $\delta > 0$, applying Lemma~\ref{lemma:uniform} for $(w^*, k^*)$ and Lemma~\ref{lemma:empirical_same_input}, with probability at least $1 - \delta$:
\begin{align*}
  \mathcal{R} (w^*, k^*) &\le \mathcal{L} (w^*, k^*; \mathcal{D}_t) + C(d, d_t, M, H, L, \delta)/\sqrt{n} \\
  &\le - 2 \widehat{\mathcal{T}}^{\mathrm{lab}}_{\lambda}(\mathcal{D}_s, \mathcal{D}_t) + 2 \|A^*_\lambda\|^2_F ~ \mathcal{L} (w^*, h^*; \mathcal{D}_s) + C(d, d_t, M, H, L, \delta)/\sqrt{n}.
\end{align*}

Since $\mathrm{Tr}(\mathcal{D}_s, \mathbb{P}_t) = -\mathcal{R} (w^*, k^*)$, Theorem~\ref{thm:generalization_same_input} holds.

\section{More details for experiment settings}

\subsection{More details for Sections~\ref{exp:different_input}--\ref{sec:lambda_exp}}
\label{appendix:experiment_settings_1}

For these experiments, we train our source models from scratch using the MSE loss with the AdamW optimizer~\citep{loshchilov2018decoupled}, which we run for 40 epochs with batch size of 64 and the cosine learning rate scheduler. To obtain good source models, we resize all input images to 256$\times$256 and apply basic image augmentations without horizontal flipping (i.e., affine transformation, Gaussian blur, and color jitter). We also scale all labels into $[0, 1]$ using the width and height of the input images.

For the transfer learning setting with head re-training, we freeze the trained feature extractor and re-train the regression head on the target dataset using the same setting above, except that we run 15 epochs on the CUB-200-2011 dataset and 30 epochs on the OpenMonkey dataset. For half fine-tuning, we unfreeze the last convolution layer and the head classifier since the number of trainable parameters is around half of the total number of parameters. For full fine-tuning, we unfreeze the whole network. In these two fine-tuning settings, we fine-tune for 15 epochs on both datasets. We use PyTorch~\citep{paszke2019pytorch} for implementation.

\subsection{More details for Section~\ref{sec:beyond_regression}}
\label{appendix:experiment_settings_2}

For this experiment, we use the following 8 ImageNet pre-trained models as the source models: ResNet50, ResNet101, ResNet152~\citep{he2016deep}, DenseNet121, DenseNet169, DenseNet201~\citep{huang2017densely}, GoogleNet~\citep{szegedy2015going}, and Inceptionv3~\citep{szegedy2016rethinking}. These models are taken from the PyTorch Model Zoo. 

We use the dSprites dataset~\citep{matthey2017dsprites} for the target task. This dataset contains 737,280 images with 4 outputs for regression: x and y positions, scale, and orientation. The train-test split is similar to the settings in~\cite{you2021logme}: 60\% for training, 20\% for validation, and 20\% for testing. The transferred MSE is computed on the test set. We train our models with 10 epochs using the AdamW optimizer. The initial learning rate is $10^{-3}$, which is divided by 10 every 3 epochs.

\section{Additional experiment results}

\subsection{Usefulness of theoretical bounds}
\label{appendix:tightness_bounds}

Although the theoretical bounds in Section~\ref{sec:theory} show the relationships between the transferability of the optimal transferred model and our transferability estimators, these bounds could be loose in practice unless the number of samples is large. This is in fact a limitation of this type of generalization bounds. To show the usefulness of our bounds in practice, we conduct an experiment to investigate the generalization gap using the head re-training setting in Section~\ref{exp:different_input}. 

The generalization gap is defined as the \emph{difference between our transferability score and the negative MSE (the transferability) of the transferred model}. According to our theorems, this generalization gap is bounded above by the complexity term. We will compare the generalization gap with the absolute value of our transferability score and also inspect whether it has any significant correlation with the actual transferred MSE.

From this experiment, the ratios between the absolute value of transferability score and the generalization gap for our transferability estimators are: 1.6 (LinMSE0), 2.0 (LinMSE1), 2.3 (LabMSE0), and 2.3 (LabMSE1). These results show that the transferability scores dominate the generalization gap in practice. More importantly, there is \emph{no significant correlation} between the generalization gap and the actual transferred MSE. These findings indicate that the complexity term in our bounds may have little effects for transferability estimation, as opposed to the transferability score term that has a strong effect (shown by the high correlations in our main experiments).

\begin{table*}[t]
\caption{{\bf Kendall's-$\tau$ correlation coefficients when transferring from OpenMonkey to CUB-200-2011}. Bold numbers indicate best results in each row. Asterisks (*) indicate best results among the corresponding label-based or feature-based methods. Our estimators improve up to 28.4\% in comparison with SotA (LogME) while being 13\% better on average.}
\centering
\small
\begin{tabular}{ccccccccc}
\toprule
\multirow{2}{*}{Transfer setting} & \multicolumn{4}{c}{Label-based method} & \multicolumn{4}{c}{Feature-based method} \\
\cmidrule(lr){2-5} \cmidrule(lr){6-9}
& LabLogME & LabTransRate & LabMSE0 & LabMSE1 & LogME & TransRate & LinMSE0 & LinMSE1 \\
\midrule
Head re-training & 0.728 & 0.028 & ~~\textbf{0.935}* & 0.924 & 0.906 & 0.104 & 0.896 & 0.922*\\
Half fine-tuning & 0.525 & 0.392 & 0.644 & ~~0.646* & 0.651 & 0.291 & ~~\textbf{0.667}* & 0.646\\
Full fine-tuning & 0.497 & 0.289 & ~~0.606* & 0.594 & 0.611 & 0.328 & ~~{\bf 0.616}* & 0.594 \\
\bottomrule
\end{tabular}
\label{tab:kendall}
\end{table*}

\begin{table*}[t]
\caption{{\bf Spearman correlation coefficients when transferring from OpenMonkey to CUB-200-2011}. Bold numbers indicate best results in each row. Asterisks (*) indicate best results among the corresponding label-based or feature-based methods. Our estimators improve up to 19.9\% in comparison with SotA (LogME) while being 9.7\% better on average.}
\centering
\small
\begin{tabular}{ccccccccc}
\toprule
\multirow{2}{*}{Transfer setting} & \multicolumn{4}{c}{Label-based method} & \multicolumn{4}{c}{Feature-based method} \\
\cmidrule(lr){2-5} \cmidrule(lr){6-9}
& LabLogME & LabTransRate & LabMSE0 & LabMSE1 & LogME & TransRate & LinMSE0 & LinMSE1 \\
\midrule
Head re-training & 0.857 & 0.102 & ~~\textbf{0.994}* & 0.991 & 0.988 & 0.215 & 0.984 & ~~0.990*\\
Half fine-tuning & 0.726 & 0.409 & 0.857 & ~~0.858* & 0.857 & 0.437 & ~~\textbf{0.865}* & 0.858\\
Full fine-tuning & 0.689 & 0.433 & ~~0.826* & 0.823 & ~~\textbf{0.827}* & 0.474 & ~~\textbf{0.827}* & 0.823 \\
\bottomrule
\end{tabular}
\label{tab:spearman}
\end{table*}

\begin{table*}[t]
\caption{{\bf Correlation coefficients when transferring between 10d-output tasks from OpenMonkey to CUB-200-2011}. Bold numbers indicate best results in each row. Asterisks (*) indicate best results among the corresponding label-based or feature-based methods. All correlations are statistically significant with $p<0.001$. Our estimators with both $\lambda$ values are better than SotA (LogME).}
\centering
\small
\begin{tabular}{ccccccccc}
\toprule
\multirow{2}{*}{Transfer setting} & \multicolumn{4}{c}{Label-based method} & \multicolumn{4}{c}{Feature-based method} \\
\cmidrule(lr){2-5} \cmidrule(lr){6-9}
& LabLogME & LabTransRate & LabMSE0 & LabMSE1 & LogME & TransRate & LinMSE0  & LinMSE1 \\
\midrule
Head re-training & 0.970 & 0.719 & 0.991* & 0.989 & 0.968 & 0.656 & 0.990 & ~~\textbf{0.995}*\\
Half fine-tuning & 0.944 & 0.742 & 0.963* & 0.943 & 0.954 & 0.684 & ~~\textbf{0.980}* & 0.958\\
Full fine-tuning & 0.878 & 0.736 & 0.892* & 0.863 & 0.892 & 0.669 & ~~\textbf{0.916}* & 0.881\\
\bottomrule
\end{tabular}
\label{tab:different_input_high_dim}
\end{table*}

\subsection{Additional results for Section~\ref{exp:different_input}}
\label{appendix:high_dim_exp_1}

\minisection{Detailed correlation plots for Table~\ref{tab:different_input}.}
In Figures~\ref{fig:different_input_head_rt},~\ref{fig:different_input_half_ft}, and~\ref{fig:different_input_full_ft}, we show the detailed correlation plots and $p$-values for our experiment results reported in Table~\ref{tab:different_input} of the main paper. From these plots, all correlations are statistically significant with $p < 0.001$, except for TransRate and LabTransRate with head re-training.

\minisection{Additional results with non-linear correlation metrics.}
In Tables~\ref{tab:kendall} and~\ref{tab:spearman}, we report the Kendall's-$\tau$ and Spearman correlation coefficients to complement the results in Table~\ref{tab:different_input} of the main paper. These coefficients, as described in~\cite{bolya2021scalable}, are used to assess the ranking associations or the monotonic relationships between the transferability measures and the model performance. Based on the findings presented in these tables, our proposed scores are generally on par with or outperform the current state-of-the-art (SotA) approach, LogME~\citep{you2021logme}, with an average correlation improvement of 9.7\% and 13\% for Spearman and Kendall's-$\tau$ coefficients, respectively. This serves as a strong evidence illustrating the effectiveness of our proposed measures, not only in the linear relationship assessment, but also in the non-linear one.

\minisection{Additional result with high-dimensional labels.}
Using the setting in Section~\ref{exp:different_input}, we also conducted an additional experiment where both source and target tasks have 10-dimensional labels. In particular, we train a source model to predict five OpenMonkey keypoints: \emph{right eye}, \emph{left eye}, \emph{nose}, \emph{head}, and \emph{neck} simultaneously (i.e., this source model returns a 10-dimensional output). The source model is then transferred to a target task that predicts a combination of five CUB-200-2011 keypoints. We consider each combination of 5 keypoints among 10 CUB-200-2011 keypoints as a target task, resulting in 252 target tasks that all have 10-dimensional labels.

We also run 3 transfer learning algorithms: head re-training, half fine-tuning, and full fine-tune, using the same training settings as in Section~\ref{exp:different_input}. For TransRate and LabTransRate, we use 2 bins per dimension instead of 5 bins to reduce the computational costs. The results for this experiment are reported in Table~\ref{tab:different_input_high_dim}. From these results, our approaches are better than the baselines for both $\lambda$ values.

\subsection{Additional results for Section~\ref{sec:exp_shared_inputs}}
\label{appendix:high_dim_exp_2}

\begin{table*}[t]
\caption{{\bf Correlation coefficients when transferring from 2d-output tasks to 10d-output tasks on CUB-200-2011}. Bold numbers indicate best results in each row. Asterisks (*) indicate best results among the corresponding label-based or feature-based methods. Except for TransRate with half and full fine-tuning, all correlations are statistically significant with $p<0.001$. Our estimators are better than SotA (LogME) in most cases.}
\centering
\small
\begin{tabular}{ccccccccc}
\toprule
\multirow{2}{*}{Transfer setting} & \multicolumn{4}{c}{Label-based method} & \multicolumn{4}{c}{Feature-based method} \\
\cmidrule(lr){2-5} \cmidrule(lr){6-9}
& LabLogME & LabTransRate & LabMSE0  & LabMSE1  & LogME & TransRate & LinMSE0  & LinMSE1 \\
\midrule
Head re-training & 0.602 & 0.632 & ~~0.868* & 0.816 & 0.885 & 0.549 & 0.901 & ~~{\bf 0.973}*\\
Half fine-tuning & 0.491 & 0.645 & 0.771 & ~~0.881* & 0.804 & 0.072 & ~~{\bf 0.913}* & 0.818\\
Full fine-tuning & 0.397 & 0.632 & 0.727 & ~~{\bf 0.888*} & 0.756 & 0.050 & ~~0.884* & 0.833\\
\bottomrule
\end{tabular}
\label{tab:shared_input_high_dim}
\end{table*}

\minisection{Detailed correlation plots for Table~\ref{tab:shared_input}.}
In Figures~\ref{fig:shared_input_cub_head_rt}--~\ref{fig:shared_input_openmonkey_full_ft}, we show the detailed correlation plots and $p$-values for our experiment results reported in Table~\ref{tab:shared_input} of the main paper. From these plots, all correlations are statistically significant with $p < 0.001$, except for TransRate and LabTransRate as well as the full fine-tuning setting on the CUB-200-2011 dataset.

\minisection{Additional result for each individual source task.}
We report in Tables~\ref{tab:full_all_sources_cub} and~\ref{tab:full_all_sources_openmonkey} more comprehensive results for all source tasks on CUB-200-2011 and OpenMonkey respectively. Each row of the tables corresponds to one source task and shows the correlation coefficients when transferring to all other tasks in the respective dataset. From the tables, our transferability estimators are consistently better than LogME, LabLogME, TransRate, and LabTransRate for most source tasks on both datasets. These results confirm the effectiveness of our proposed methods.

\minisection{Additional result with high-dimensional labels.}
In this additional experiment, we further show the effectiveness of our proposed methods when the target tasks have higher dimensional labels. In particular, we transfer from 4 source tasks on CUB-200-2011 (\emph{back}, \emph{beak}, \emph{belly}, and \emph{breast}) to all the combinations of 5 attributes among the remaining tasks (except for \emph{right eye}, \emph{right leg}, and \emph{right wing}, which may not always be available in the data). In total, we have 224 source-target pairs, where the source tasks have 2-dimensional labels and the target tasks have 10-dimensional labels. We use the same training settings as in Section~\ref{sec:exp_shared_inputs} of the main paper, except that we also use 2 bins per dimension when calculating TransRate and LabTransRate to reduce computational costs. Table~\ref{tab:shared_input_high_dim} reports the results for this experiment. These results clearly show that our methods, LinMSE0 and LinMSE1, are better than the LogME and TransRate baselines in most cases.

\begin{figure*}[h]
\captionsetup[subfigure]{justification=centering}
    \begin{subfigure}[b]{0.23\textwidth}
    \includegraphics[width=\textwidth]{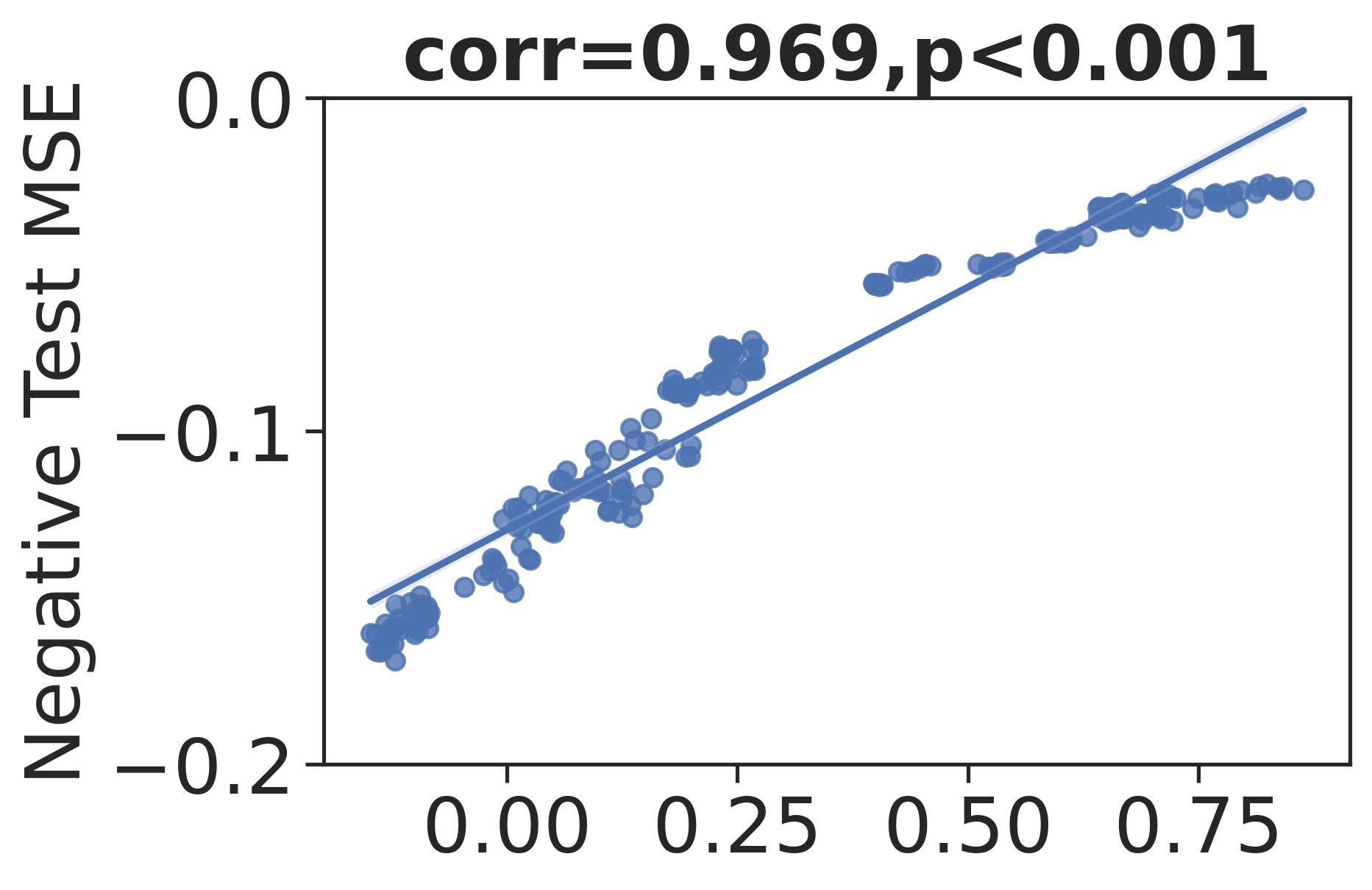}
    \caption{LogME}
    \end{subfigure}
    {\hskip 4pt}
    \begin{subfigure}[b]{0.23\textwidth}
    \includegraphics[width=\textwidth]{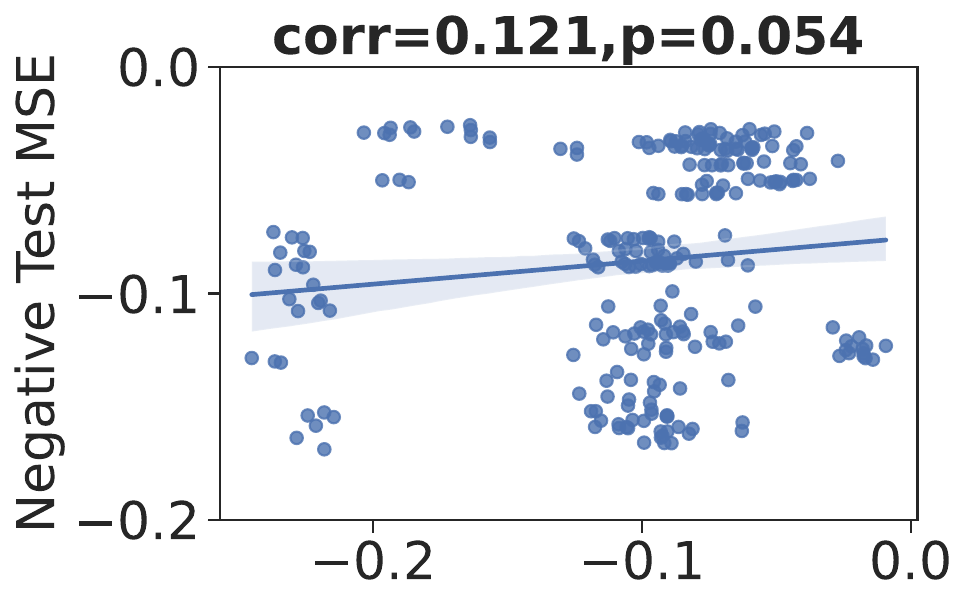}
    \caption{TransRate}
    \end{subfigure}
    {\hskip 4pt}
    \begin{subfigure}[b]{0.23\textwidth}
    \includegraphics[width=\textwidth]{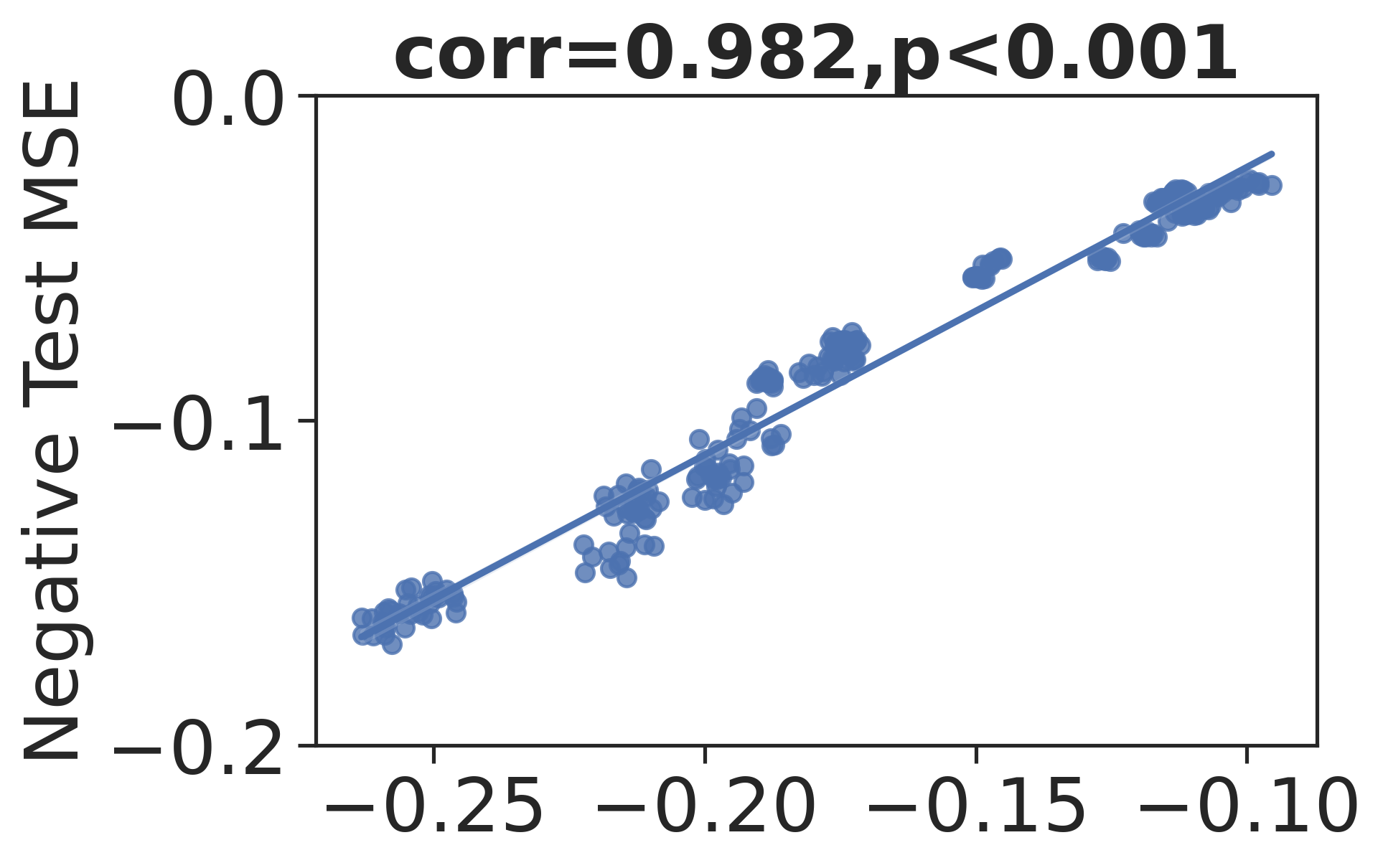}
    \caption{LinMSE0}
    \end{subfigure}
    {\hskip 4pt}
    \begin{subfigure}[b]{0.23\textwidth}
    \includegraphics[width=\textwidth]{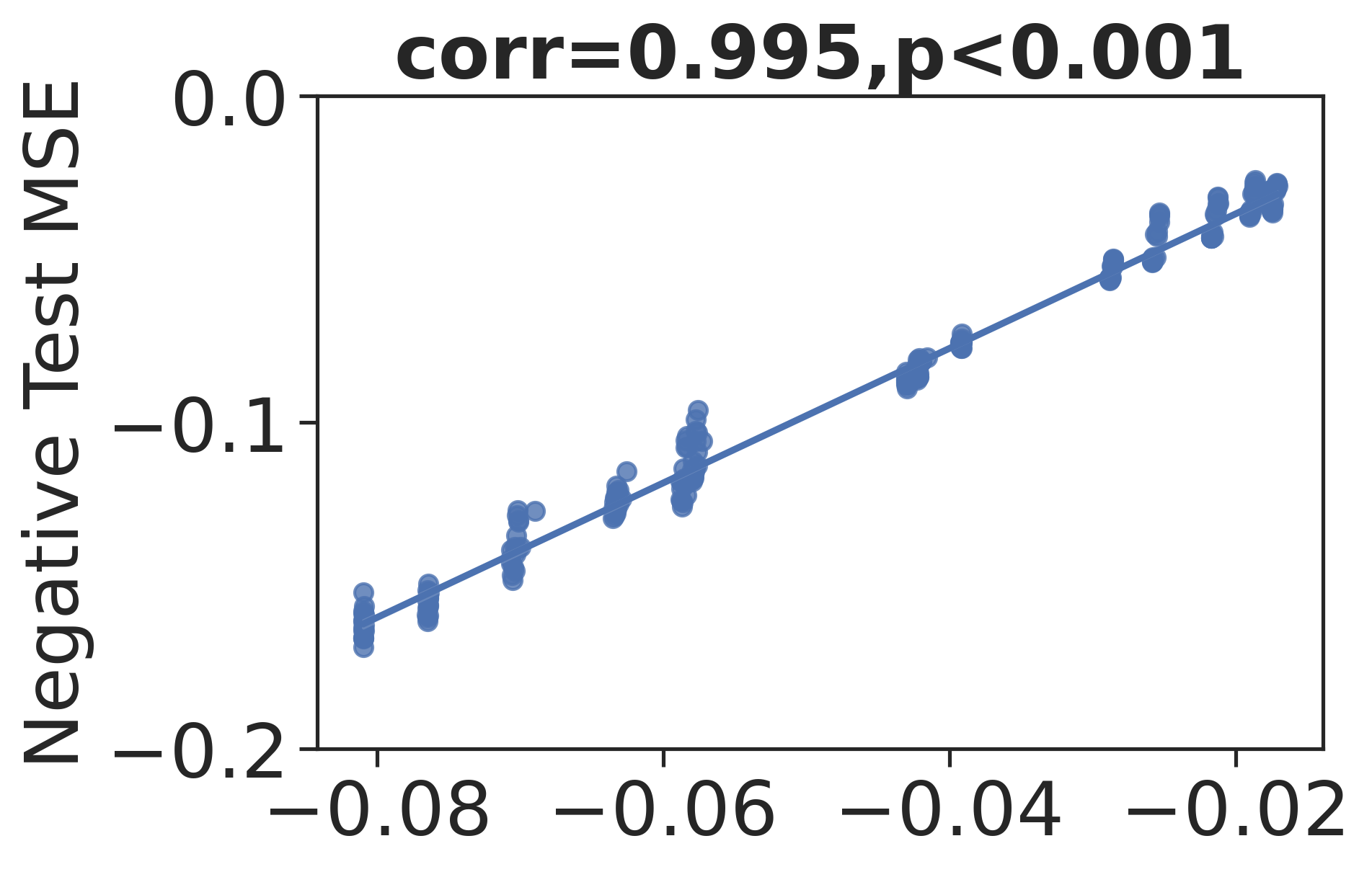}
    \caption{LinMSE1}
    \end{subfigure}
    
    {\vskip 0.2cm}
    
    \begin{subfigure}[b]{0.23\textwidth}
    \includegraphics[width=\textwidth]{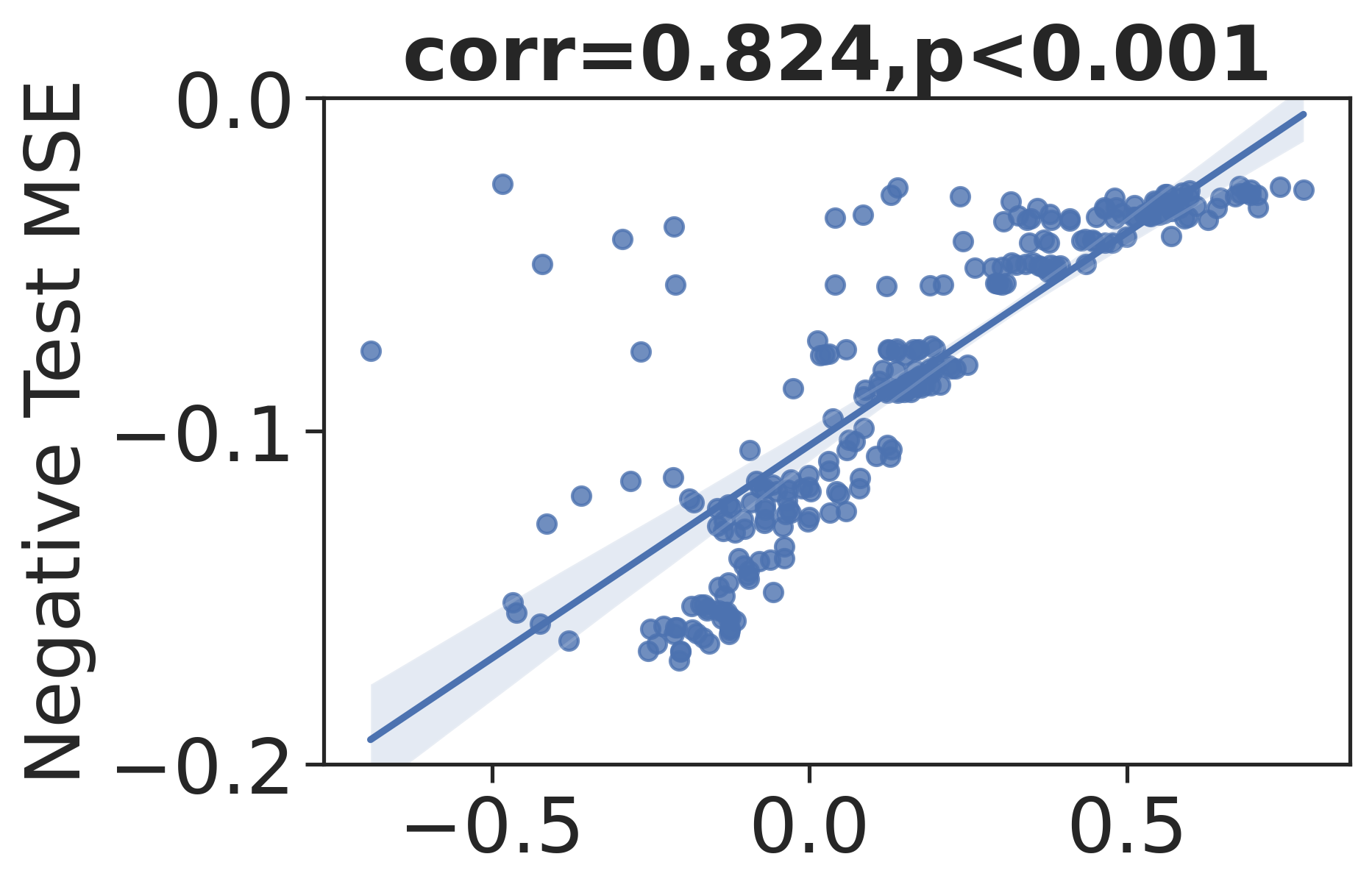}
    \caption{LabLogME}
    \end{subfigure}
    {\hskip 4pt}
    \begin{subfigure}[b]{0.23\textwidth}
    \includegraphics[width=\textwidth]{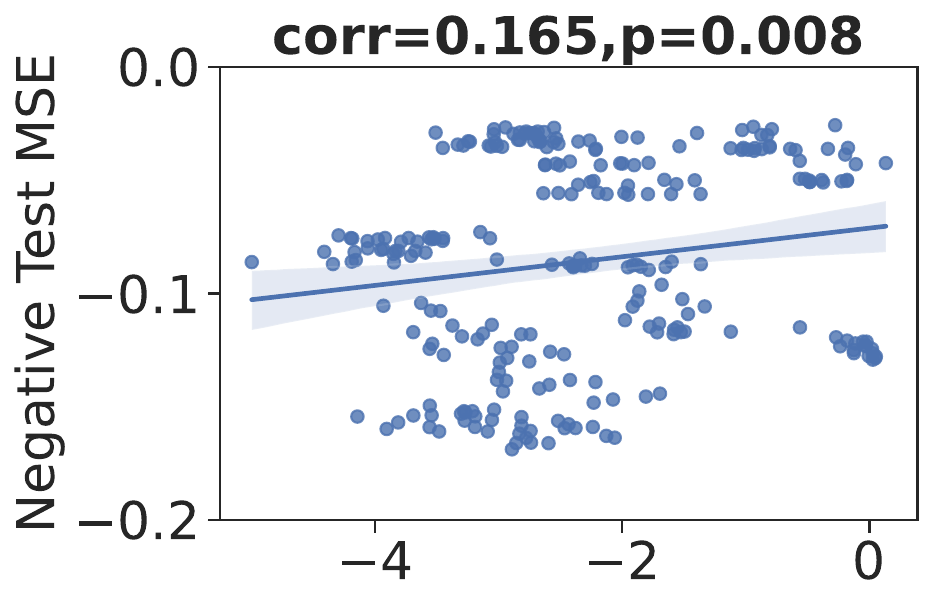}
    \caption{LabTransRate}
    \end{subfigure}
    {\hskip 4pt}
    \begin{subfigure}[b]{0.23\textwidth}
    \includegraphics[width=\textwidth]{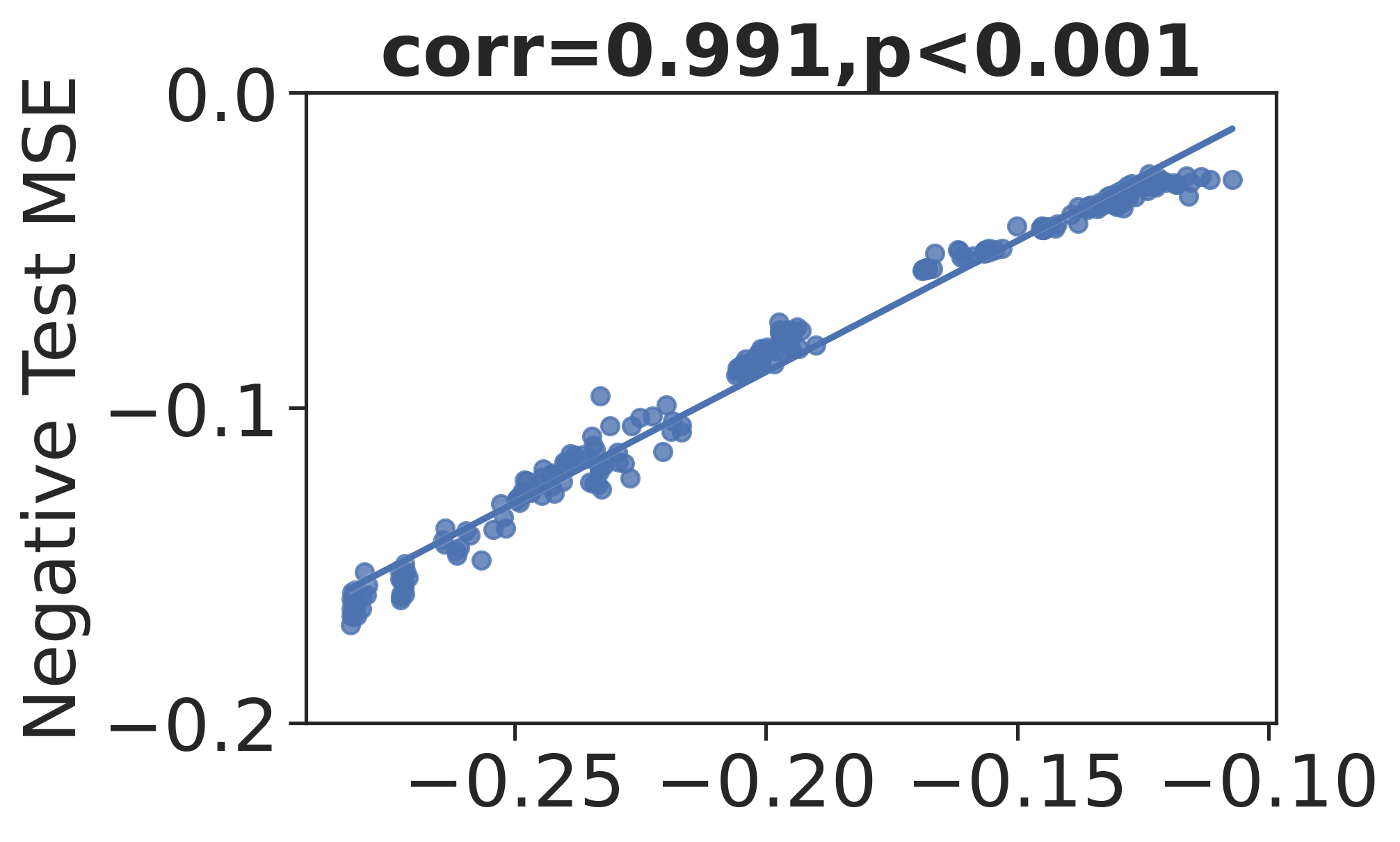}
    \caption{LabMSE0}
    \end{subfigure}
    {\hskip 4pt}
    \begin{subfigure}[b]{0.23\textwidth}
    \includegraphics[width=\textwidth]{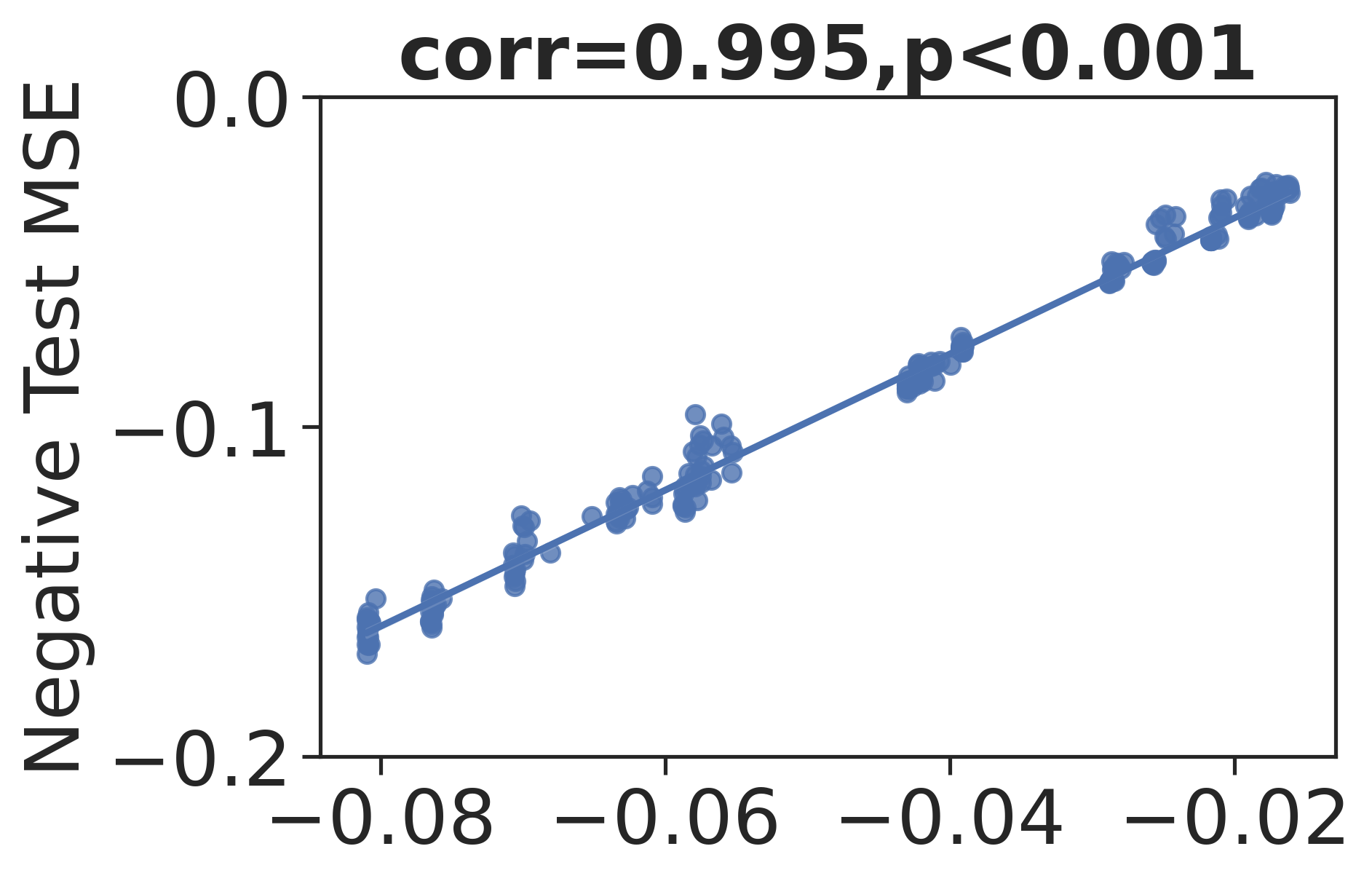}
    \caption{LabMSE1}
    \end{subfigure}
    {\vskip -0.2cm}
    \caption{\textbf{Correlation coefficients and $p$-values between transferability estimators and negative test MSEs} when transferring with head re-training from OpenMonkey to CUB-200-2011.}
    \label{fig:different_input_head_rt}
    
    {\vskip 0.4cm}
    
    \begin{subfigure}[b]{0.23\textwidth}
    \includegraphics[width=\textwidth]{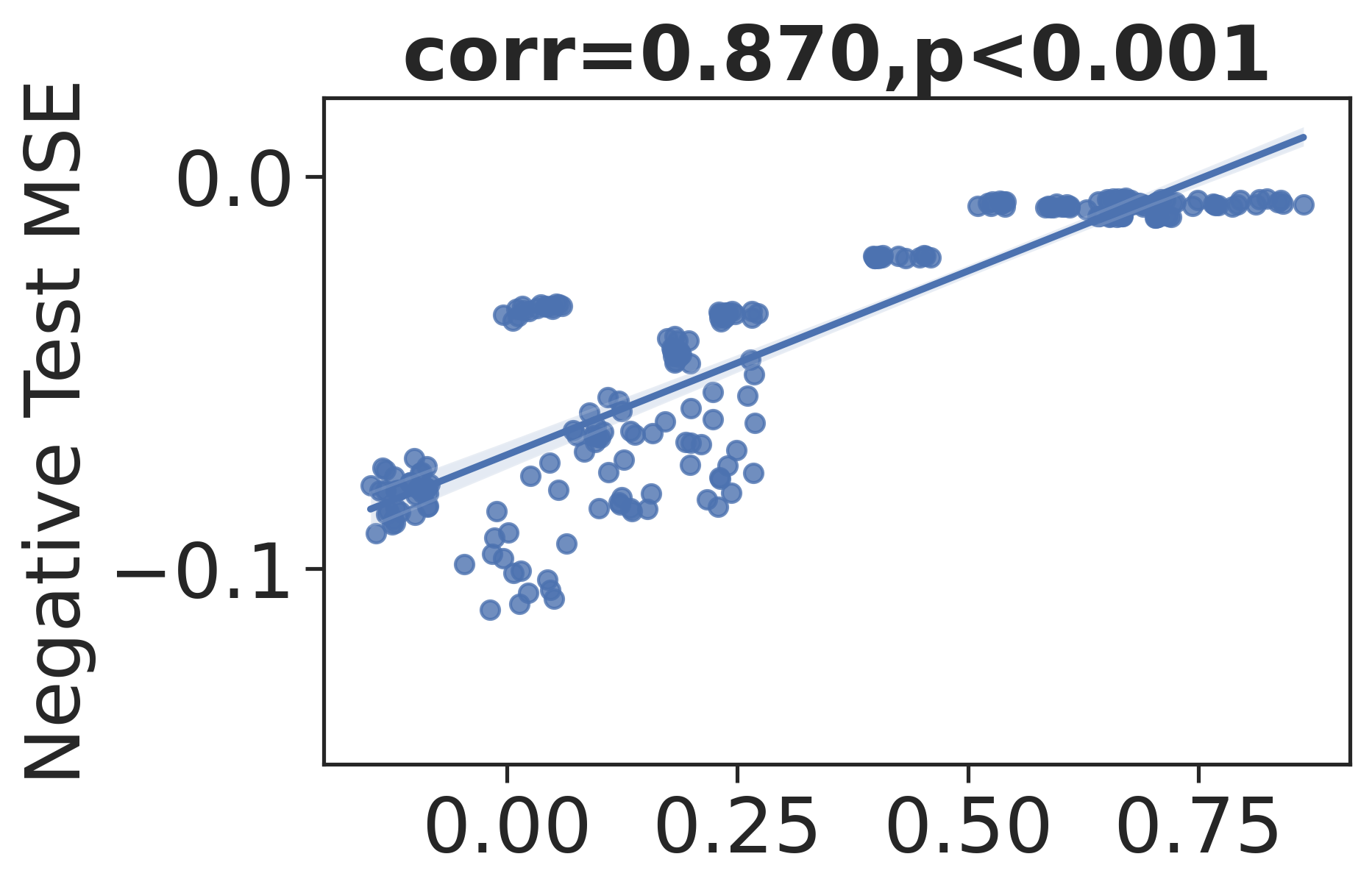}
    \caption{LogME}
    \end{subfigure}
    {\hskip 4pt}
    \begin{subfigure}[b]{0.23\textwidth}
    \includegraphics[width=\textwidth]{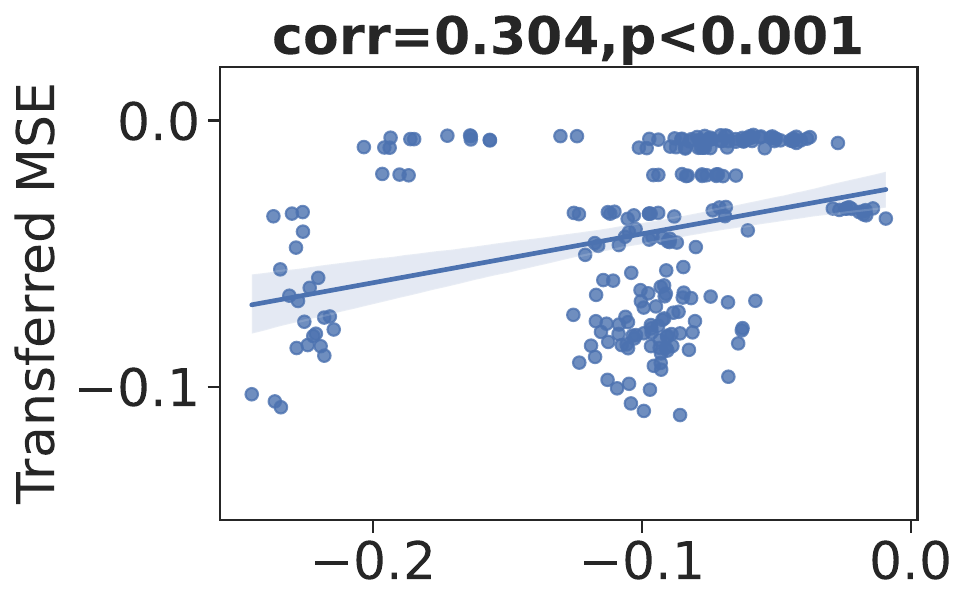}
    \caption{TransRate}
    \end{subfigure}
    {\hskip 4pt}
    \begin{subfigure}[b]{0.23\textwidth}
    \includegraphics[width=\textwidth]{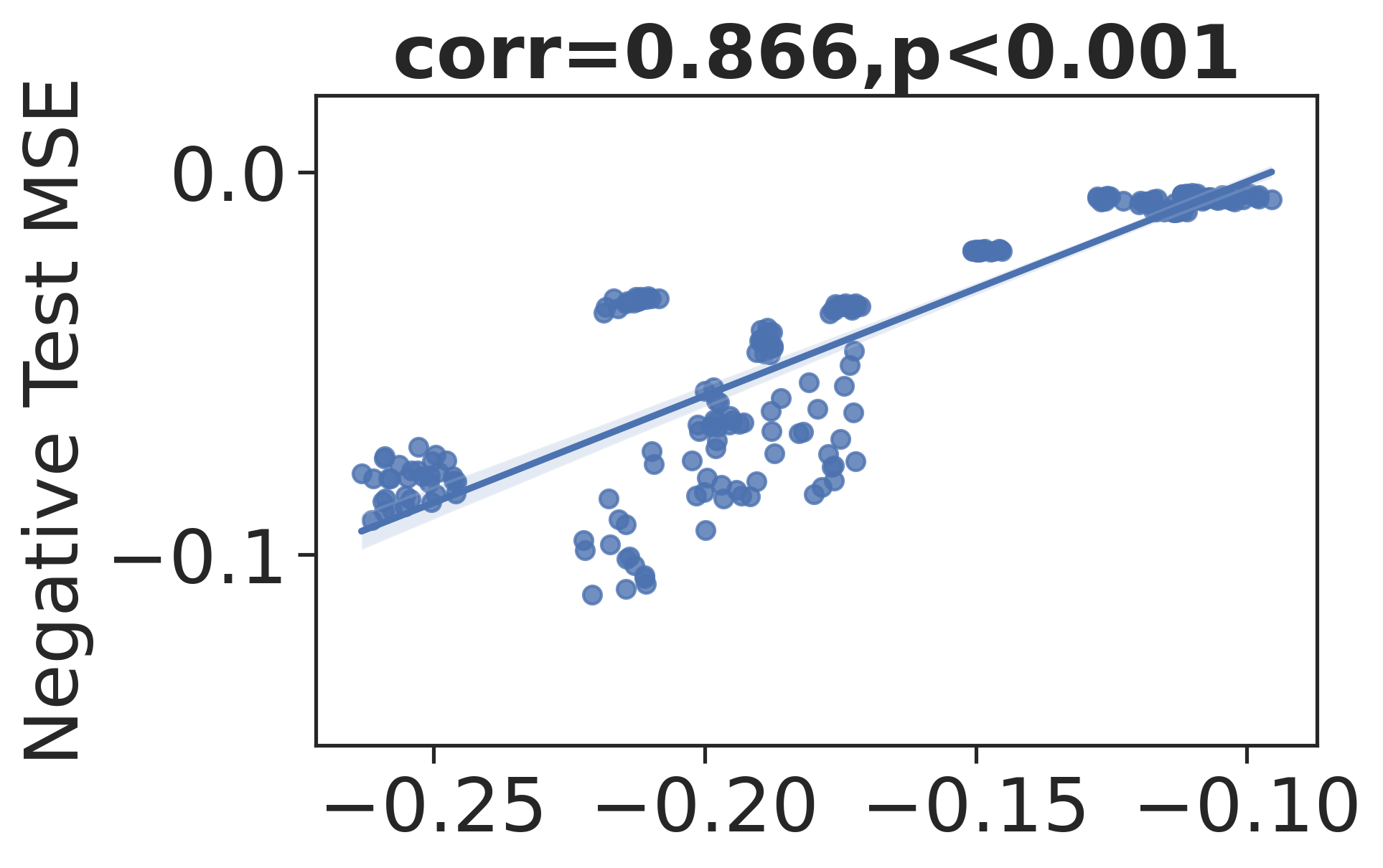}
    \caption{LinMSE0}
    \end{subfigure}
    {\hskip 4pt}
    \begin{subfigure}[b]{0.23\textwidth}
    \includegraphics[width=\textwidth]{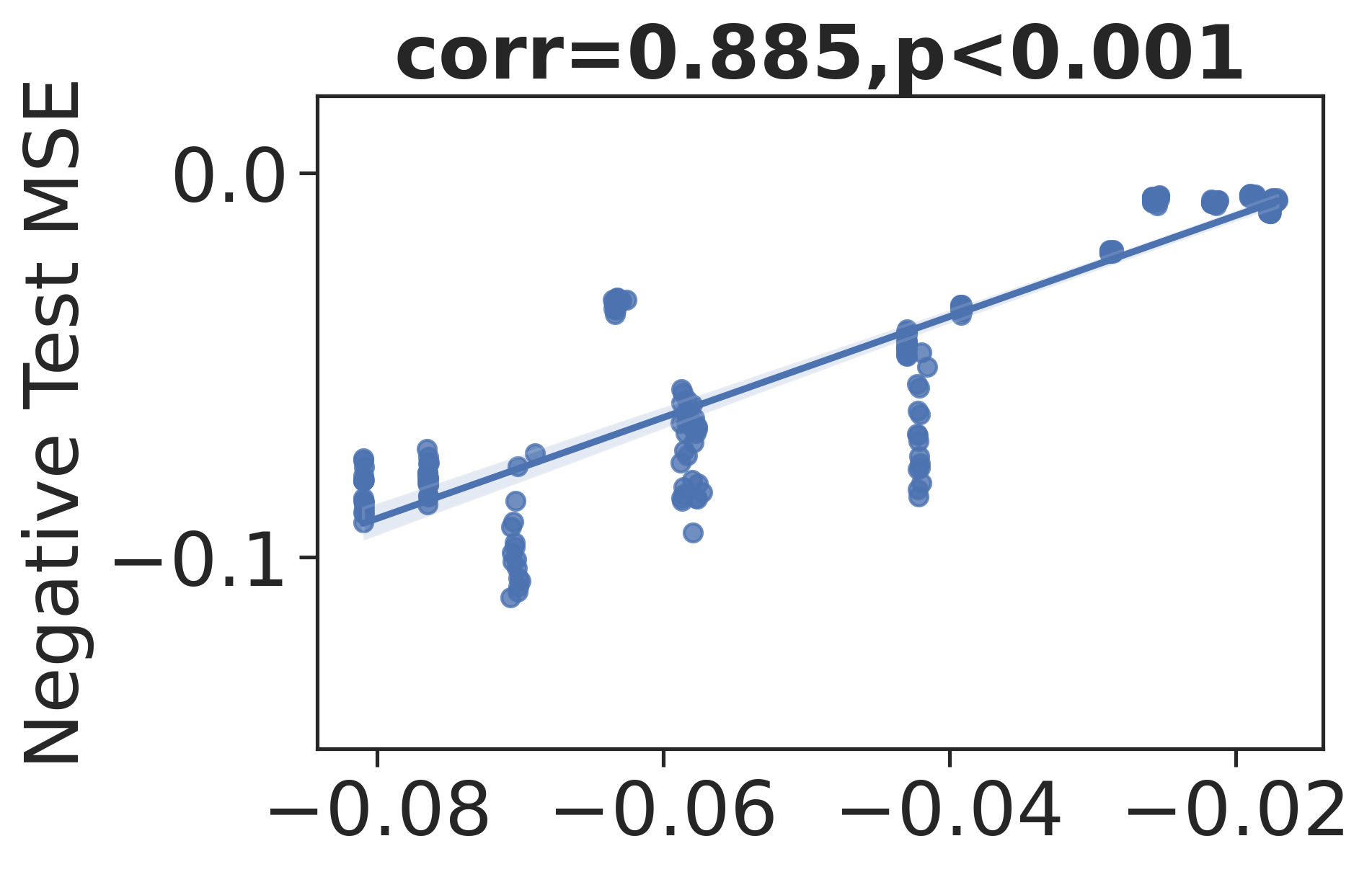}
    \caption{LinMSE1}
    \end{subfigure}
    
    {\vskip 0.2cm}
    
    \begin{subfigure}[b]{0.23\textwidth}
\includegraphics[width=\textwidth]{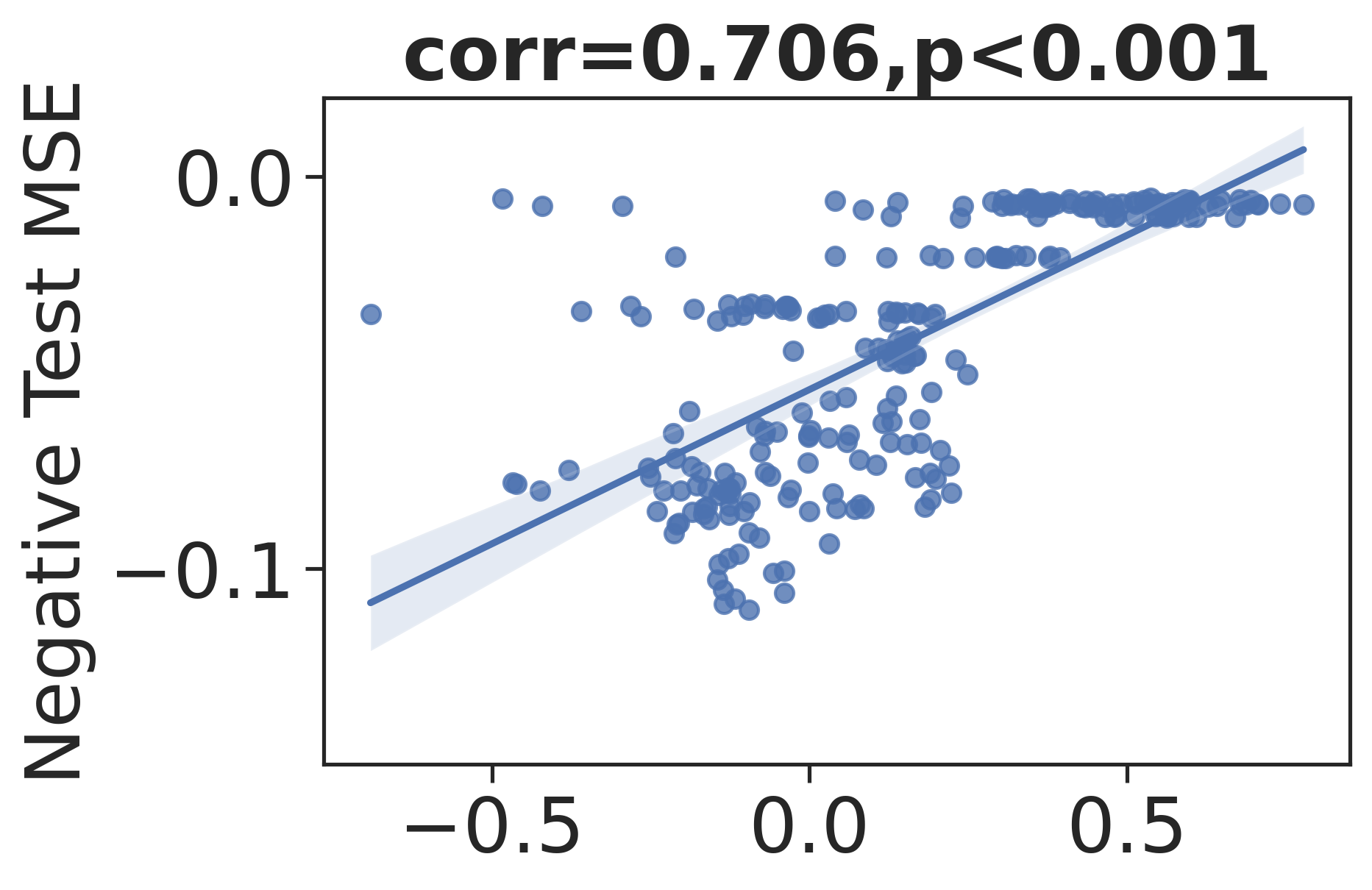}
    \caption{LabLogME}
    \end{subfigure}
    {\hskip 4pt}
    \begin{subfigure}[b]{0.23\textwidth}
    \includegraphics[width=\textwidth]{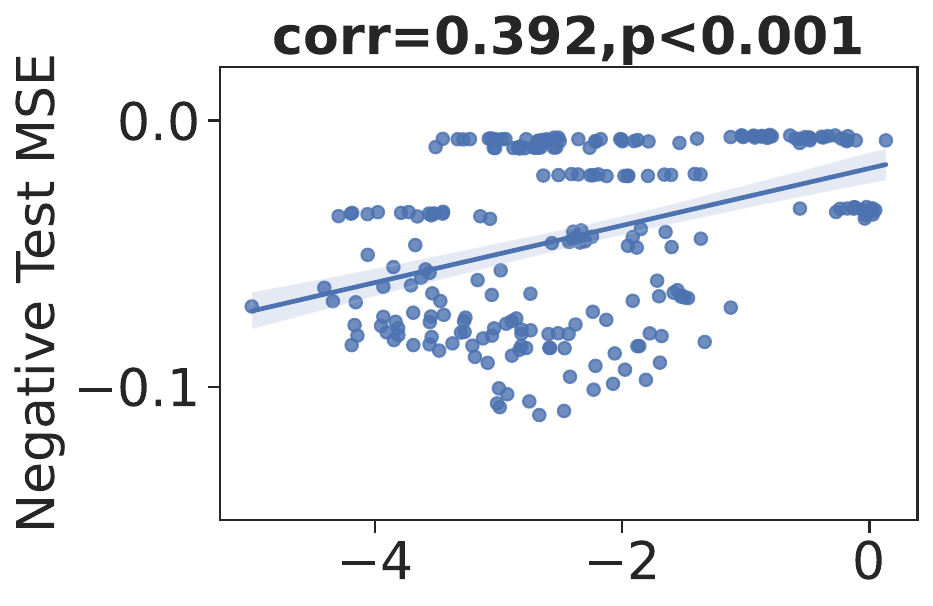}
    \caption{LabTransRate}
    \end{subfigure}
    {\hskip 4pt}
    \begin{subfigure}[b]{0.23\textwidth}
    \includegraphics[width=\textwidth]{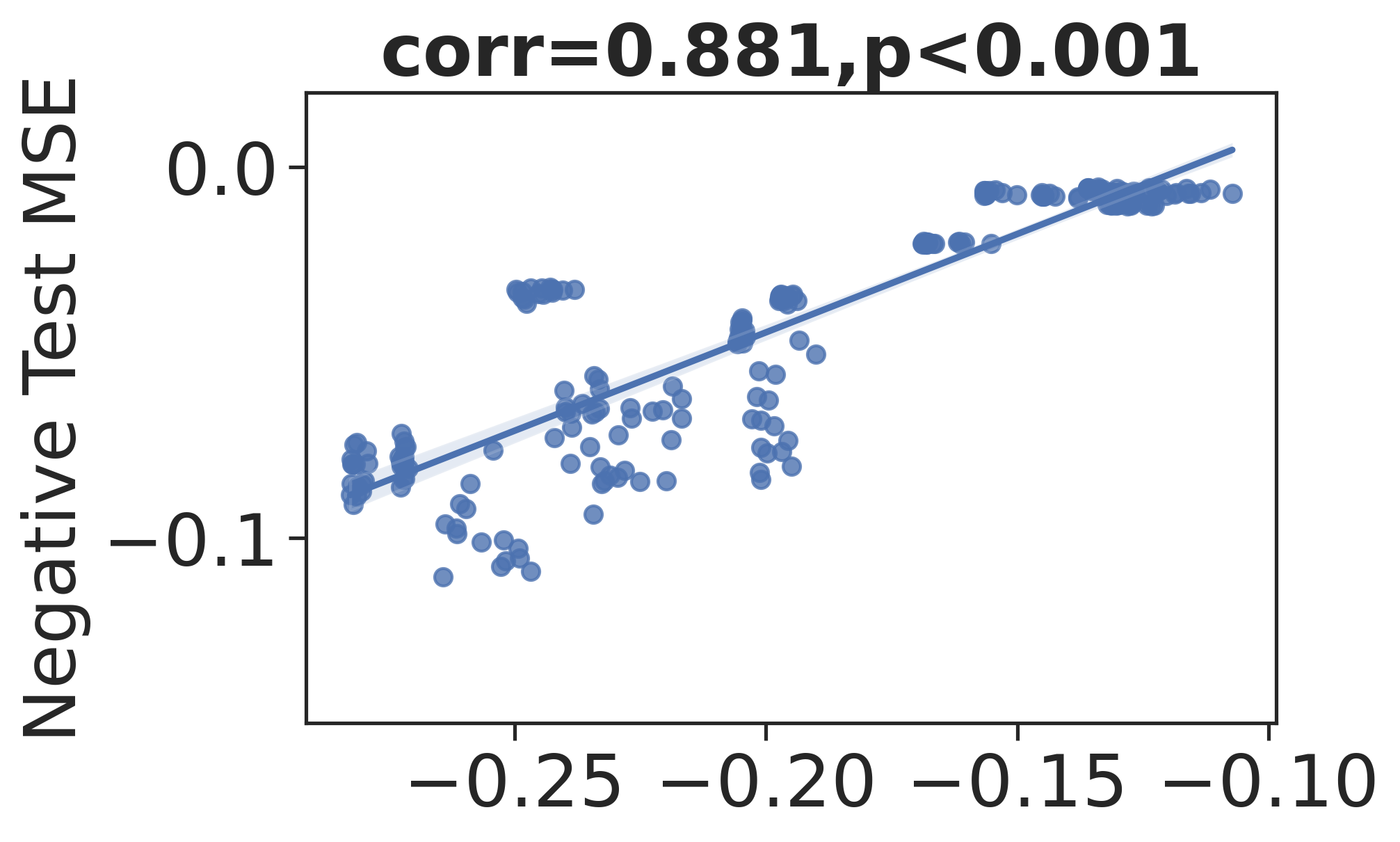}
    \caption{LabMSE0}
    \end{subfigure}
    {\hskip 4pt}
    \begin{subfigure}[b]{0.23\textwidth}
    \includegraphics[width=\textwidth]{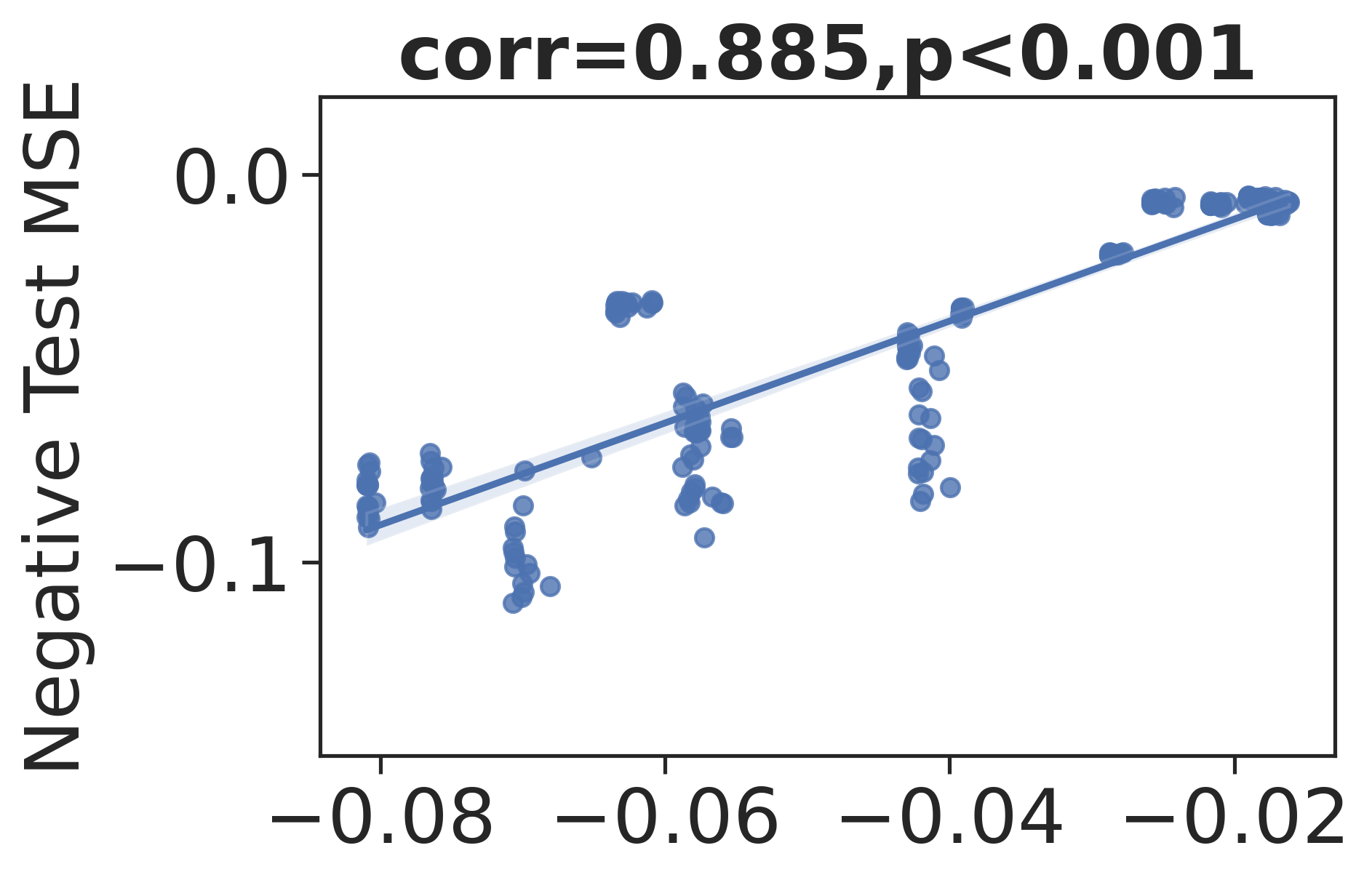}
    \caption{LabMSE1}
    \end{subfigure}
    {\vskip -0.2cm}
    \caption{\textbf{Correlation coefficients and $p$-values between transferability estimators and negative test MSEs} when transferring with half fine-tuning from OpenMonkey to CUB-200-2011.}
    \label{fig:different_input_half_ft}
    
    {\vskip 0.4cm}
    
    \begin{subfigure}[b]{0.23\textwidth}
    \includegraphics[width=\textwidth]{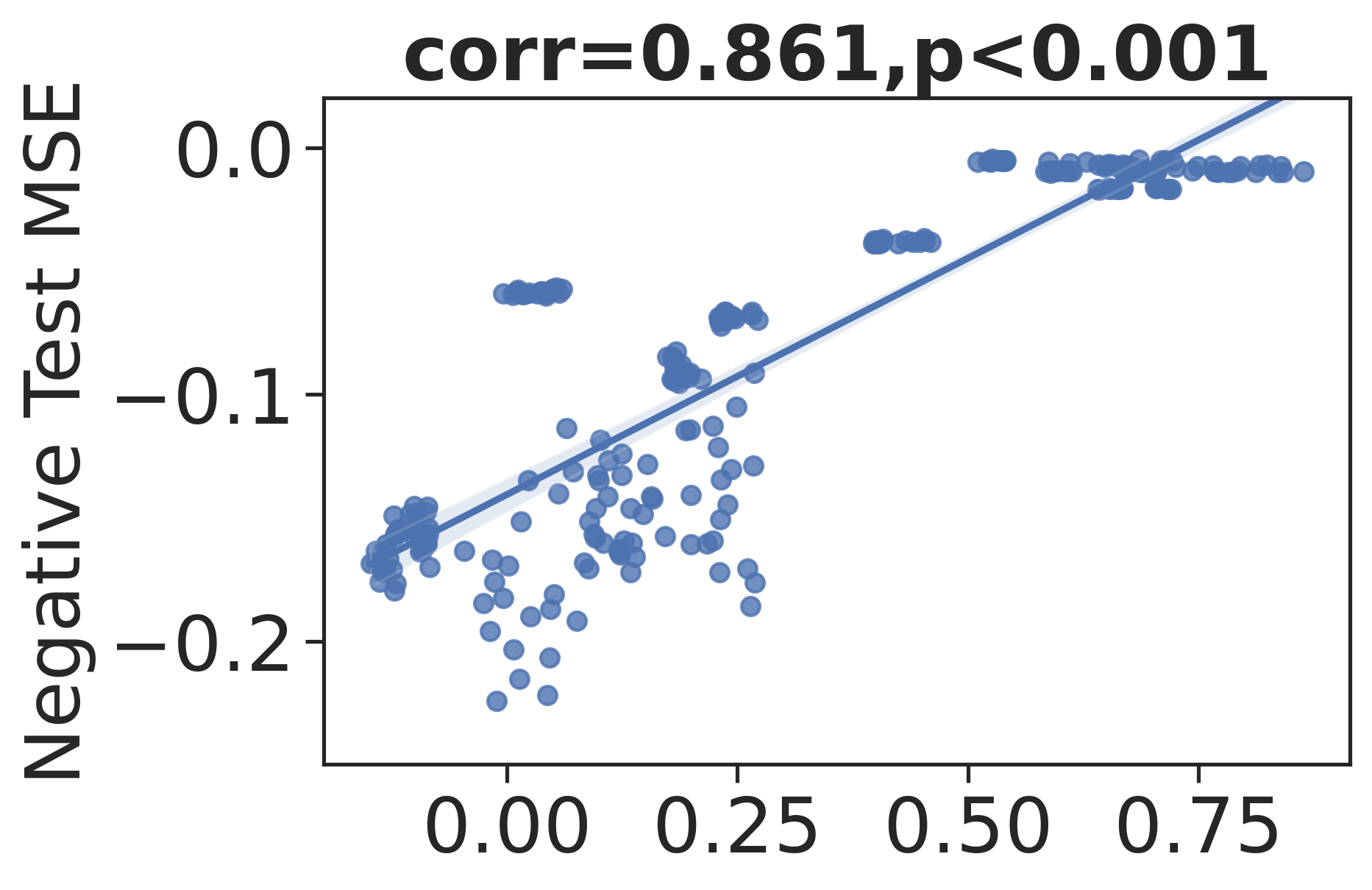}
    \caption{LogME}
    \end{subfigure}
    {\hskip 4pt}
    \begin{subfigure}[b]{0.23\textwidth}
    \includegraphics[width=\textwidth]{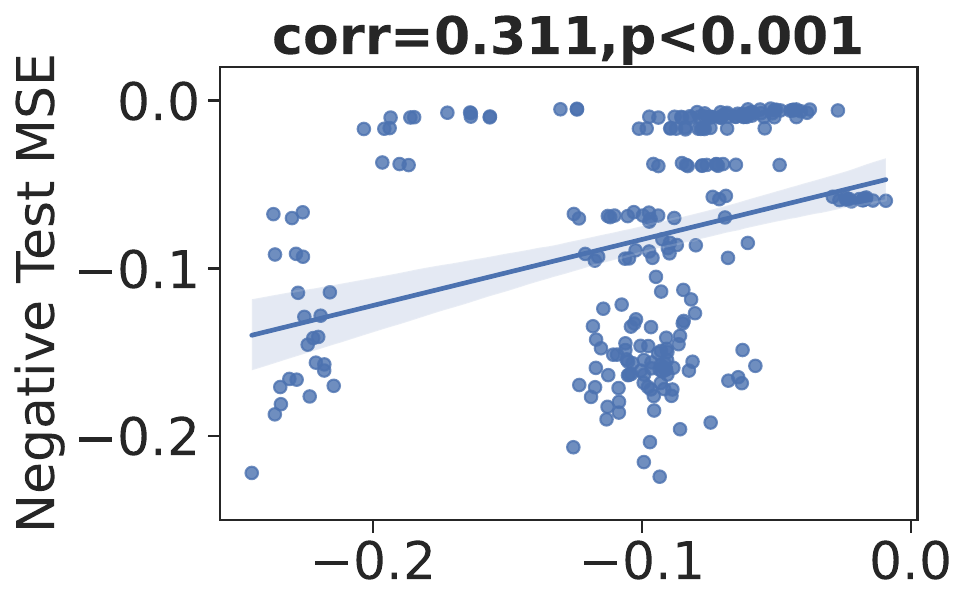}
    \caption{TransRate}
    \end{subfigure}
    {\hskip 4pt}
    \begin{subfigure}[b]{0.23\textwidth}
    \includegraphics[width=\textwidth]{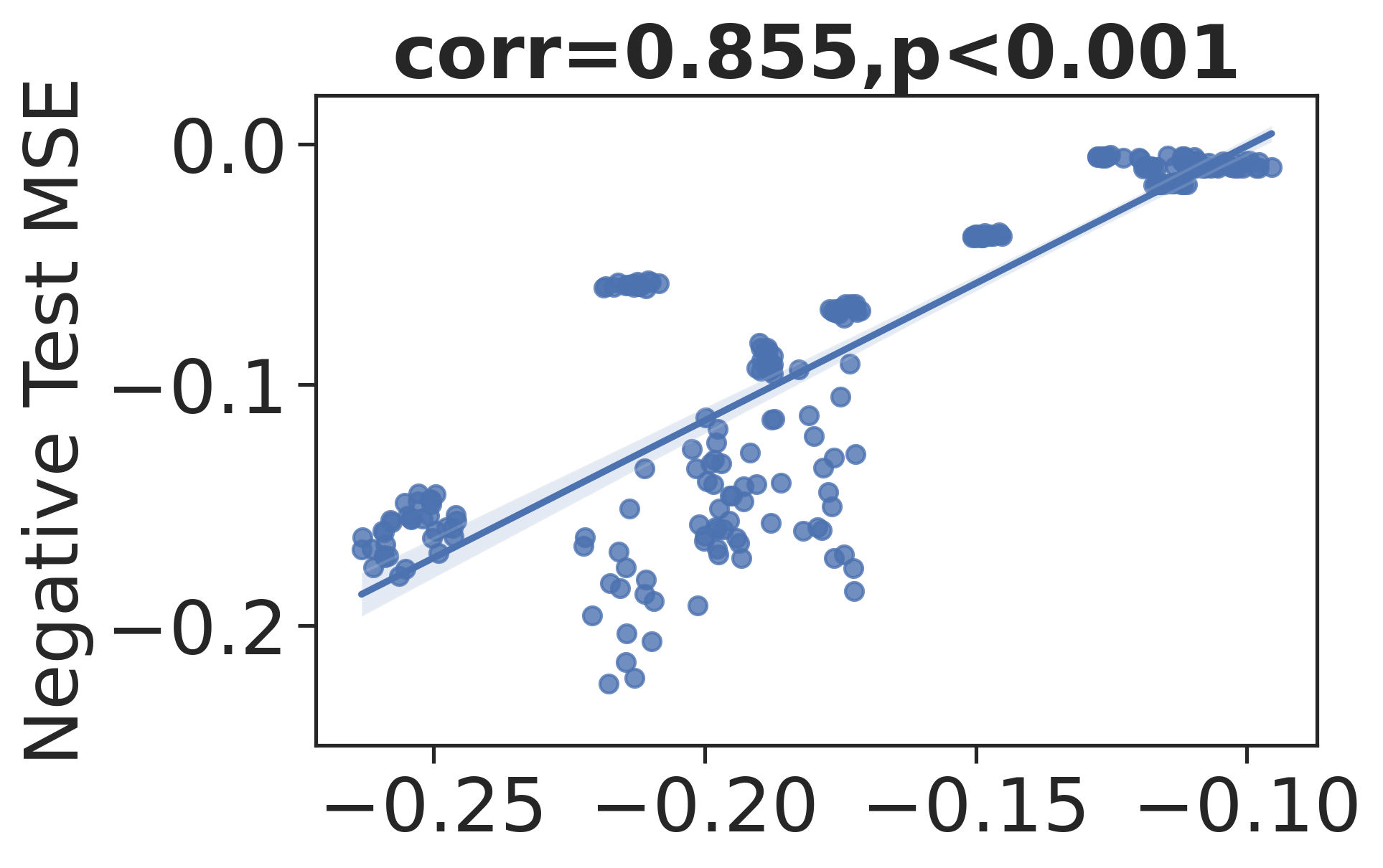}
    \caption{LinMSE0}
    \end{subfigure}
    {\hskip 4pt}
    \begin{subfigure}[b]{0.23\textwidth}
    \includegraphics[width=\textwidth]{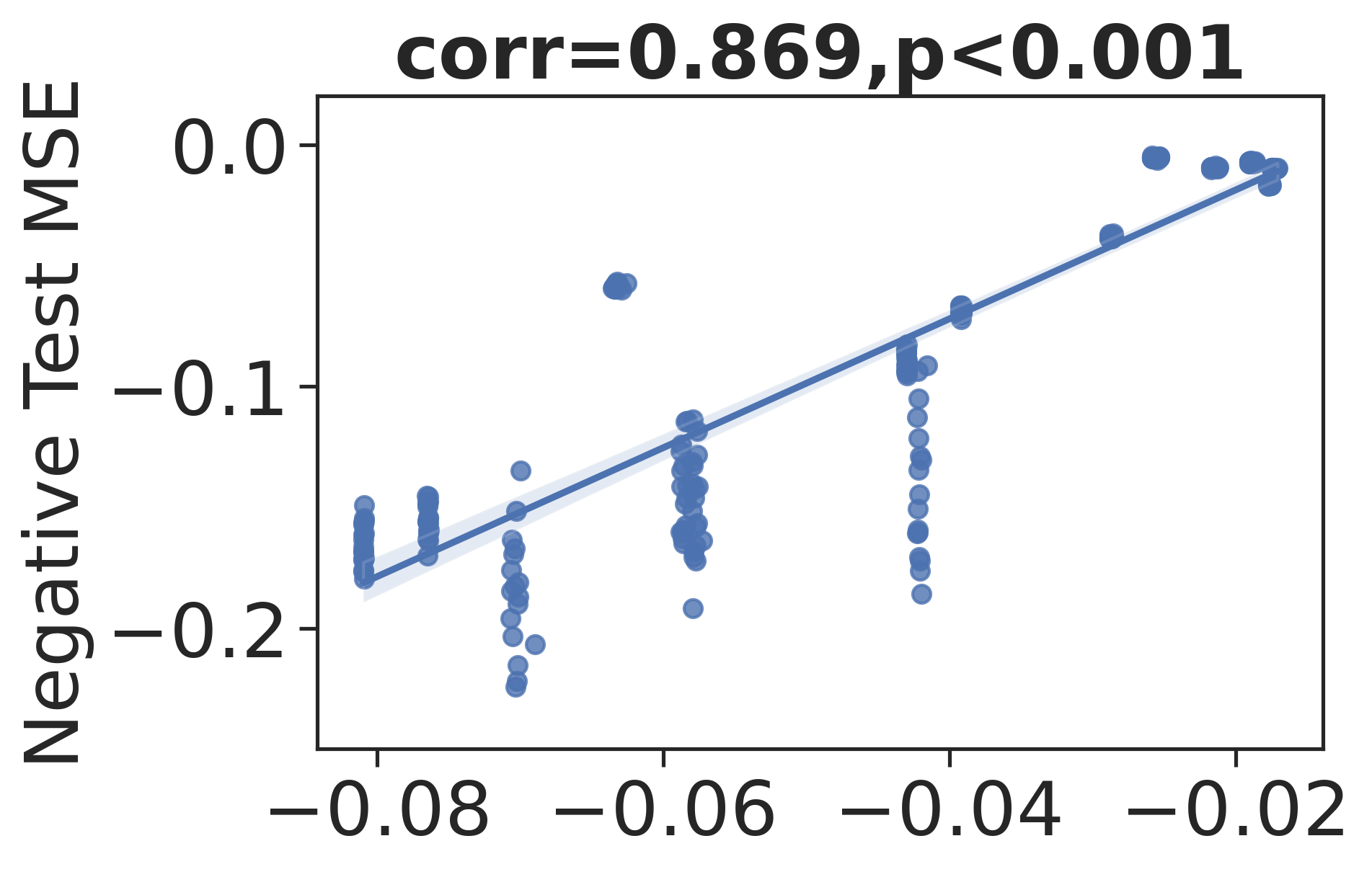}
    \caption{LinMSE1}
    \end{subfigure}
    
    {\vskip 0.2cm}
    
    \begin{subfigure}[b]{0.23\textwidth}
    \includegraphics[width=\textwidth]{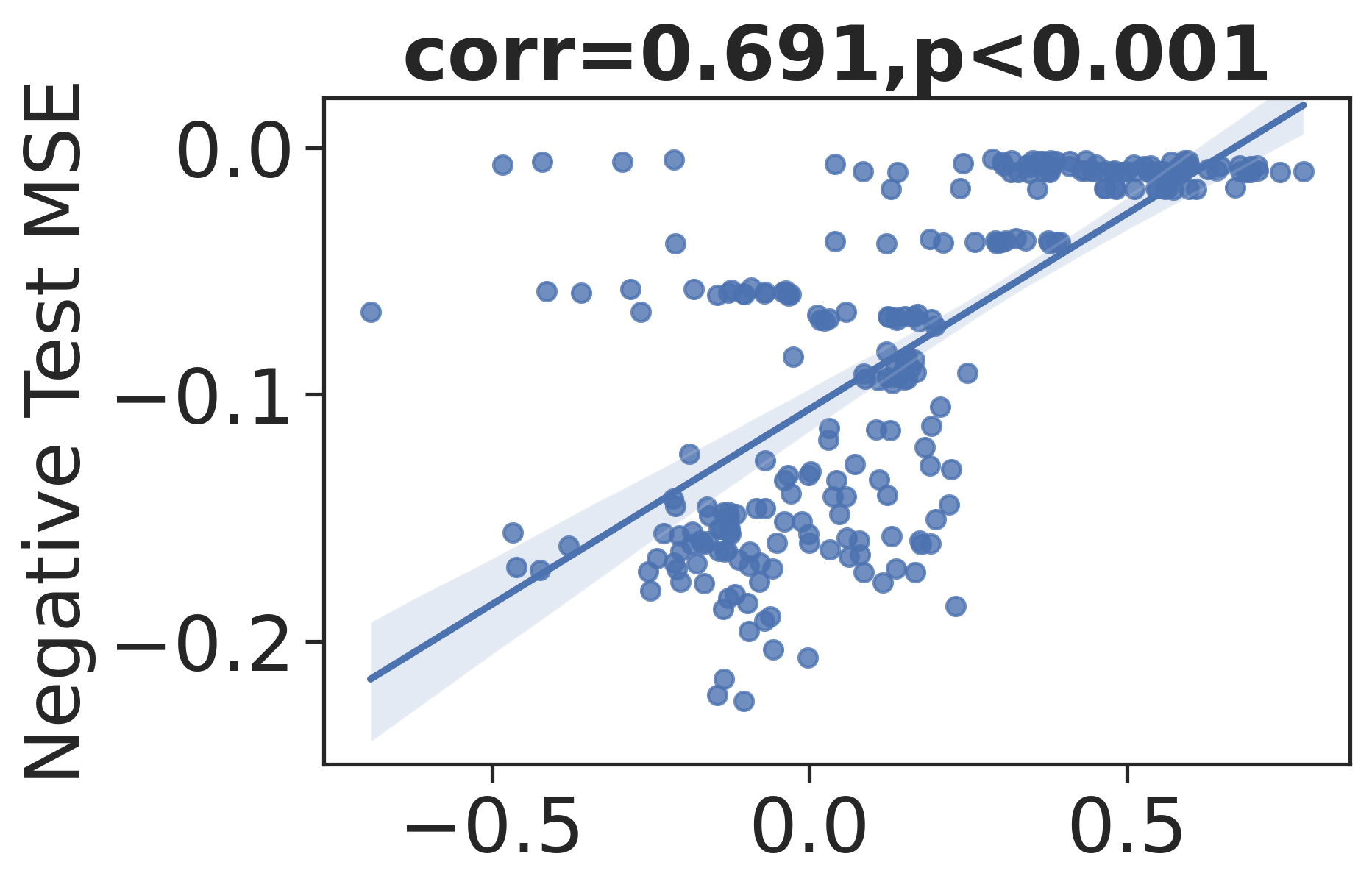}
    \caption{LabLogME}
    \end{subfigure}
    {\hskip 4pt}
    \begin{subfigure}[b]{0.23\textwidth}
    \includegraphics[width=\textwidth]{figures/supp/c3_new/transrate.pdf}
    \caption{LabTransRate}
    \end{subfigure}
    {\hskip 4pt}
    \begin{subfigure}[b]{0.23\textwidth}
\includegraphics[width=\textwidth]{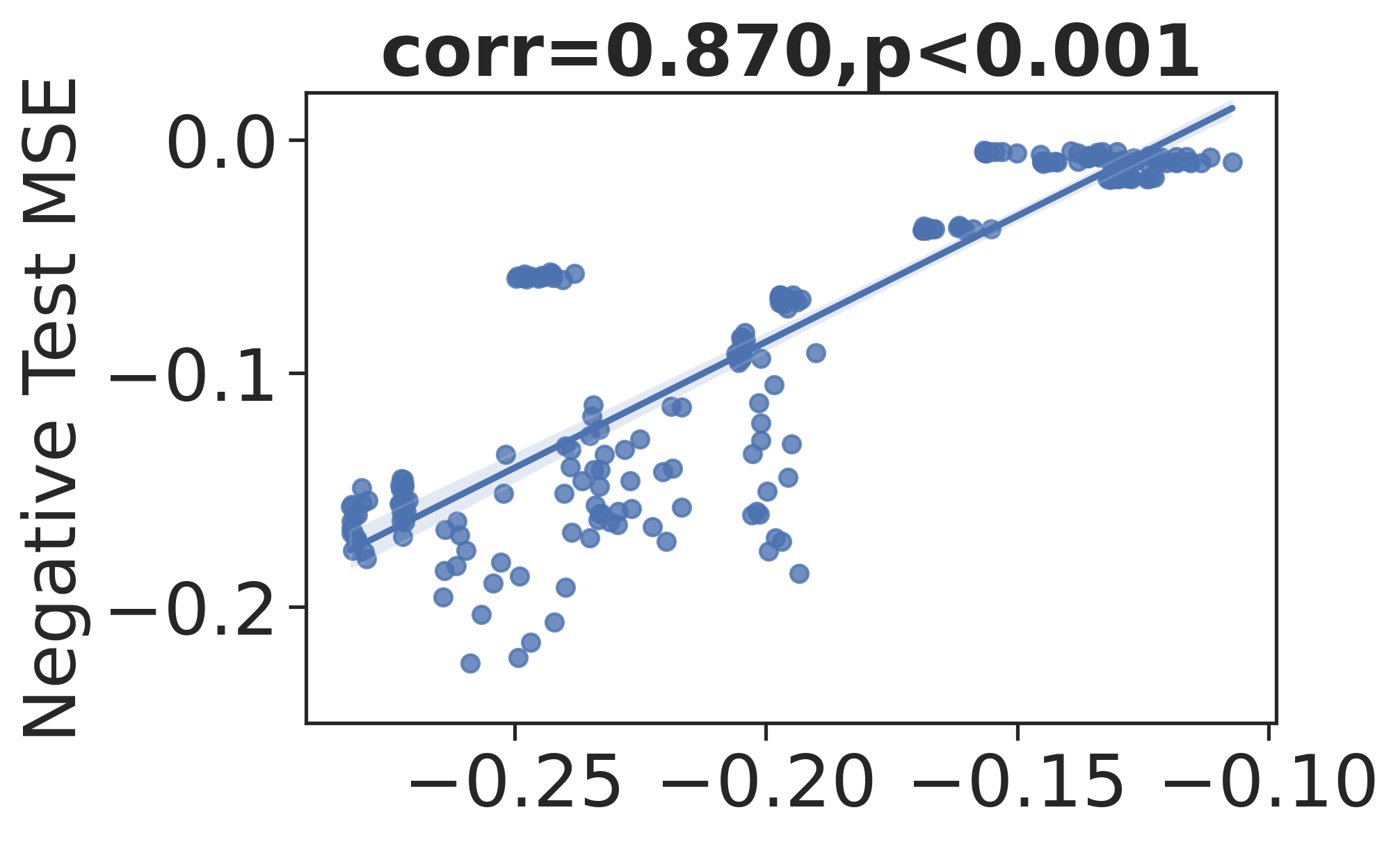}
    \caption{LabMSE0}
    \end{subfigure}
    {\hskip 4pt}
    \begin{subfigure}[b]{0.23\textwidth}
    \includegraphics[width=\textwidth]{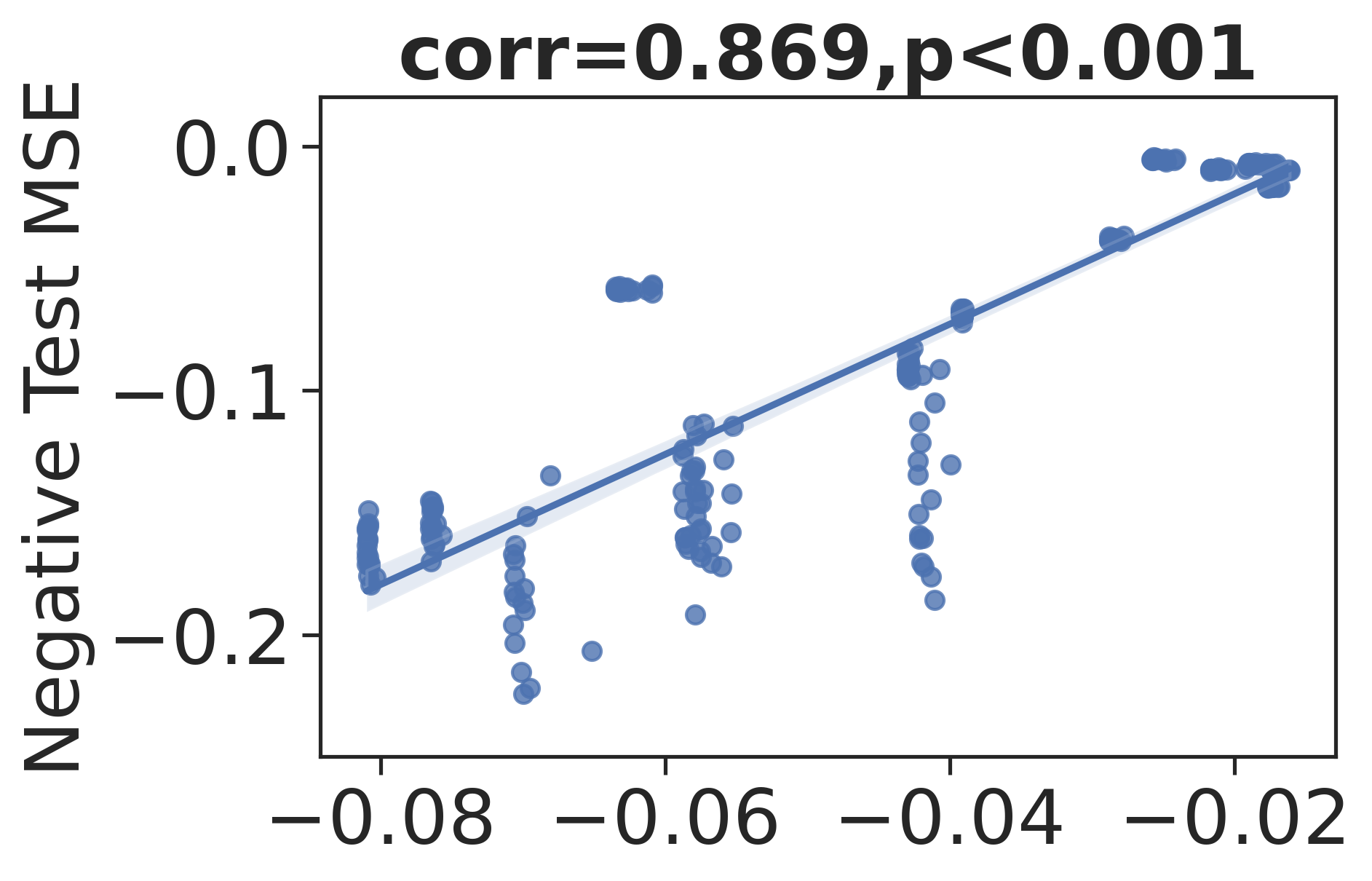}
    \caption{LabMSE1}
    \end{subfigure}
    {\vskip -0.2cm}
    \caption{\textbf{Correlation coefficients and $p$-values between transferability estimators and negative test MSEs} when transferring with full fine-tuning from OpenMonkey to CUB-200-2011.}
    \label{fig:different_input_full_ft}
\end{figure*}

\begin{figure*}[h]
\captionsetup[subfigure]{justification=centering}
    \begin{subfigure}[b]{0.23\textwidth}
    \includegraphics[width=\textwidth]{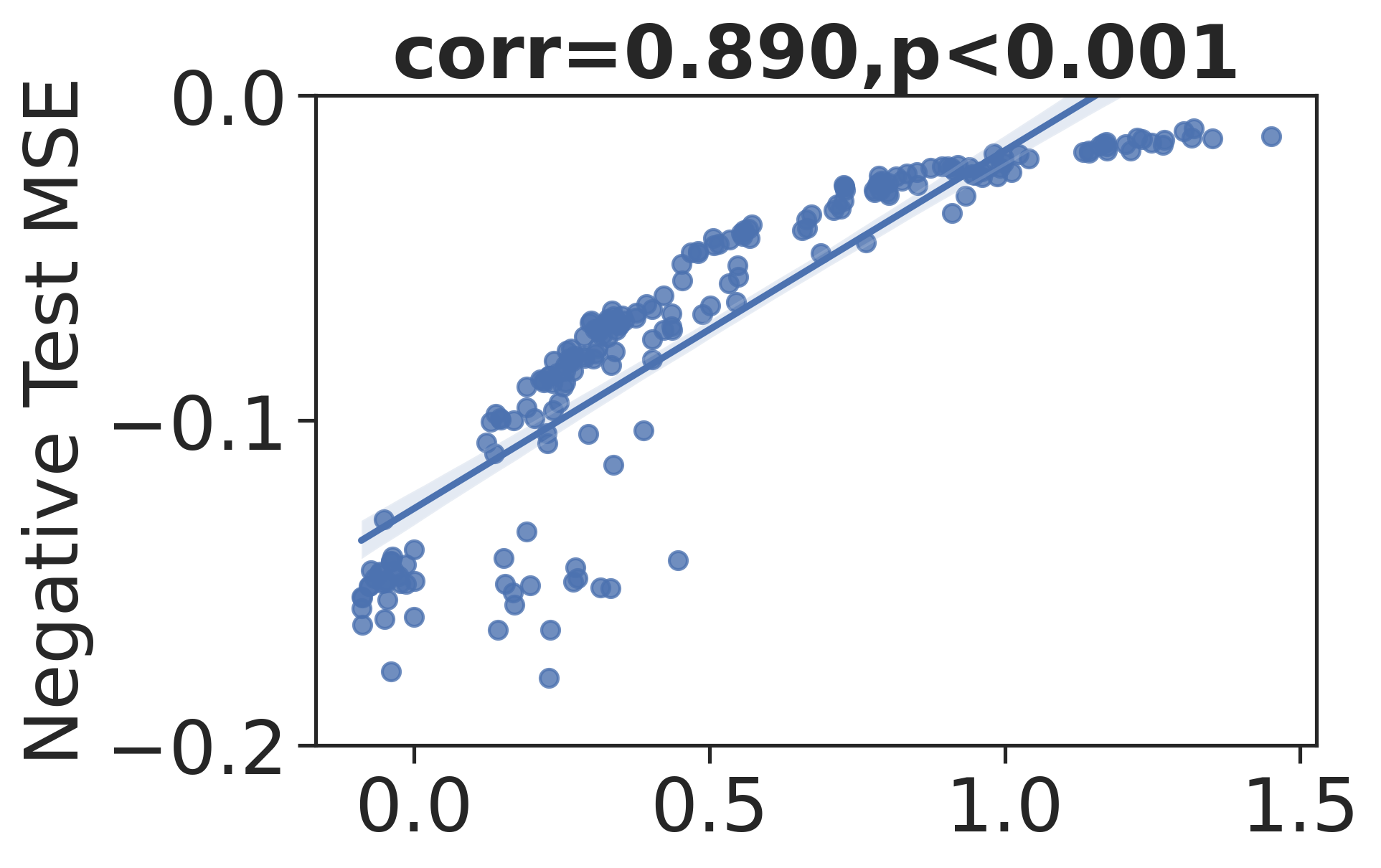}
    \caption{LogME}
    \end{subfigure}
    {\hskip 4pt}
    \begin{subfigure}[b]{0.23\textwidth}
    \includegraphics[width=\textwidth]{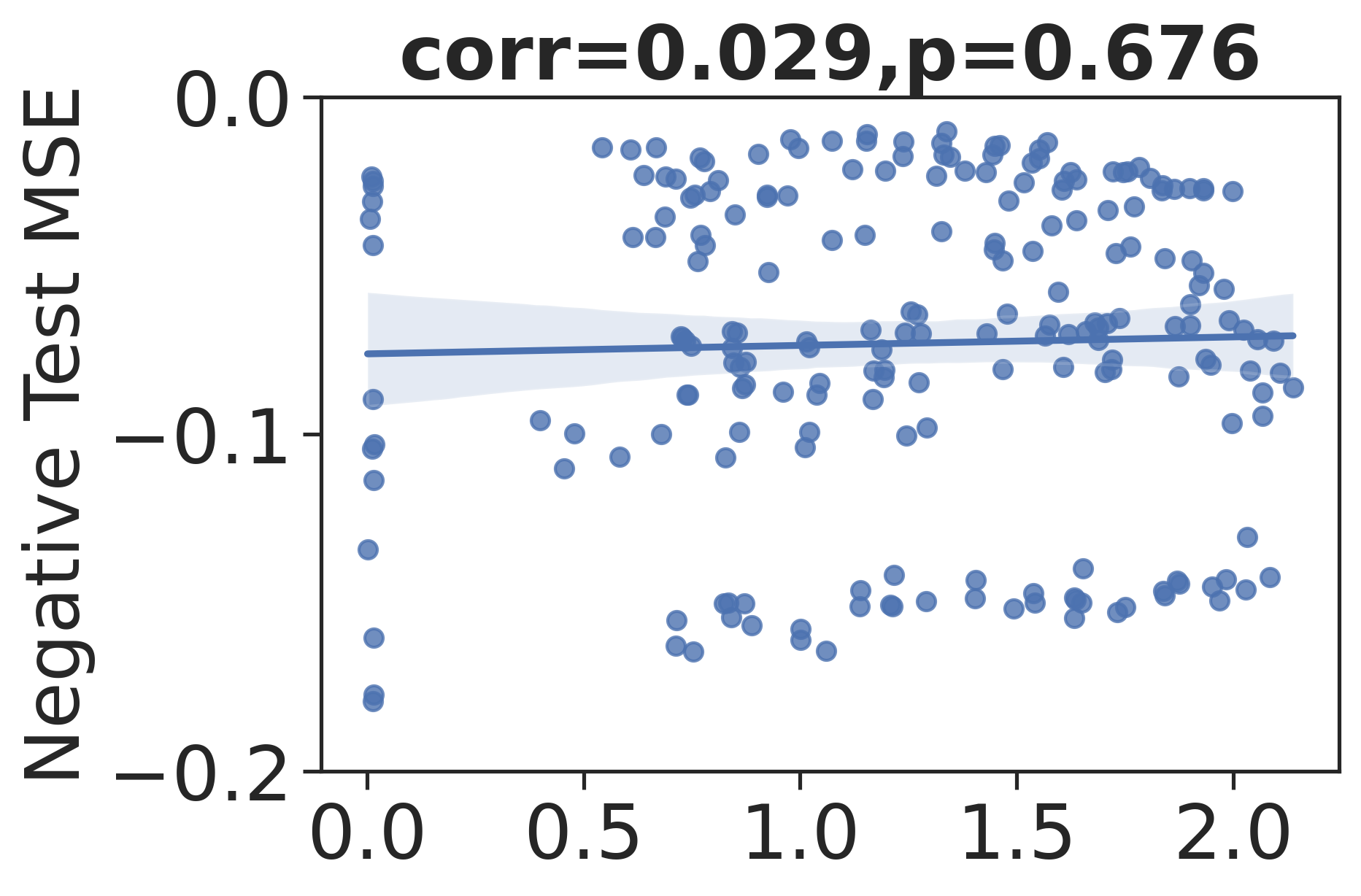}
    \caption{TransRate}
    \end{subfigure}
    {\hskip 4pt}
    \begin{subfigure}[b]{0.23\textwidth}
    \includegraphics[width=\textwidth]{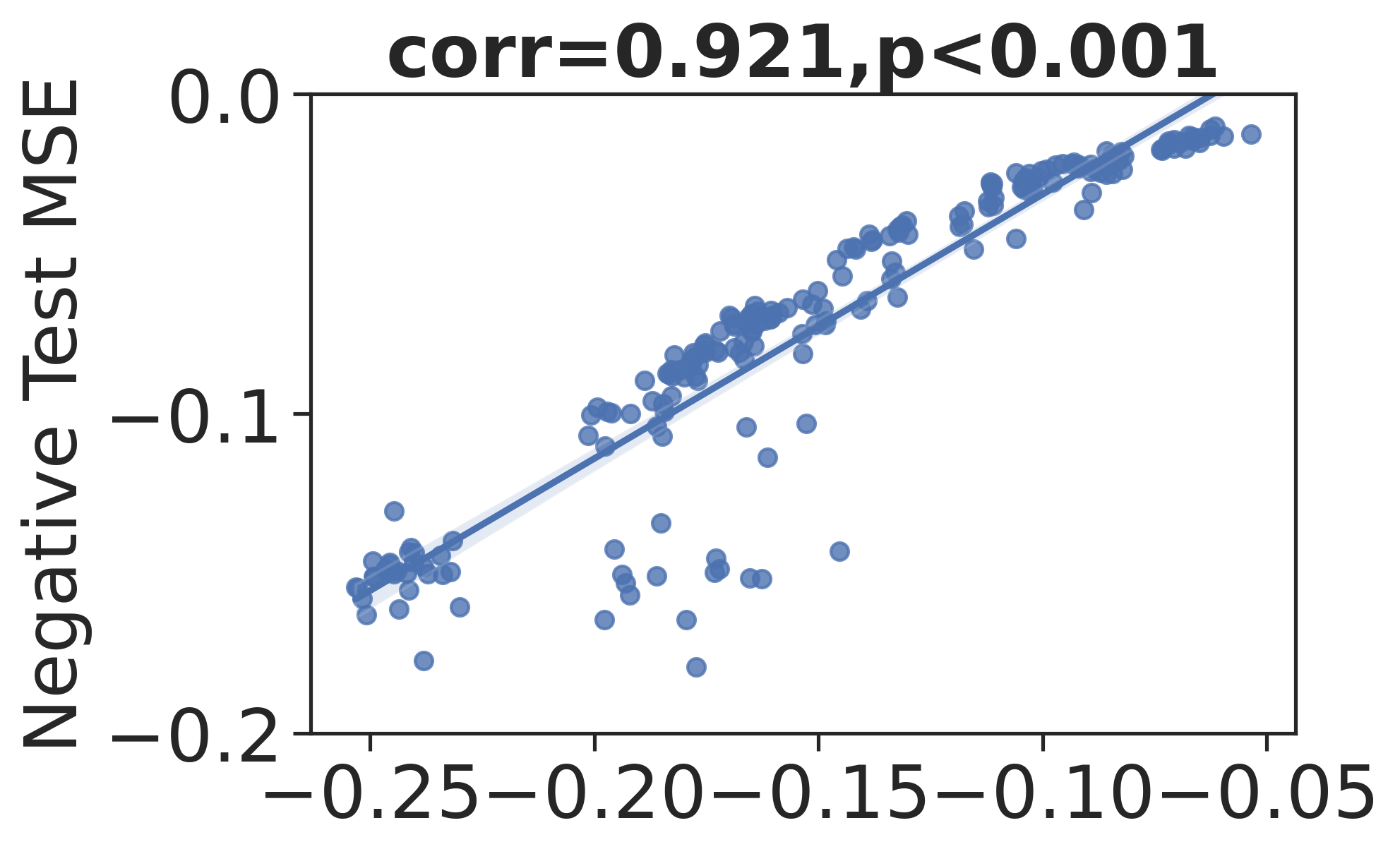}
    \caption{LinMSE0}
    \end{subfigure}
    {\hskip 4pt}
    \begin{subfigure}[b]{0.23\textwidth}
    \includegraphics[width=\textwidth]{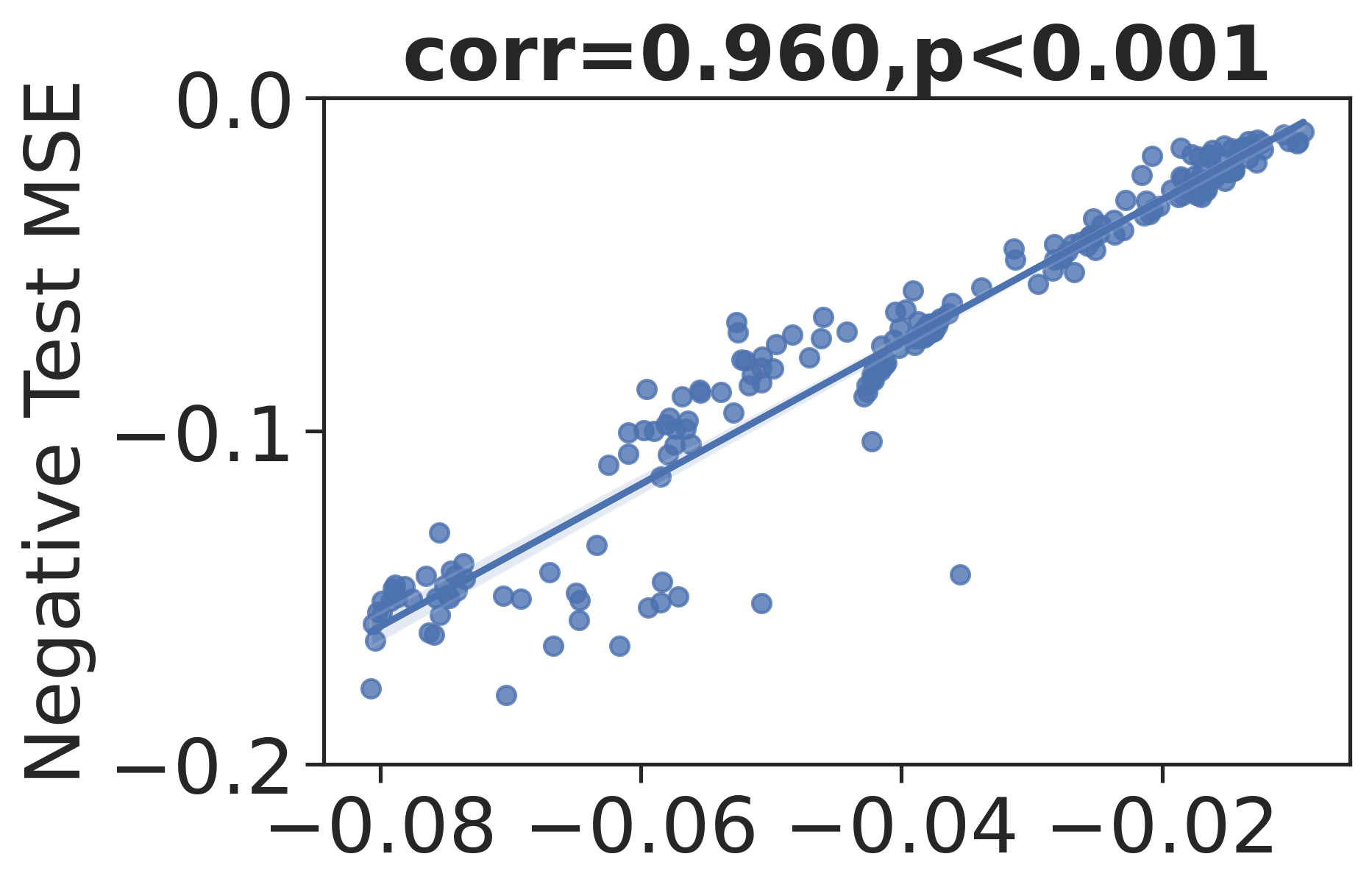}
    \caption{LinMSE1}
    \end{subfigure}
    
    {\vskip 0.2cm}
    
    \begin{subfigure}[b]{0.23\textwidth}
    \includegraphics[width=\textwidth]{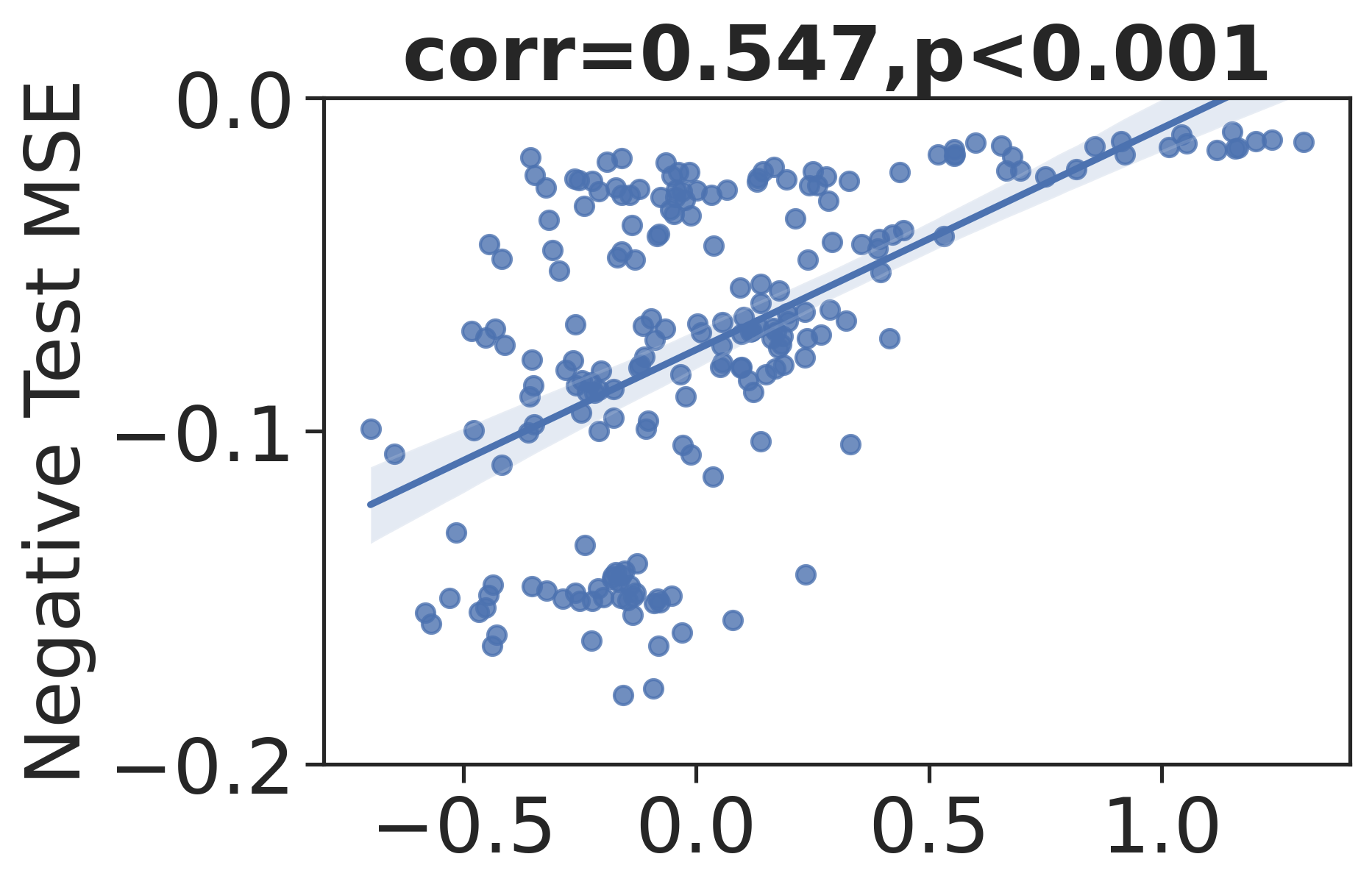}
    \caption{LabLogME}
    \end{subfigure}
    {\hskip 4pt}
    \begin{subfigure}[b]{0.23\textwidth}
    \includegraphics[width=\textwidth]{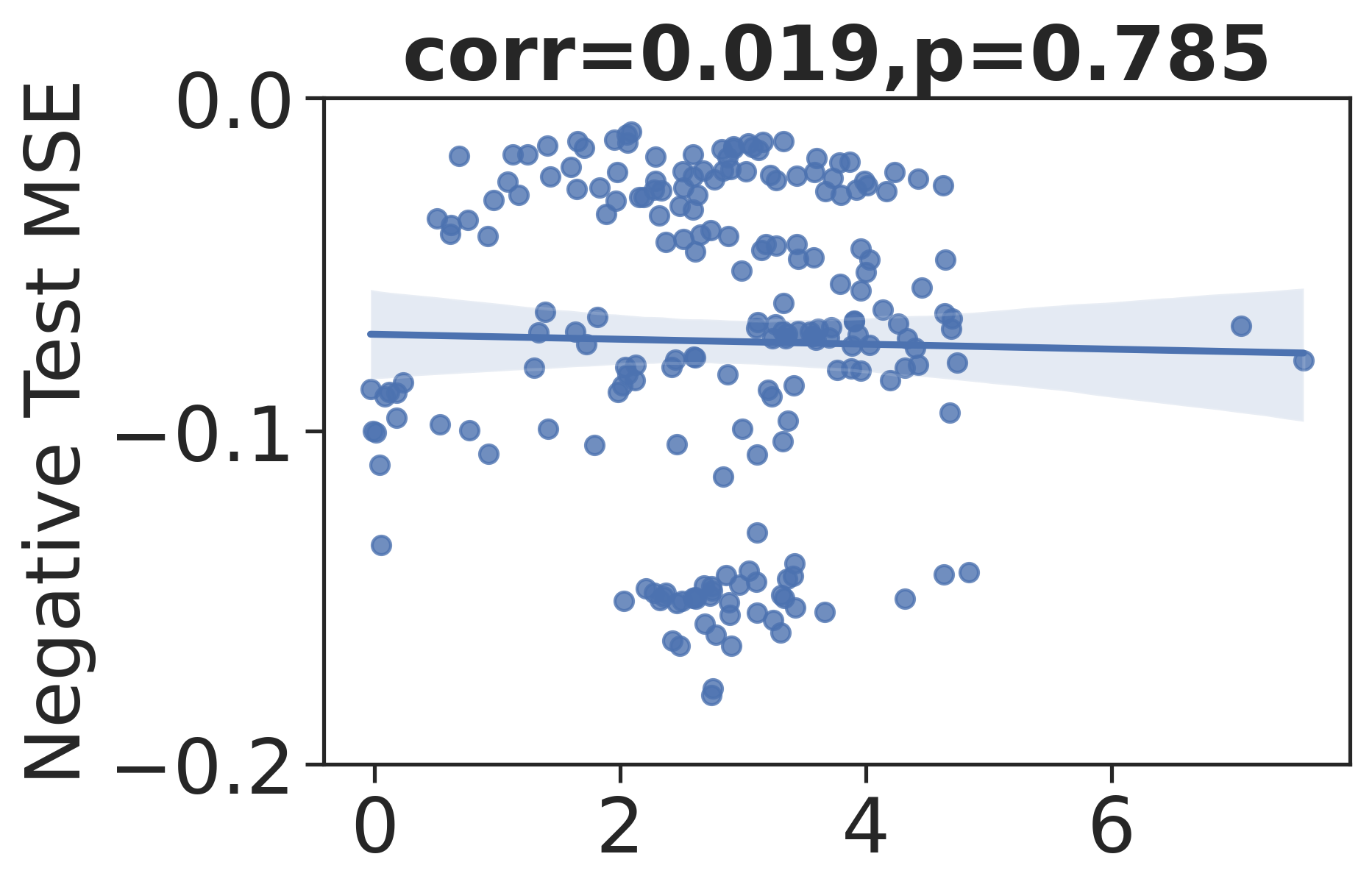}
    \caption{LabTransRate}
    \end{subfigure}
    {\hskip 4pt}
    \begin{subfigure}[b]{0.23\textwidth}
    \includegraphics[width=\textwidth]{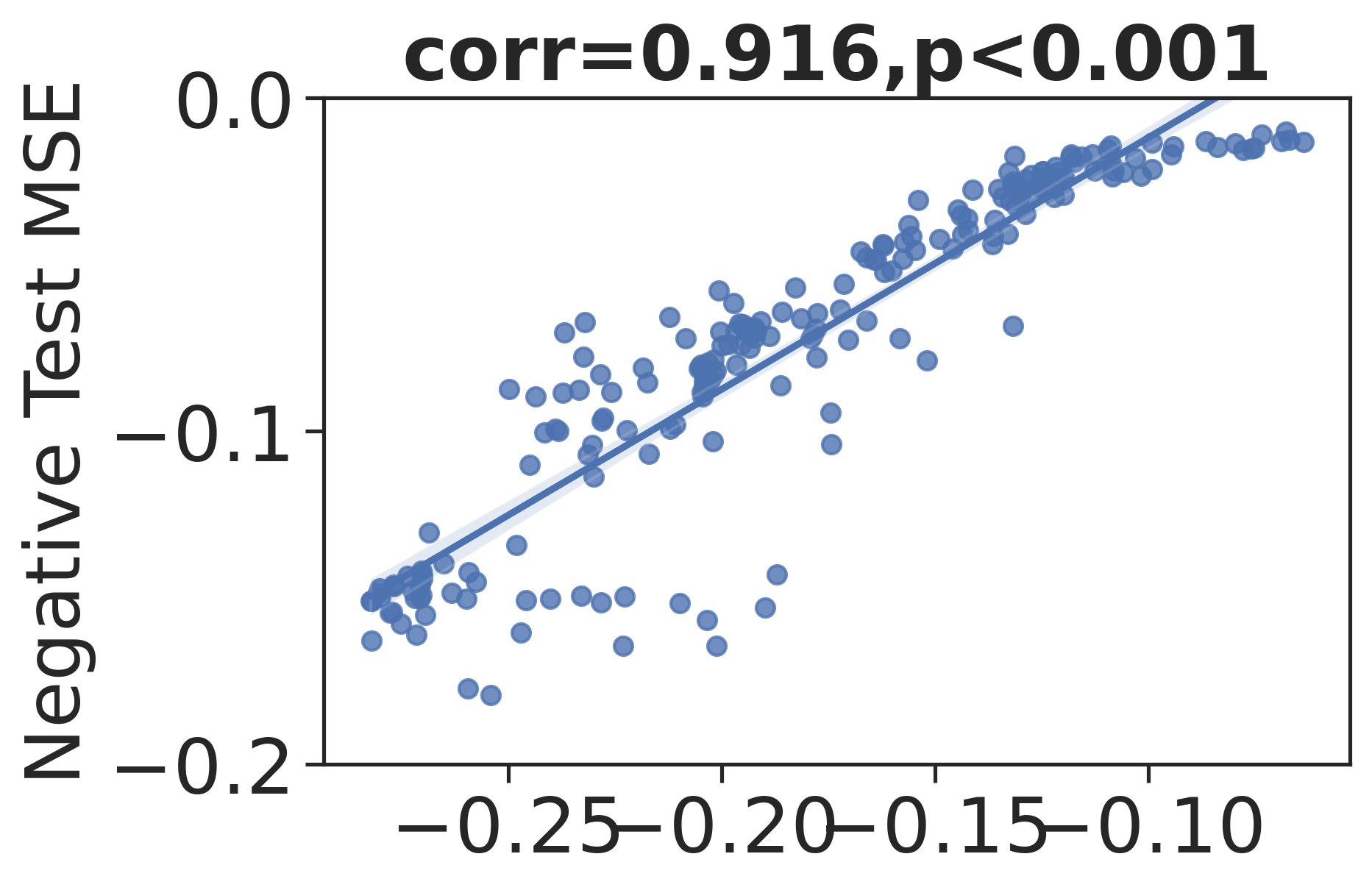}
    \caption{LabMSE0}
    \end{subfigure}
    {\hskip 4pt}
    \begin{subfigure}[b]{0.23\textwidth}
    \includegraphics[width=\textwidth]{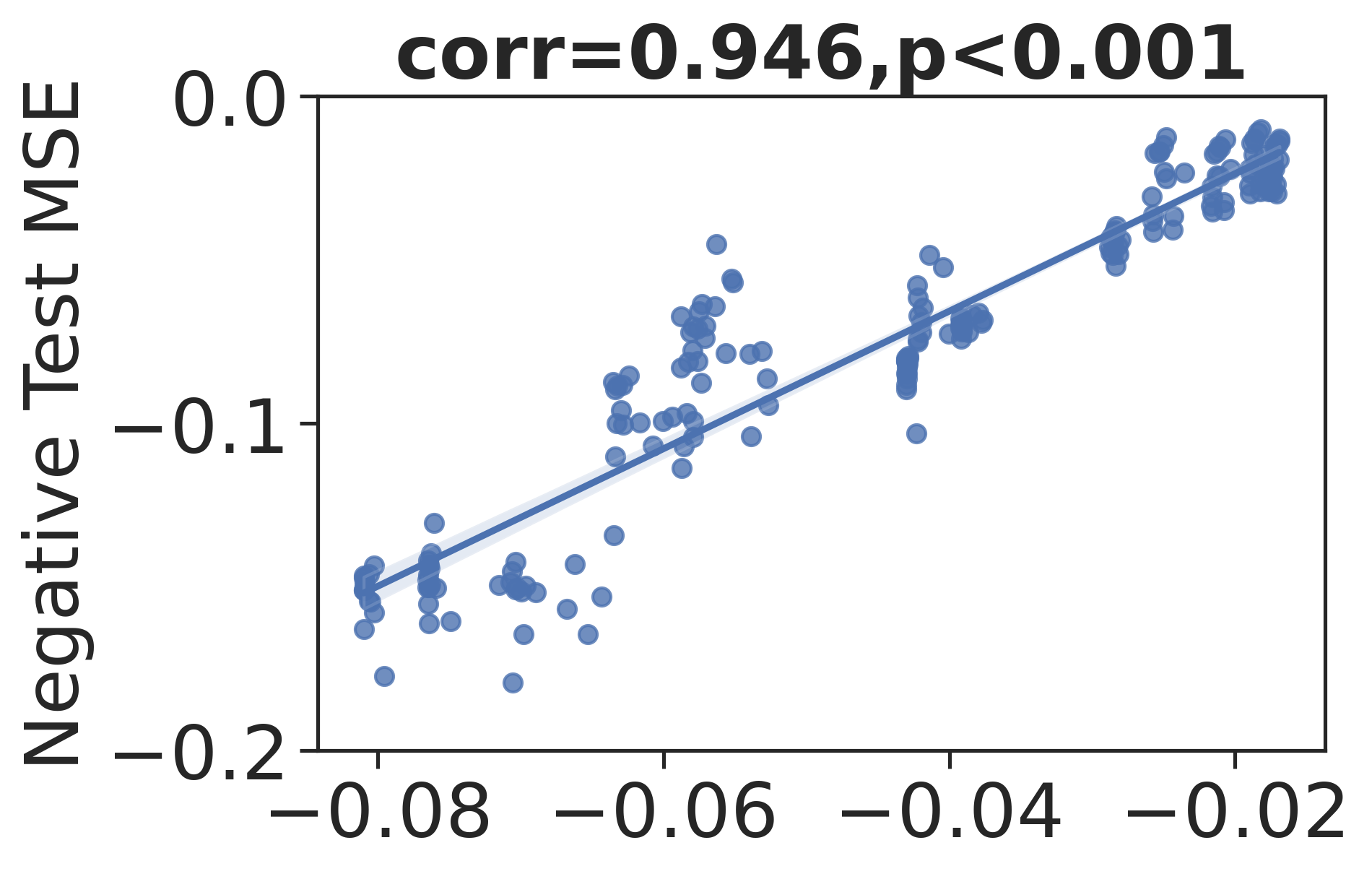}
    \caption{LabMSE1}
    \end{subfigure}
    {\vskip -0.2cm}
    \caption{\textbf{Correlation coefficients and $p$-values between transferability estimators and negative test MSEs} when transferring with head re-training between any two different keypoints (with shared inputs) on CUB-200-2011.}
    \label{fig:shared_input_cub_head_rt}
    
    {\vskip 0.4cm}
    
    \begin{subfigure}[b]{0.23\textwidth}
    \includegraphics[width=\textwidth]{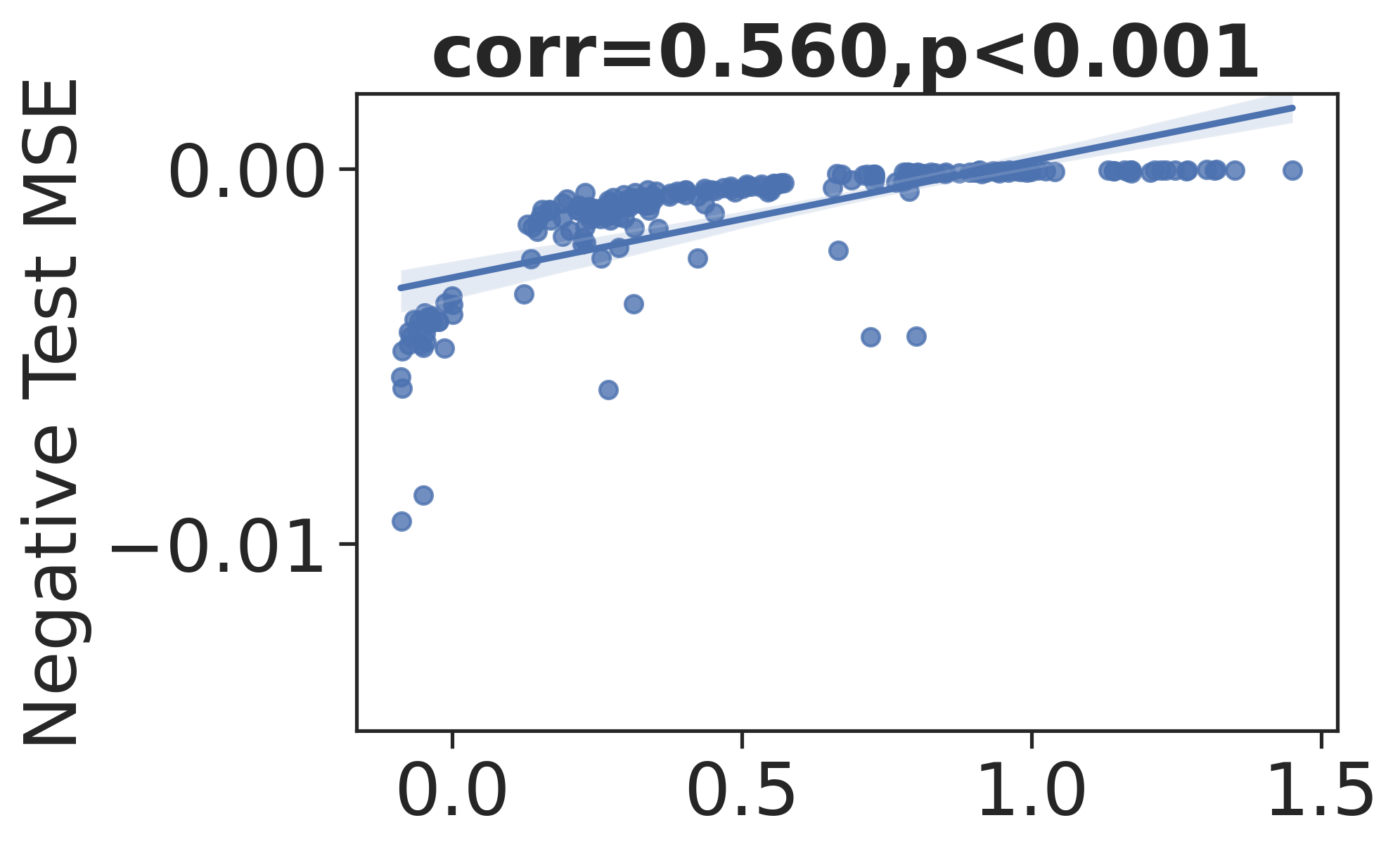}
    \caption{LogME}
    \end{subfigure}
    {\hskip 4pt}
    \begin{subfigure}[b]{0.23\textwidth}
    \includegraphics[width=\textwidth]{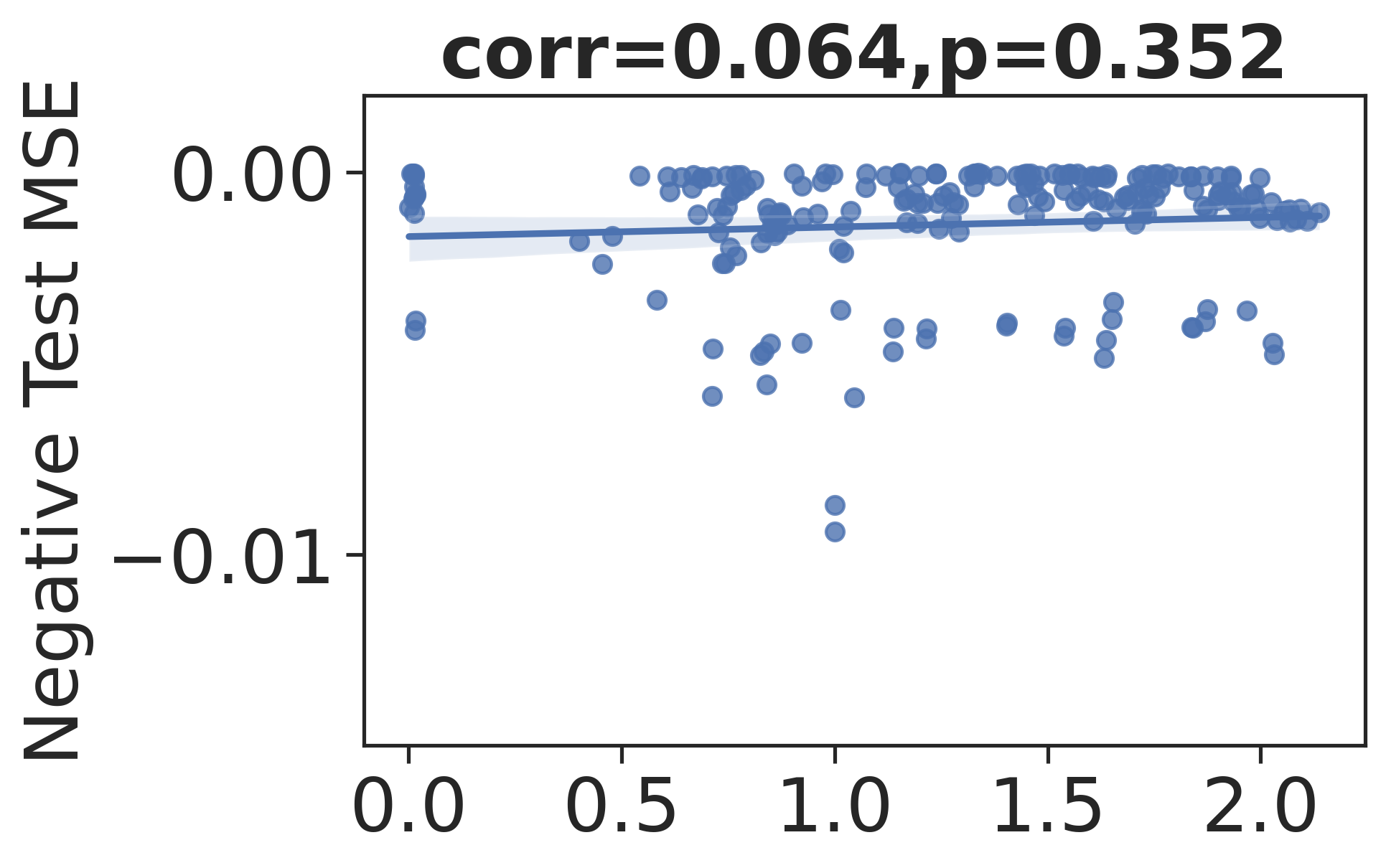}
    \caption{TransRate}
    \end{subfigure}
    {\hskip 4pt}
    \begin{subfigure}[b]{0.23\textwidth}
    \includegraphics[width=\textwidth]{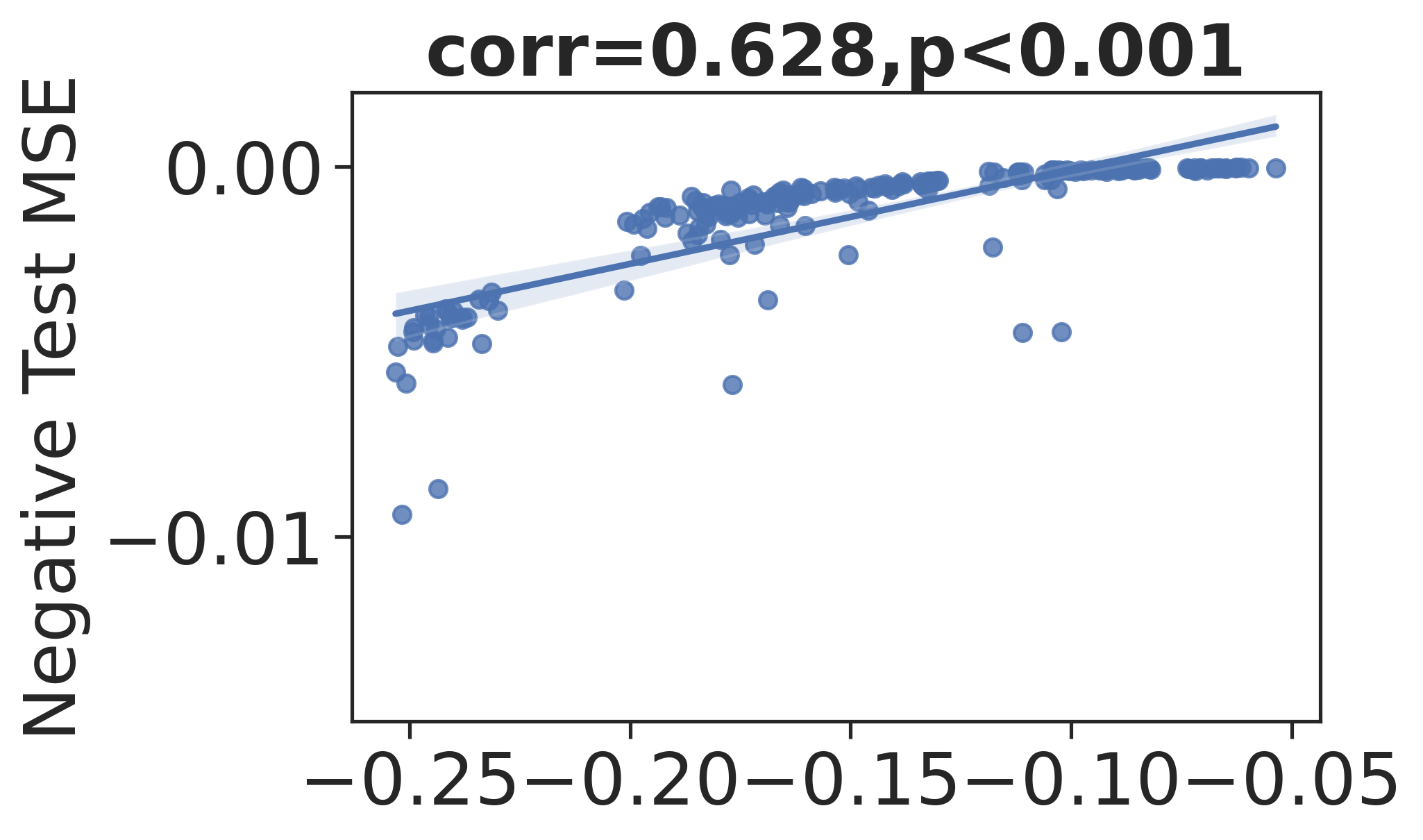}    
    \caption{LinMSE0}
    \end{subfigure}
    {\hskip 4pt}
    \begin{subfigure}[b]{0.23\textwidth}
    \includegraphics[width=\textwidth]{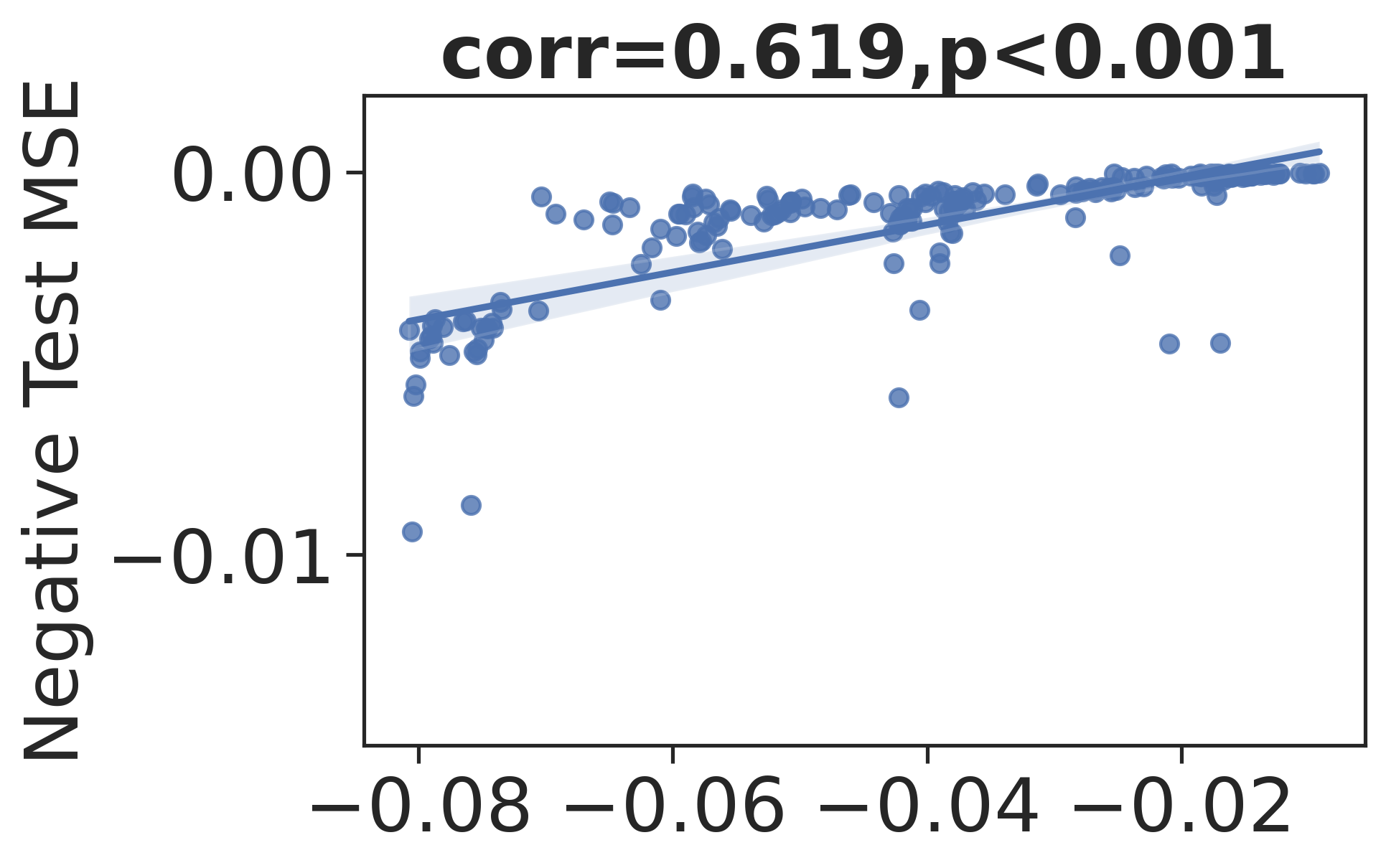}
    \caption{LinMSE1}
    \end{subfigure}
    
    {\vskip 0.2cm}
    
    \begin{subfigure}[b]{0.23\textwidth}
    \includegraphics[width=\textwidth]{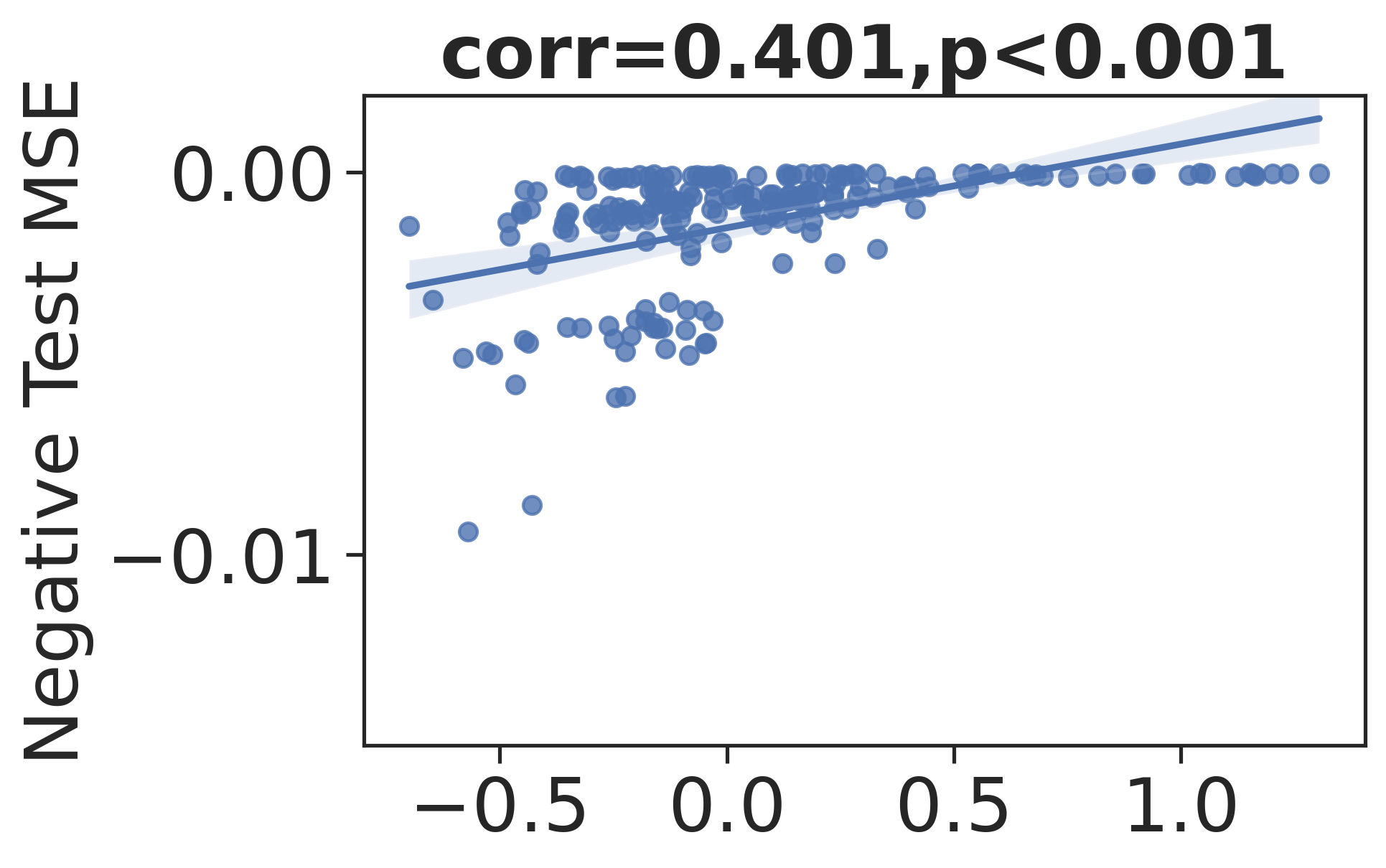}
    \caption{LabLogME}
    \end{subfigure}
    {\hskip 4pt}
    \begin{subfigure}[b]{0.23\textwidth}
    \includegraphics[width=\textwidth]{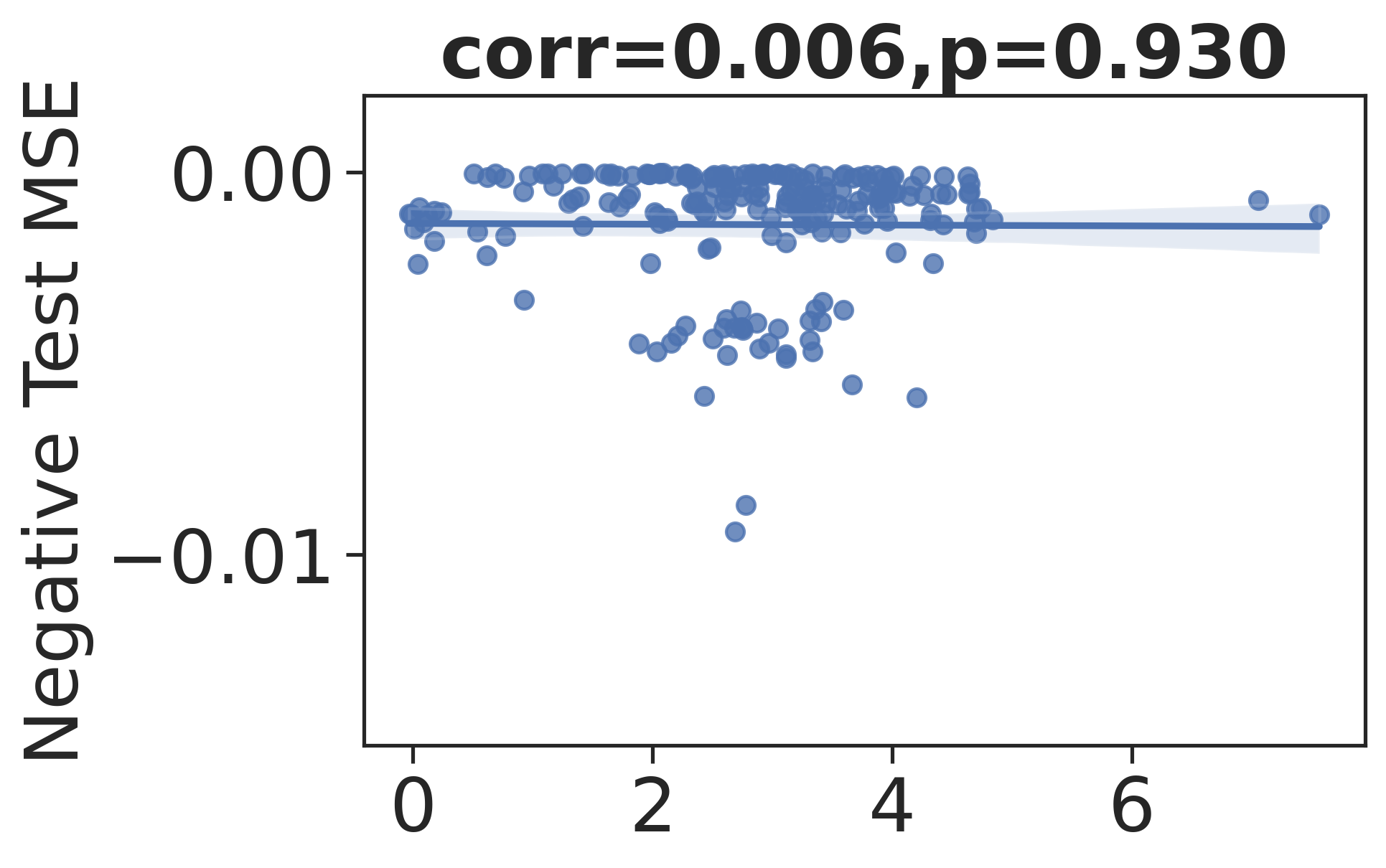}
    \caption{LabTransRate}
    \end{subfigure}
    {\hskip 4pt}
    \begin{subfigure}[b]{0.23\textwidth}
\includegraphics[width=\textwidth]{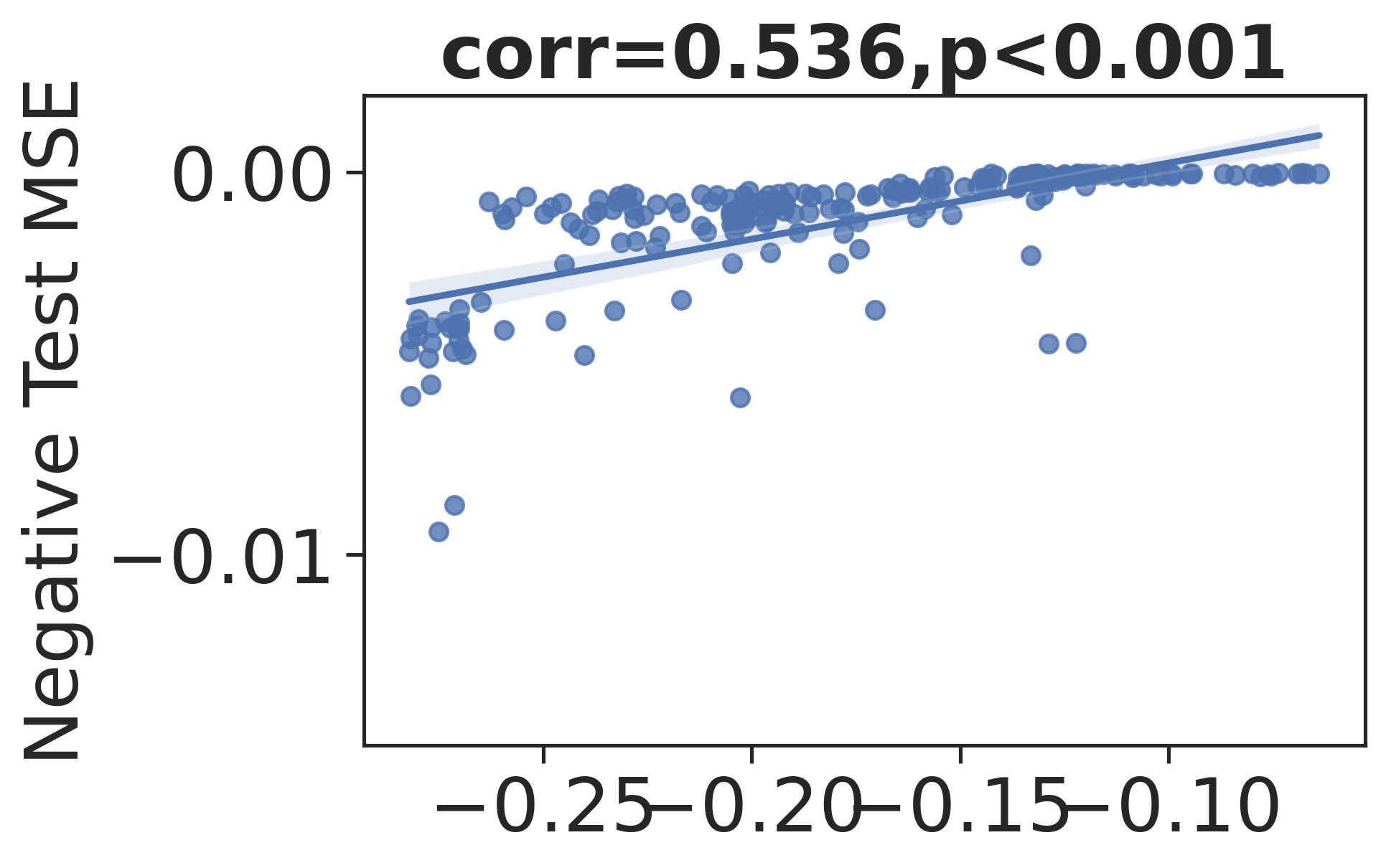}
    \caption{LabMSE0}
    \end{subfigure}
    {\hskip 4pt}
    \begin{subfigure}[b]{0.23\textwidth}
    \includegraphics[width=\textwidth]{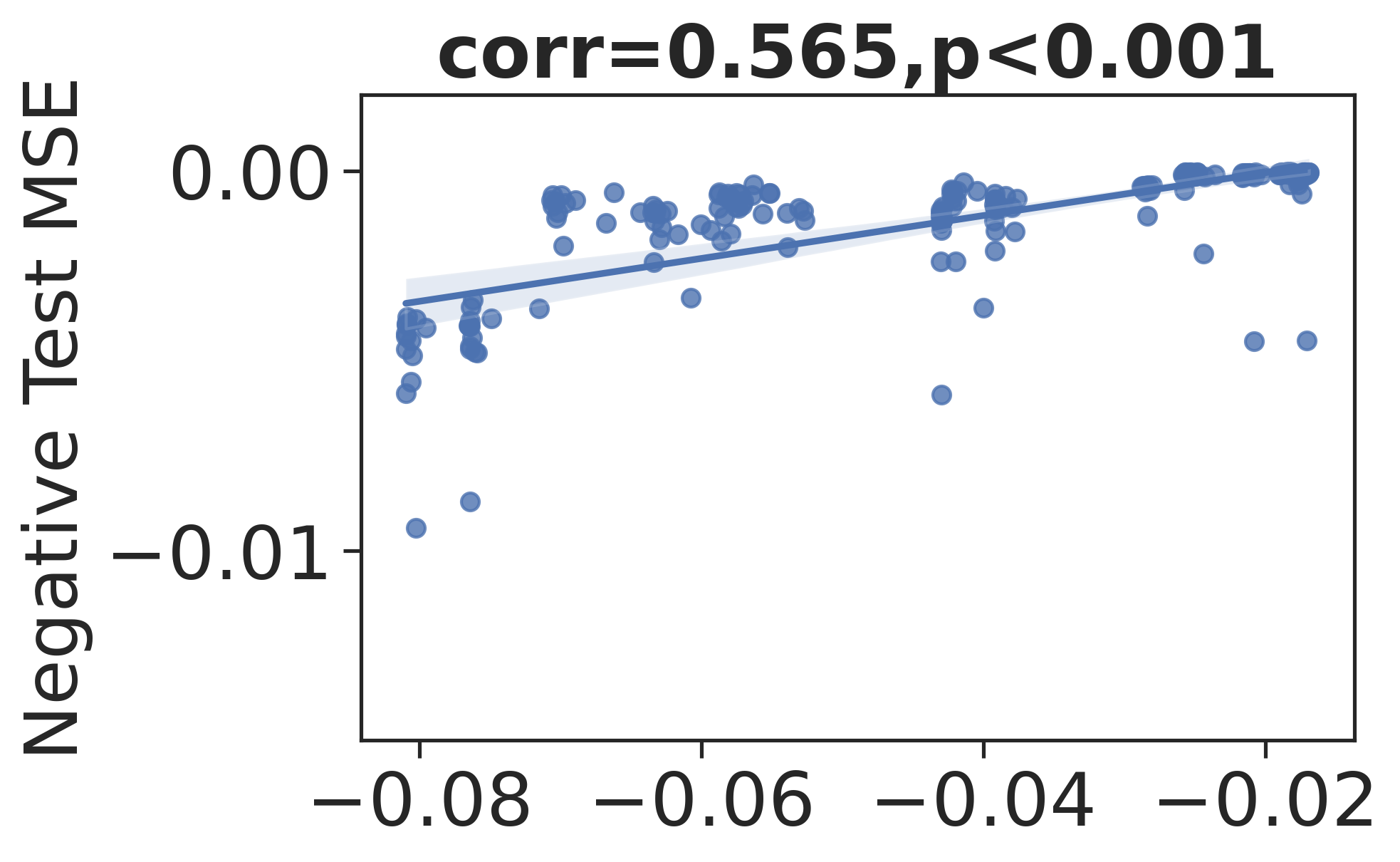}
    \caption{LabMSE1}
    \end{subfigure}
    {\vskip -0.2cm}
    \caption{\textbf{Correlation coefficients and $p$-values between transferability estimators and negative test MSEs} when transferring with half fine-tuning between any two different keypoints (with shared inputs) on CUB-200-2011.}
    \label{fig:shared_input_cub_half_ft}

    {\vskip 0.4cm}

    \begin{subfigure}[b]{0.23\textwidth}
    \includegraphics[width=\textwidth]{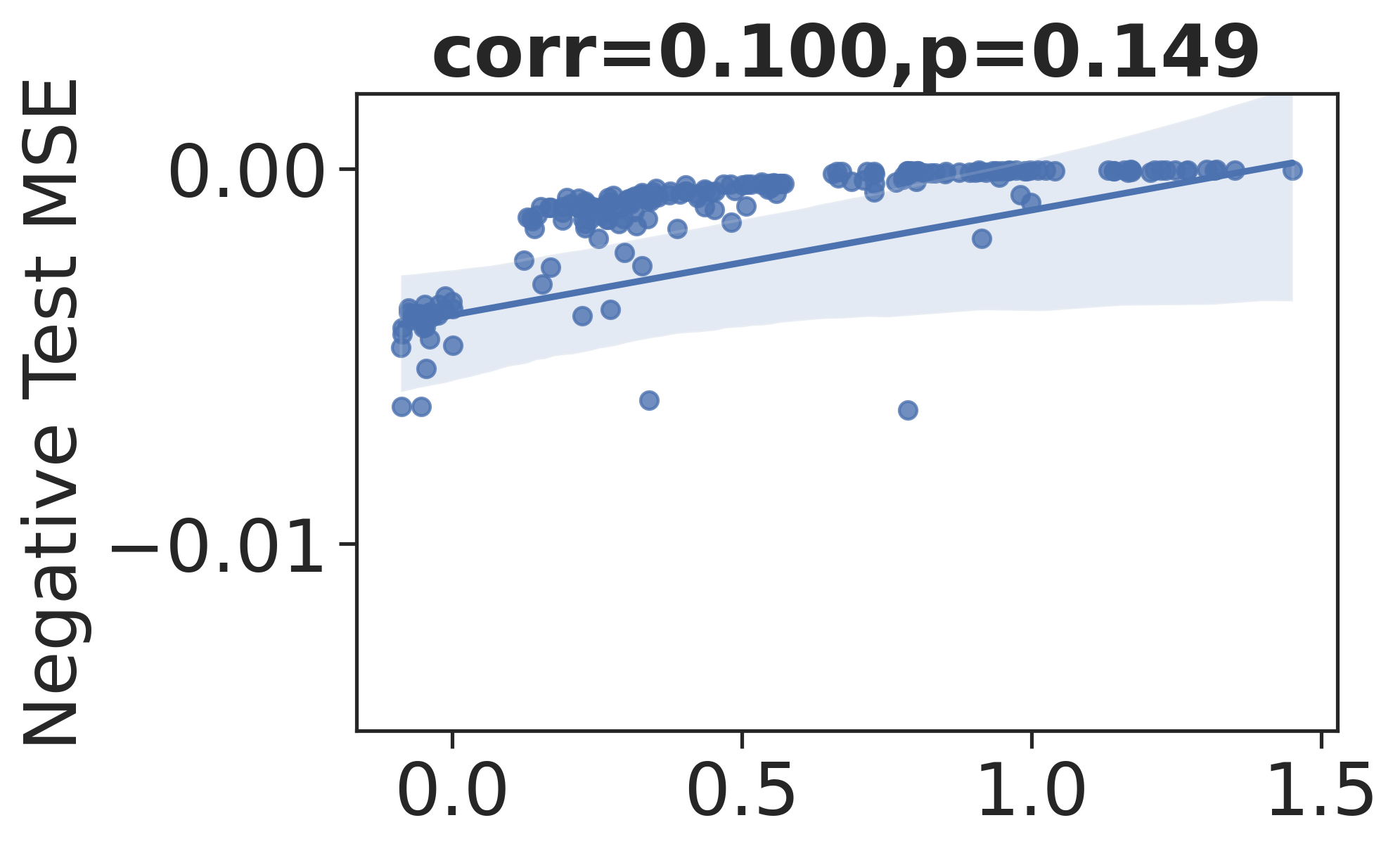}
    \caption{LogME}
    \end{subfigure}
    {\hskip 4pt}
    \begin{subfigure}[b]{0.23\textwidth}
    \includegraphics[width=\textwidth]{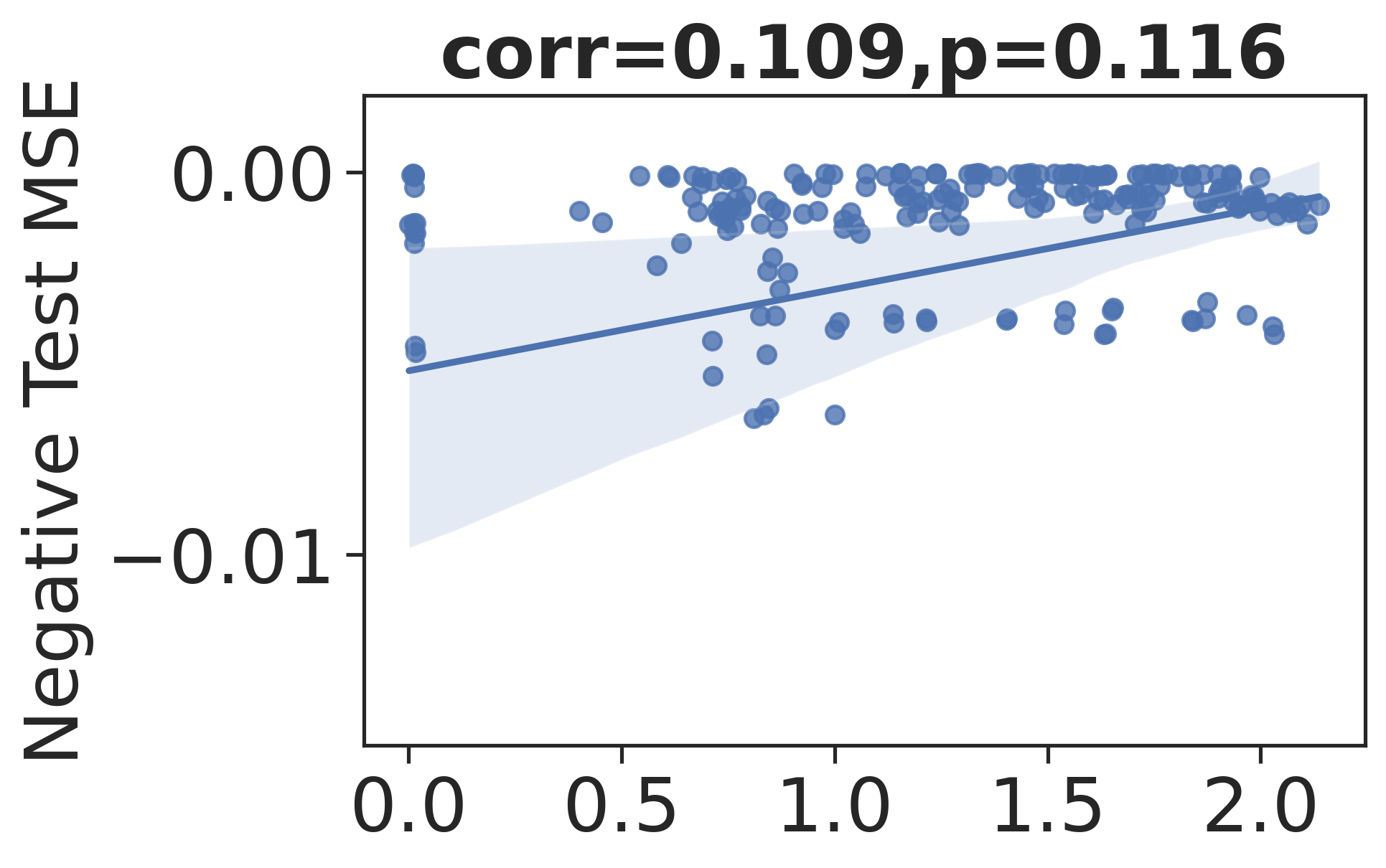}
    \caption{TransRate}
    \end{subfigure}
    {\hskip 4pt}
    \begin{subfigure}[b]{0.23\textwidth}
    \includegraphics[width=\textwidth]{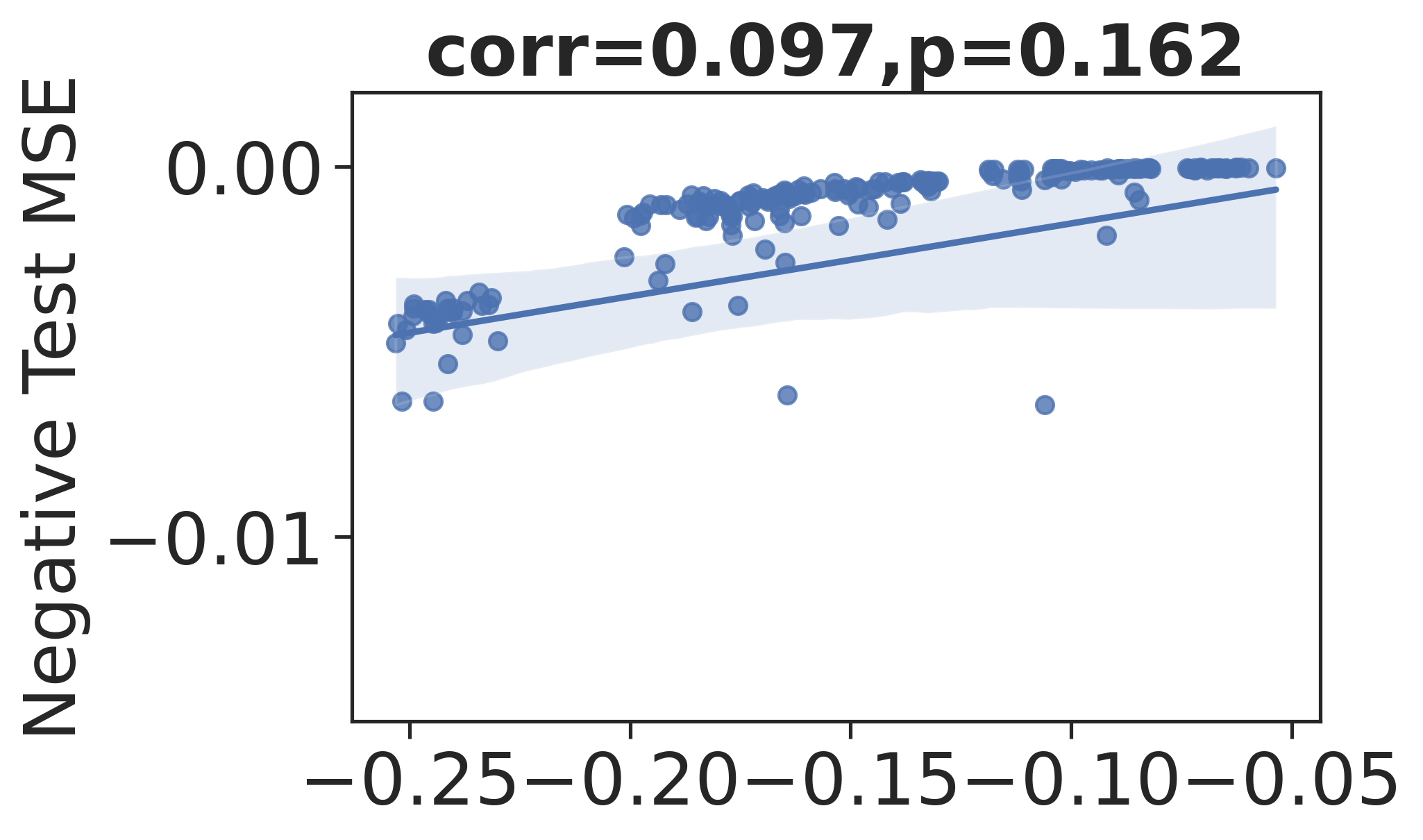}    
    \caption{LinMSE0}
    \end{subfigure}
    {\hskip 4pt}
    \begin{subfigure}[b]{0.23\textwidth}
    \includegraphics[width=\textwidth]{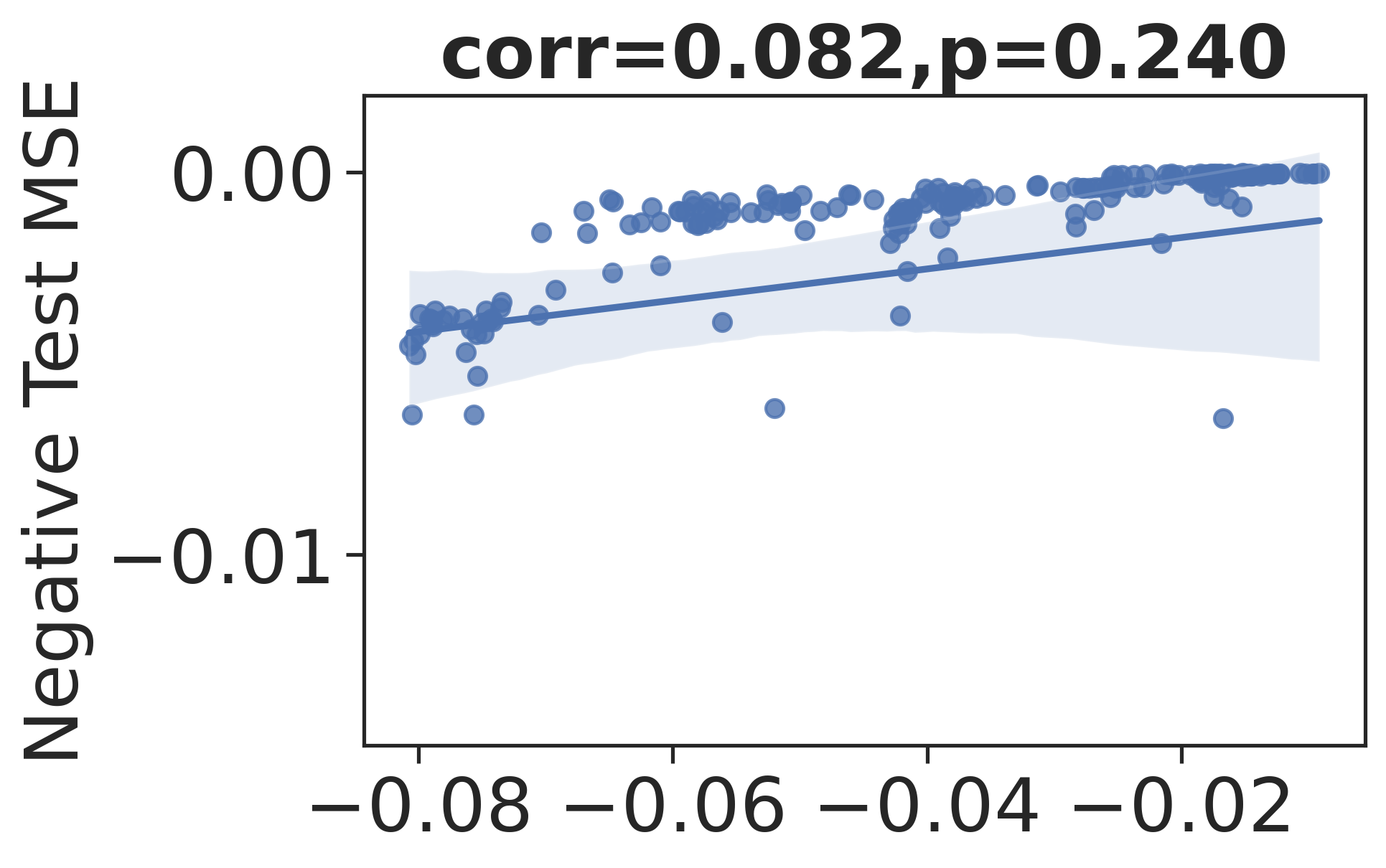}
    \caption{LinMSE1}
    \end{subfigure}
    
    {\vskip 0.2cm}
    
    \begin{subfigure}[b]{0.23\textwidth}
    \includegraphics[width=\textwidth]{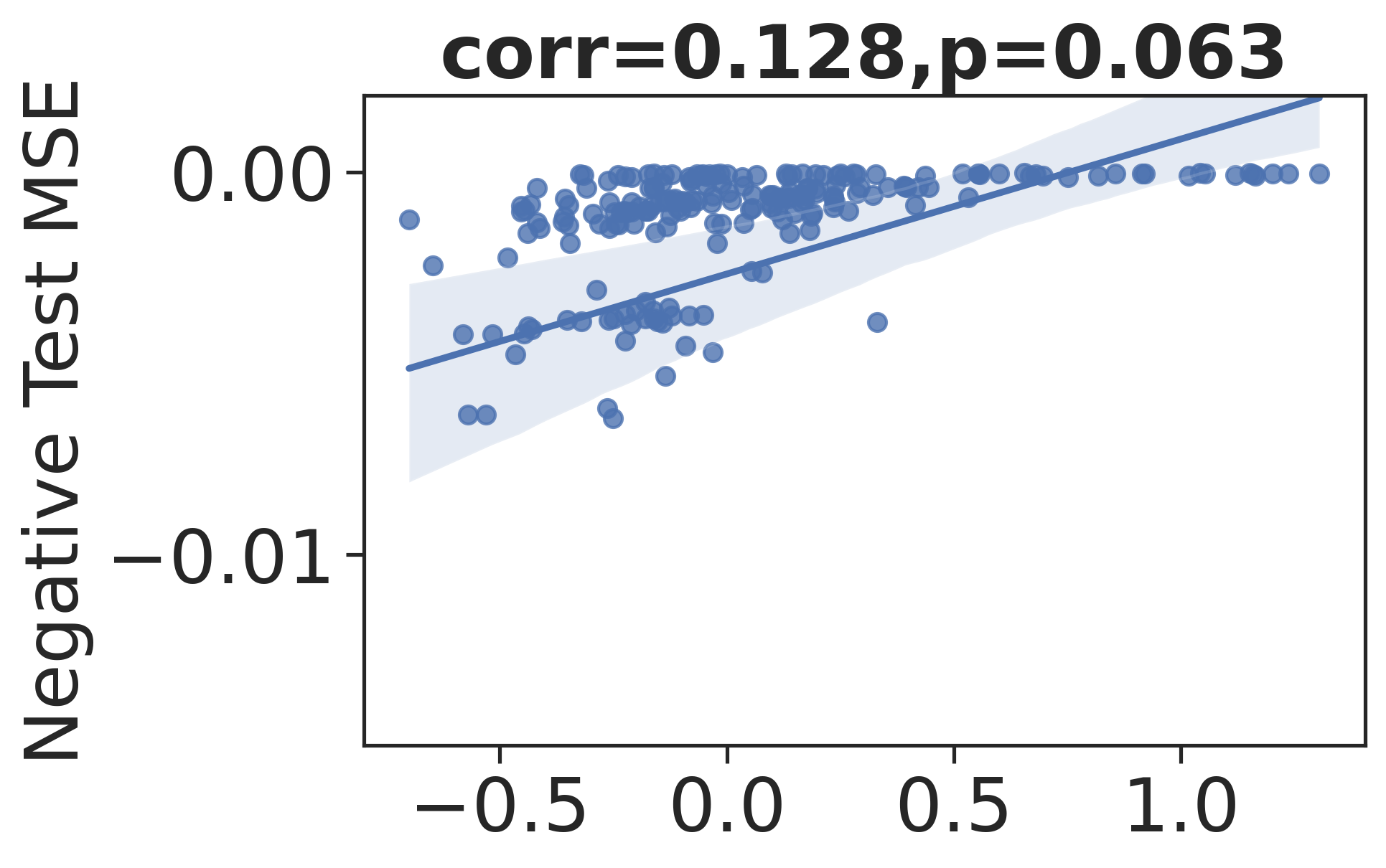}
    \caption{LabLogME}
    \end{subfigure}
    {\hskip 4pt}
    \begin{subfigure}[b]{0.23\textwidth}
    \includegraphics[width=\textwidth]{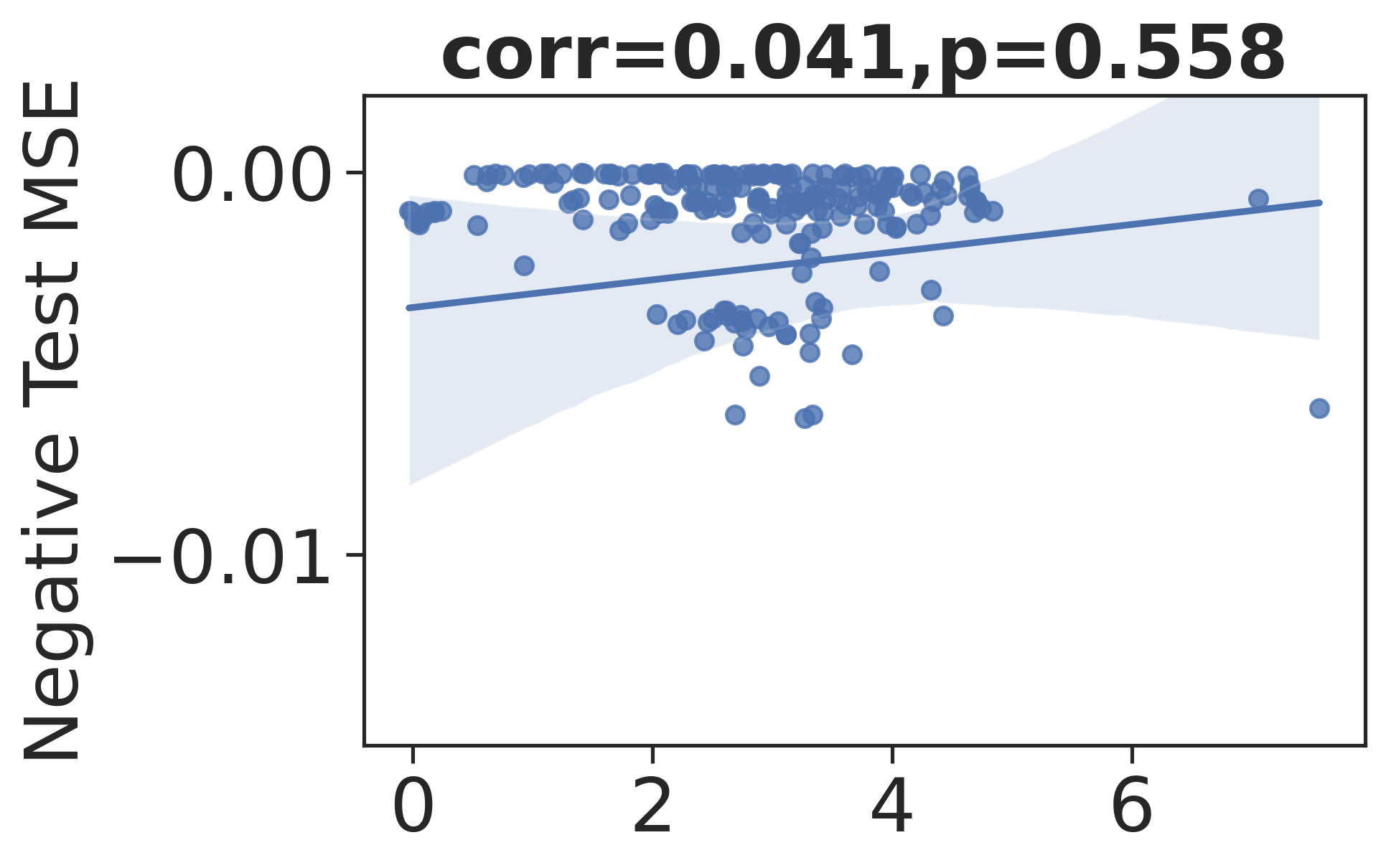}
    \caption{LabTransRate}
    \end{subfigure}
    {\hskip 4pt}
    \begin{subfigure}[b]{0.23\textwidth}
\includegraphics[width=\textwidth]{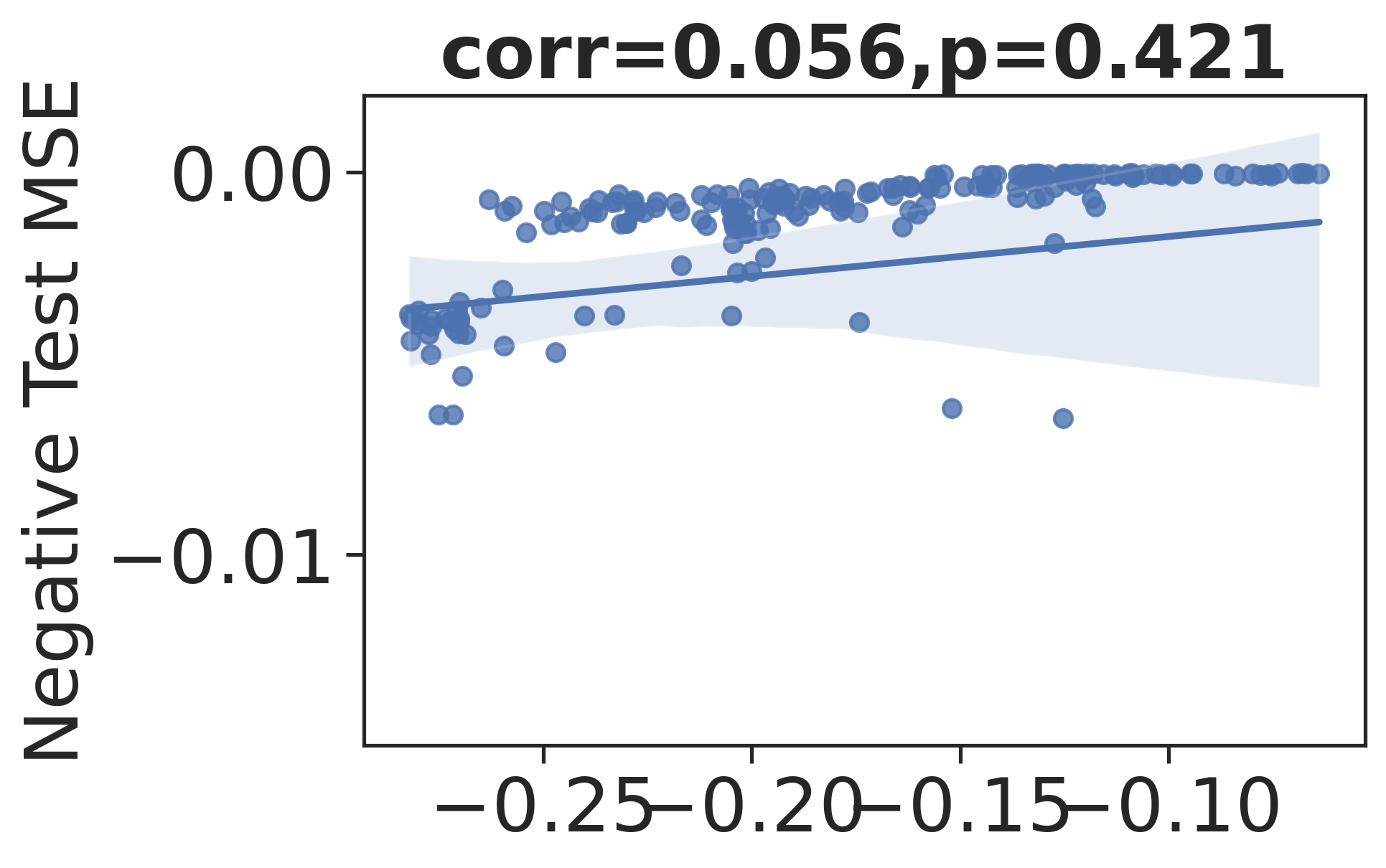}
    \caption{LabMSE0}
    \end{subfigure}
    {\hskip 4pt}
    \begin{subfigure}[b]{0.23\textwidth}
    \includegraphics[width=\textwidth]{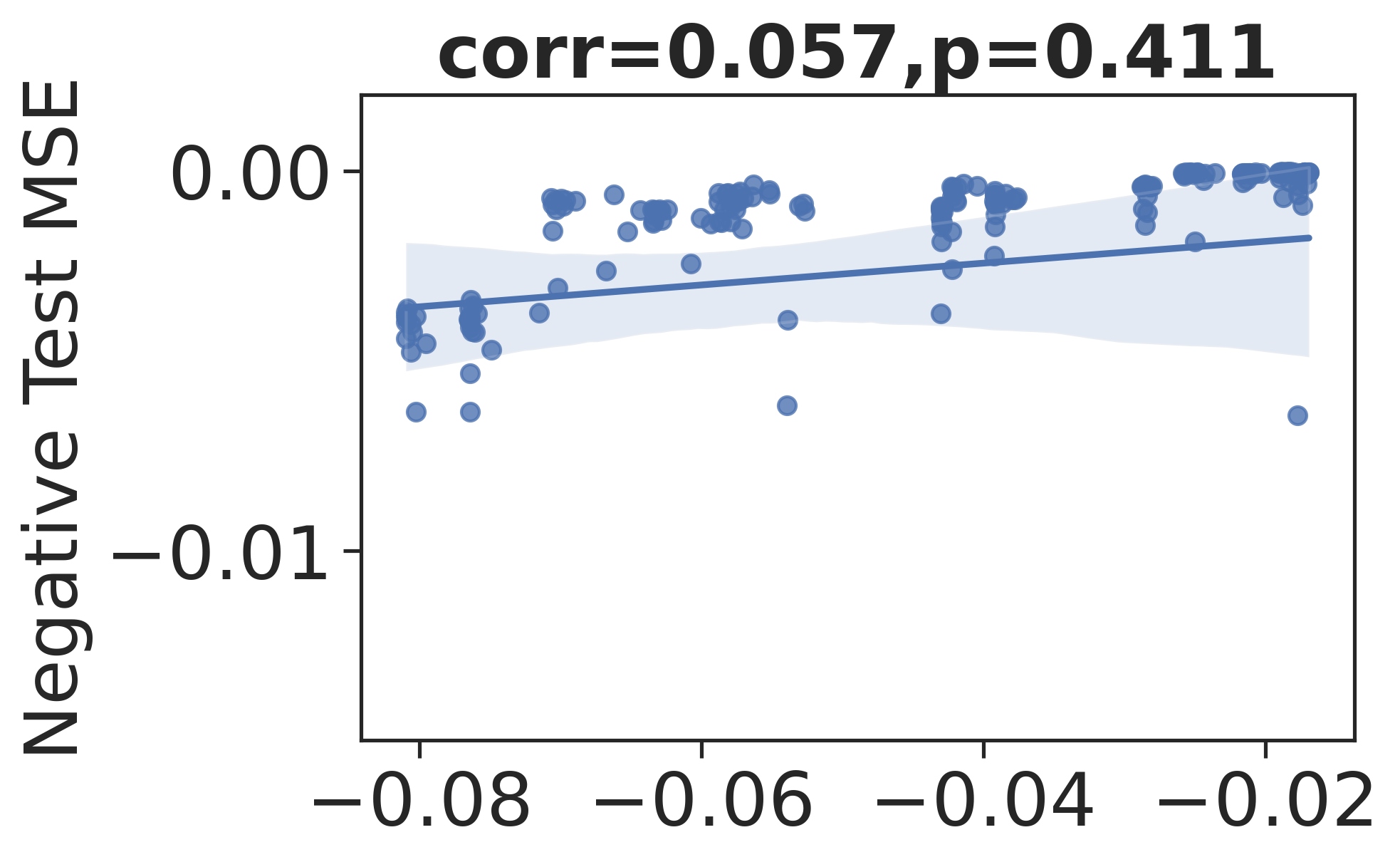}
    \caption{LabMSE1}
    \end{subfigure}
    {\vskip -0.2cm}
    \caption{\textbf{Correlation coefficients and $p$-values between transferability estimators and negative test MSEs} when transferring with full fine-tuning between any two different keypoints (with shared inputs) on CUB-200-2011.}
    \label{fig:shared_input_cub_full_ft}
\end{figure*}

\begin{figure*}[h]
\captionsetup[subfigure]{justification=centering}
    \begin{subfigure}[b]{0.23\textwidth}
    \includegraphics[width=\textwidth]{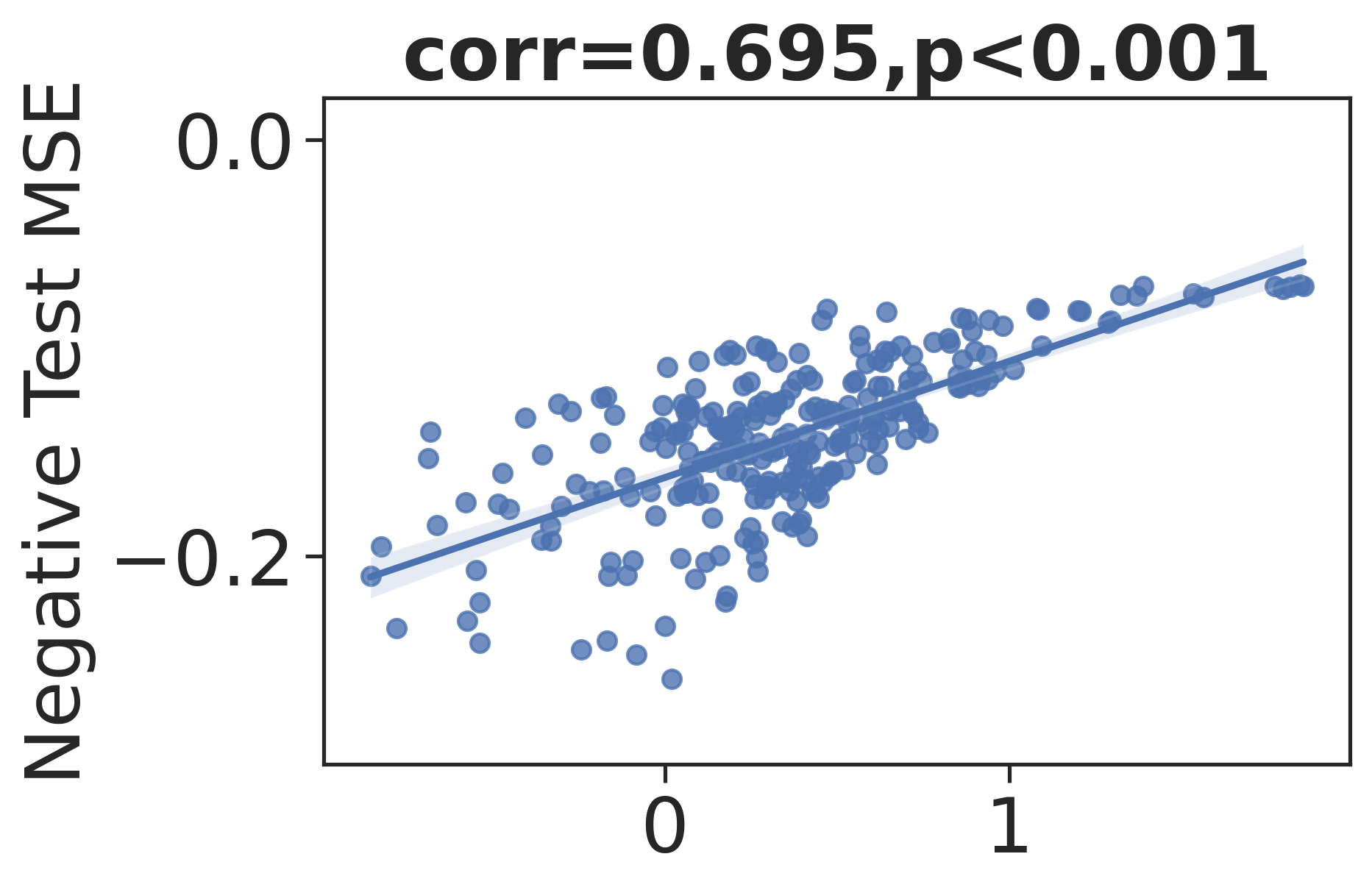}
    \caption{LogME}
    \end{subfigure}
    {\hskip 4pt}
    \begin{subfigure}[b]{0.23\textwidth}
    \includegraphics[width=\textwidth]{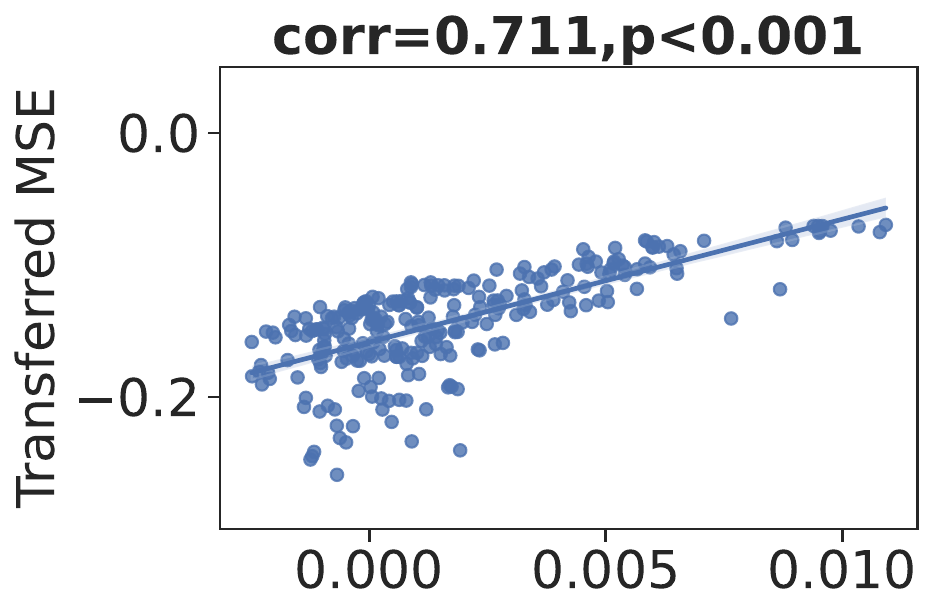}
    \caption{TransRate}
    \end{subfigure}
    {\hskip 4pt}
    \begin{subfigure}[b]{0.23\textwidth}
    \includegraphics[width=\textwidth]{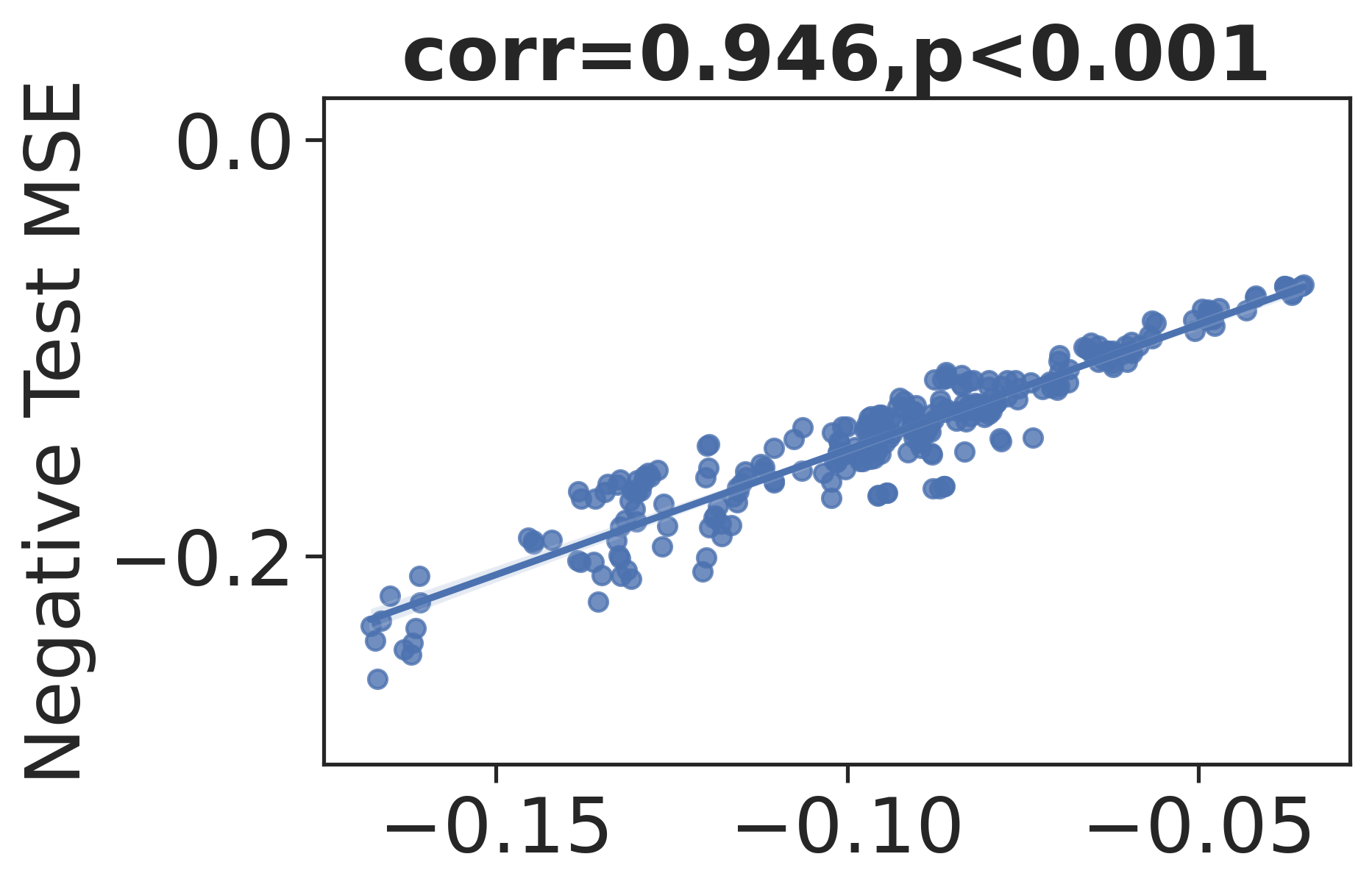}    
    \caption{LinMSE0}
    \end{subfigure}
    {\hskip 4pt}
    \begin{subfigure}[b]{0.23\textwidth}
    \includegraphics[width=\textwidth]{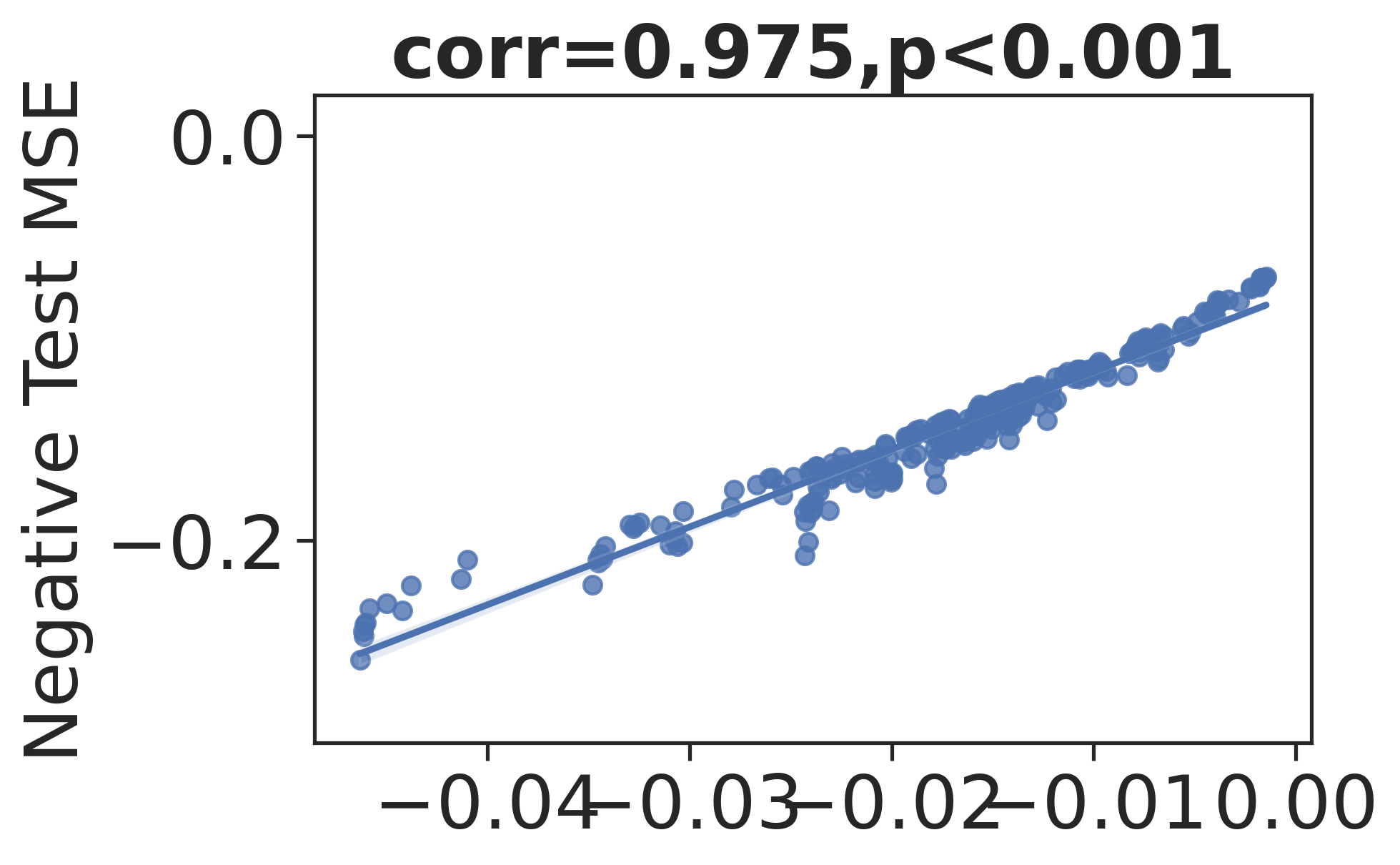}
    \caption{LinMSE1}
    \end{subfigure}
    
    {\vskip 0.2cm}
    
    \begin{subfigure}[b]{0.23\textwidth}
    \includegraphics[width=\textwidth]{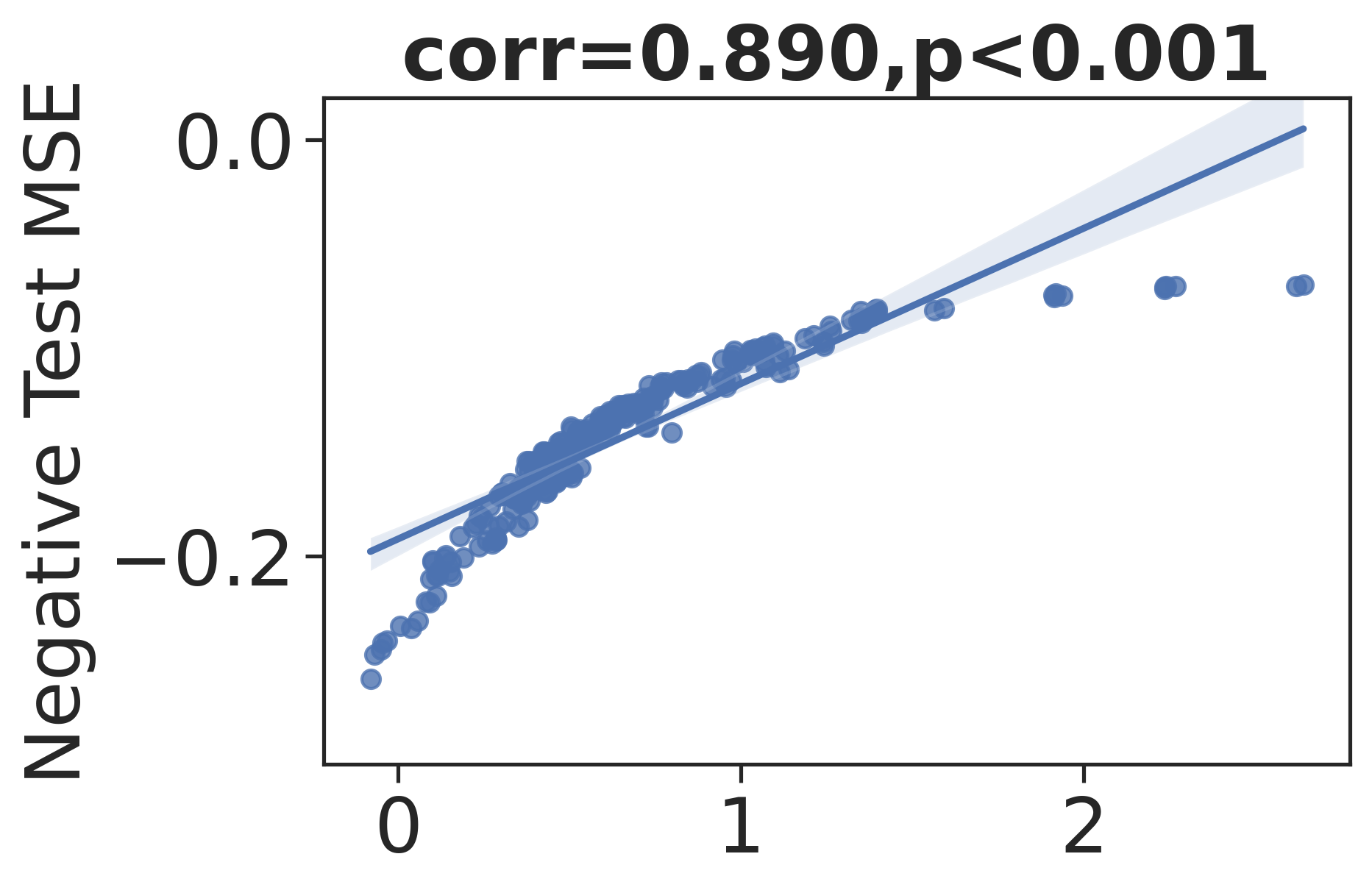}
    \caption{LabLogME}
    \end{subfigure}
    {\hskip 4pt}
    \begin{subfigure}[b]{0.23\textwidth}
    \includegraphics[width=\textwidth]{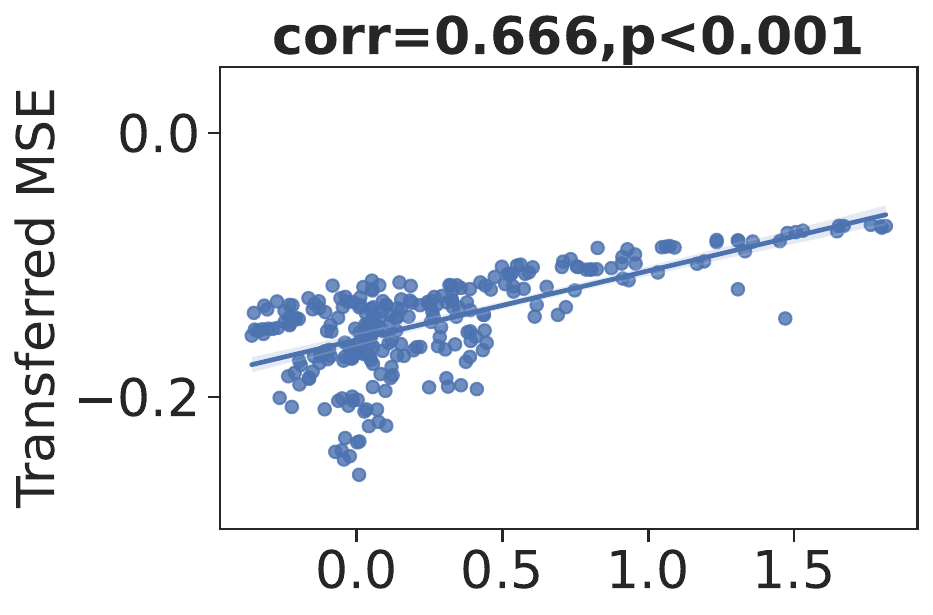}
    \caption{LabTransRate}
    \end{subfigure}
    {\hskip 4pt}
    \begin{subfigure}[b]{0.23\textwidth}
\includegraphics[width=\textwidth]{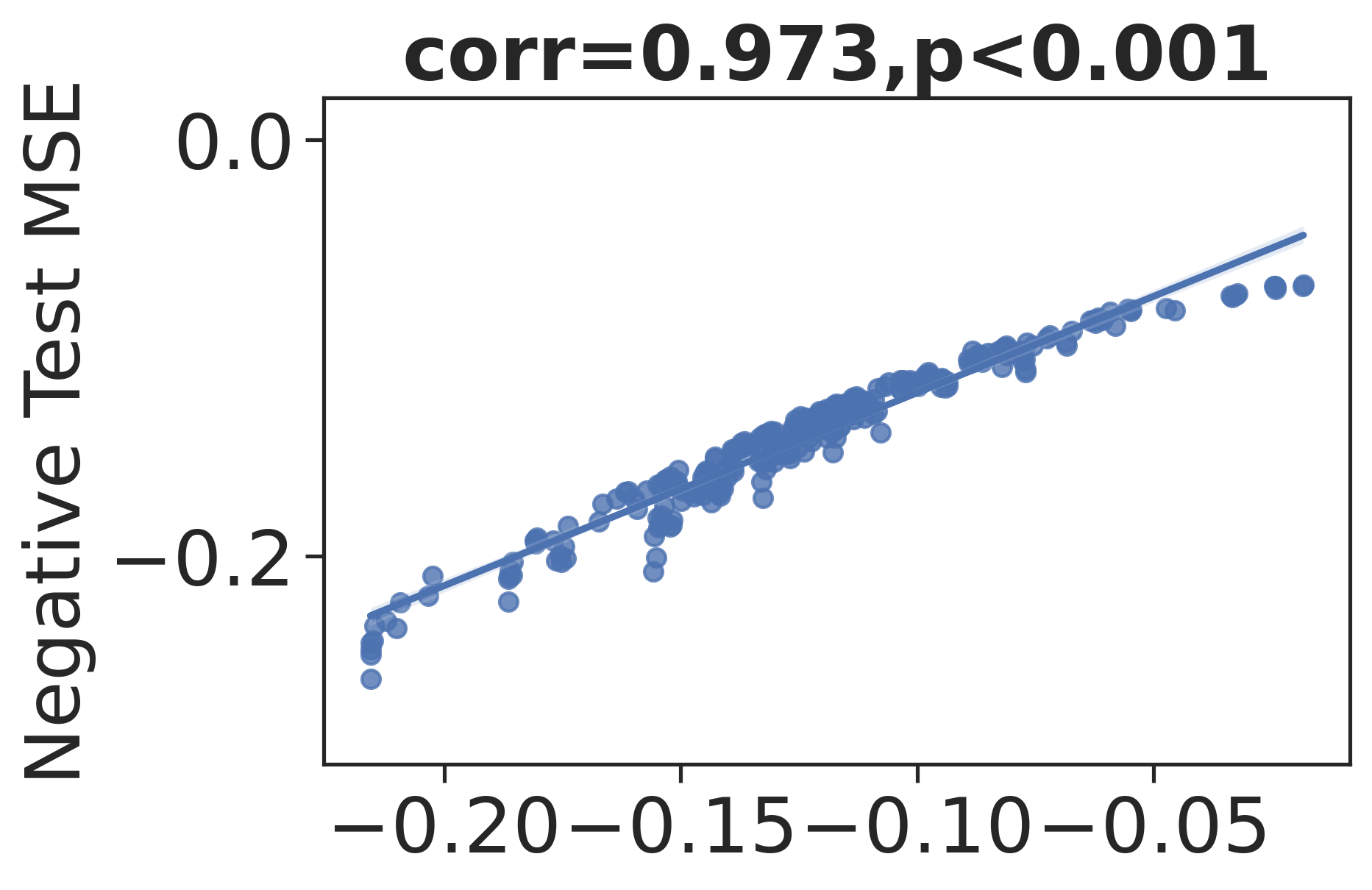}
    \caption{LabMSE0}
    \end{subfigure}
    {\hskip 4pt}
    \begin{subfigure}[b]{0.23\textwidth}
    \includegraphics[width=\textwidth]{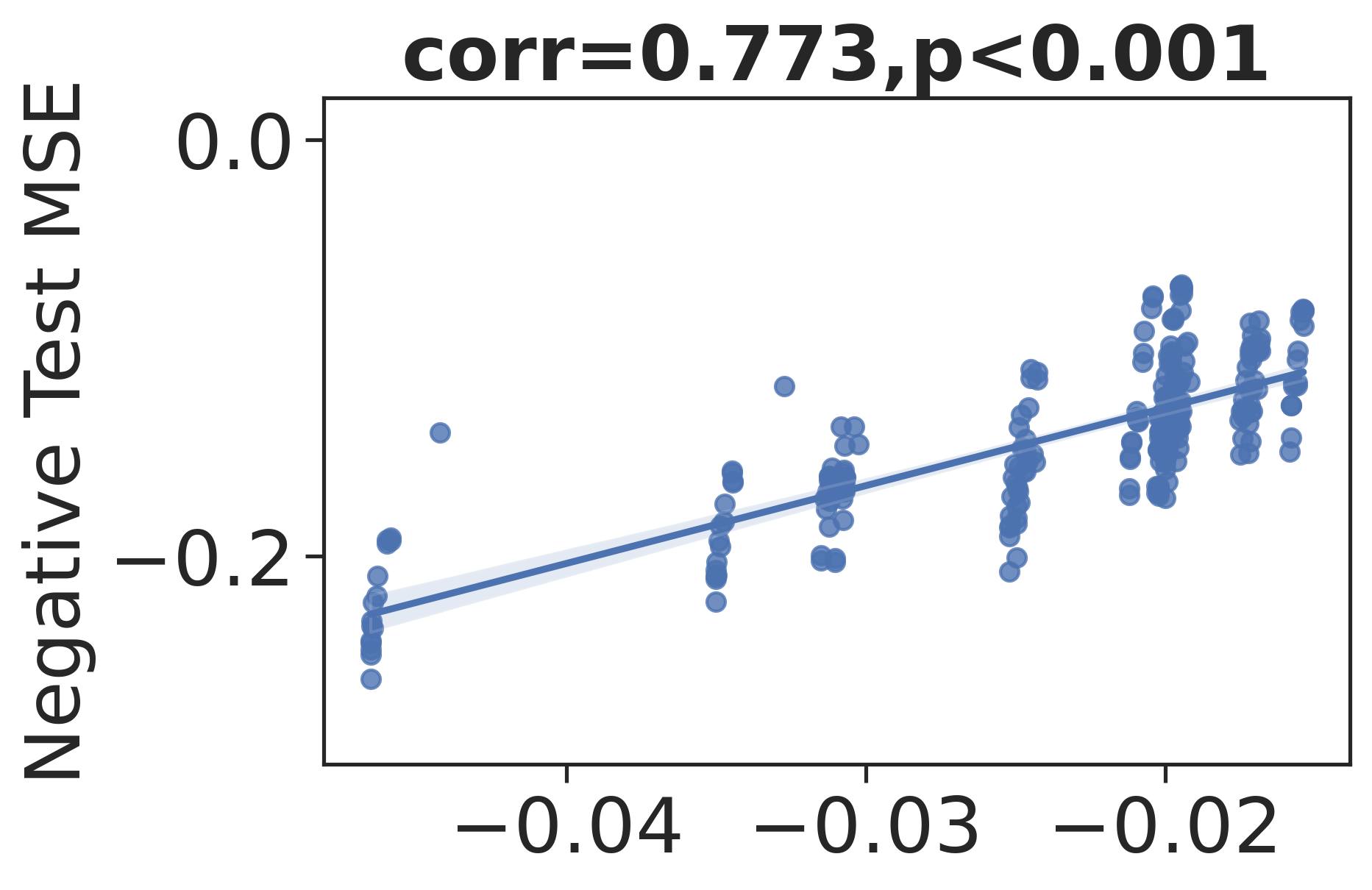}
    \caption{LabMSE1}
    \end{subfigure}
    {\vskip -0.2cm}
    \caption{\textbf{Correlation coefficients and $p$-values between transferability estimators and negative test MSEs} when transferring with head re-training between any two different keypoints (with shared inputs) on OpenMonkey.}
    \label{fig:shared_input_openmonkey_head_rt}
    
    {\vskip 0.4cm}
    
    \begin{subfigure}[b]{0.23\textwidth}
    \includegraphics[width=\textwidth]{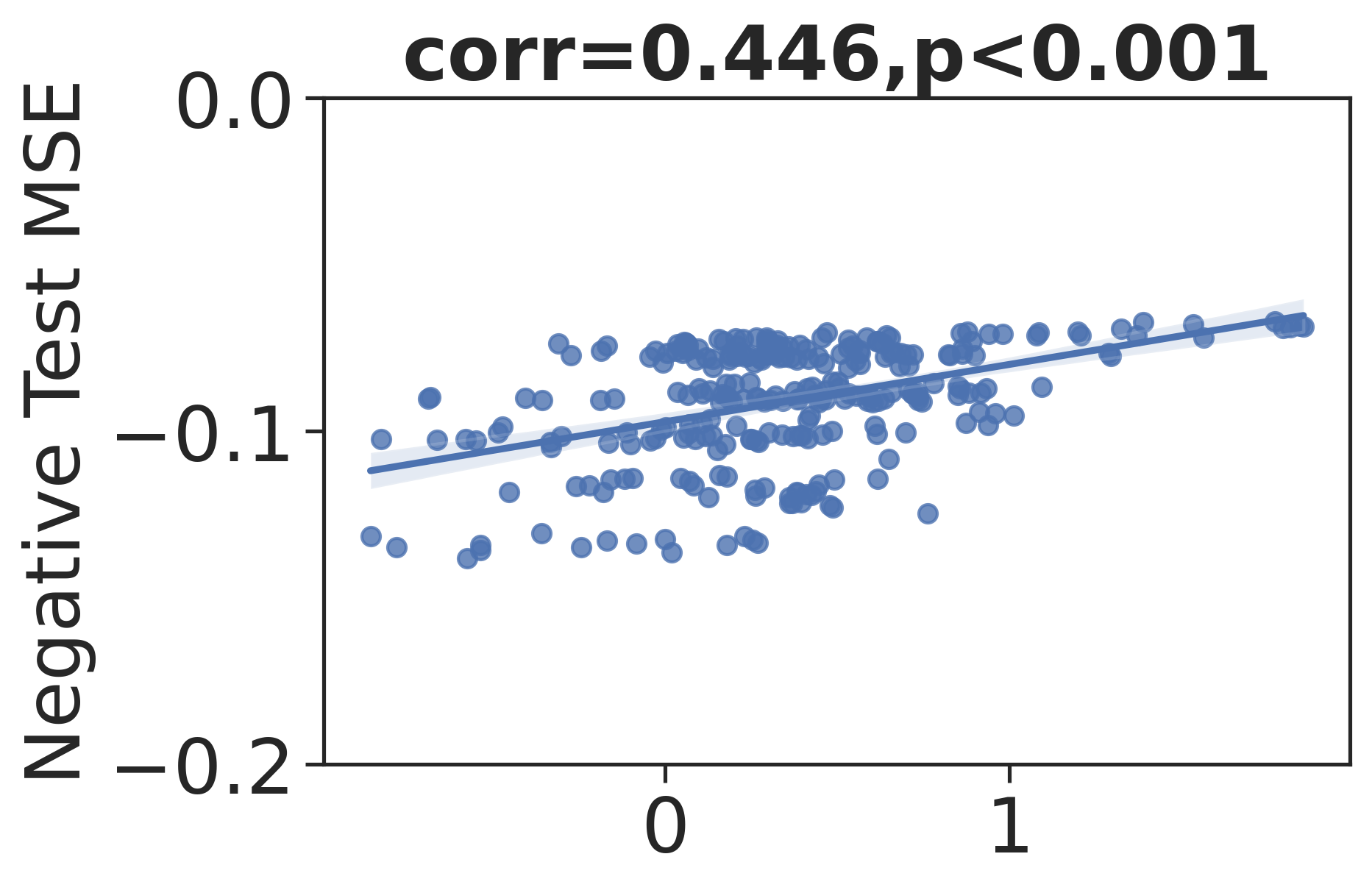}
    \caption{LogME}
    \end{subfigure}
    {\hskip 4pt}
    \begin{subfigure}[b]{0.23\textwidth}
    \includegraphics[width=\textwidth]{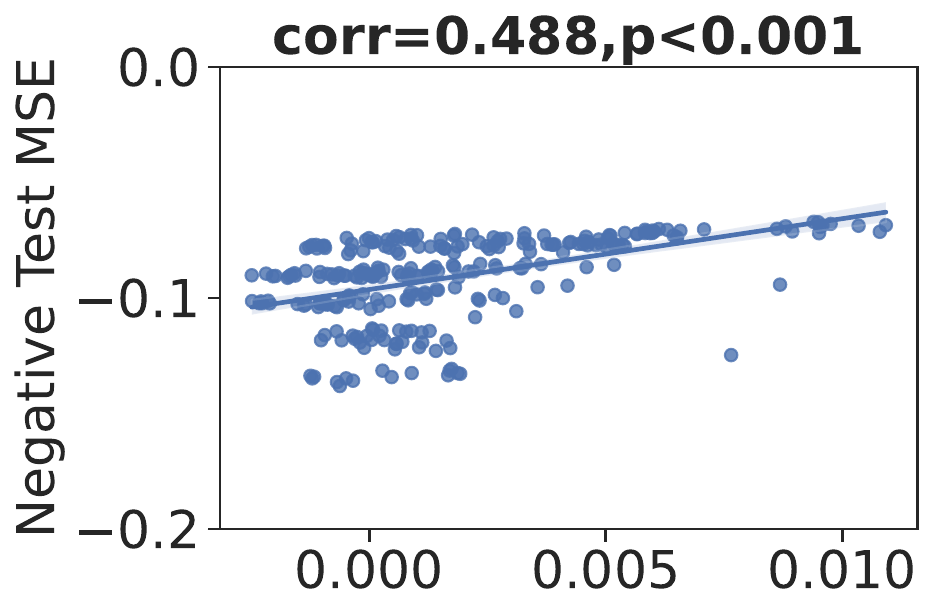}
    \caption{TransRate}
    \end{subfigure}
    {\hskip 4pt}
    \begin{subfigure}[b]{0.23\textwidth}
    \includegraphics[width=\textwidth]{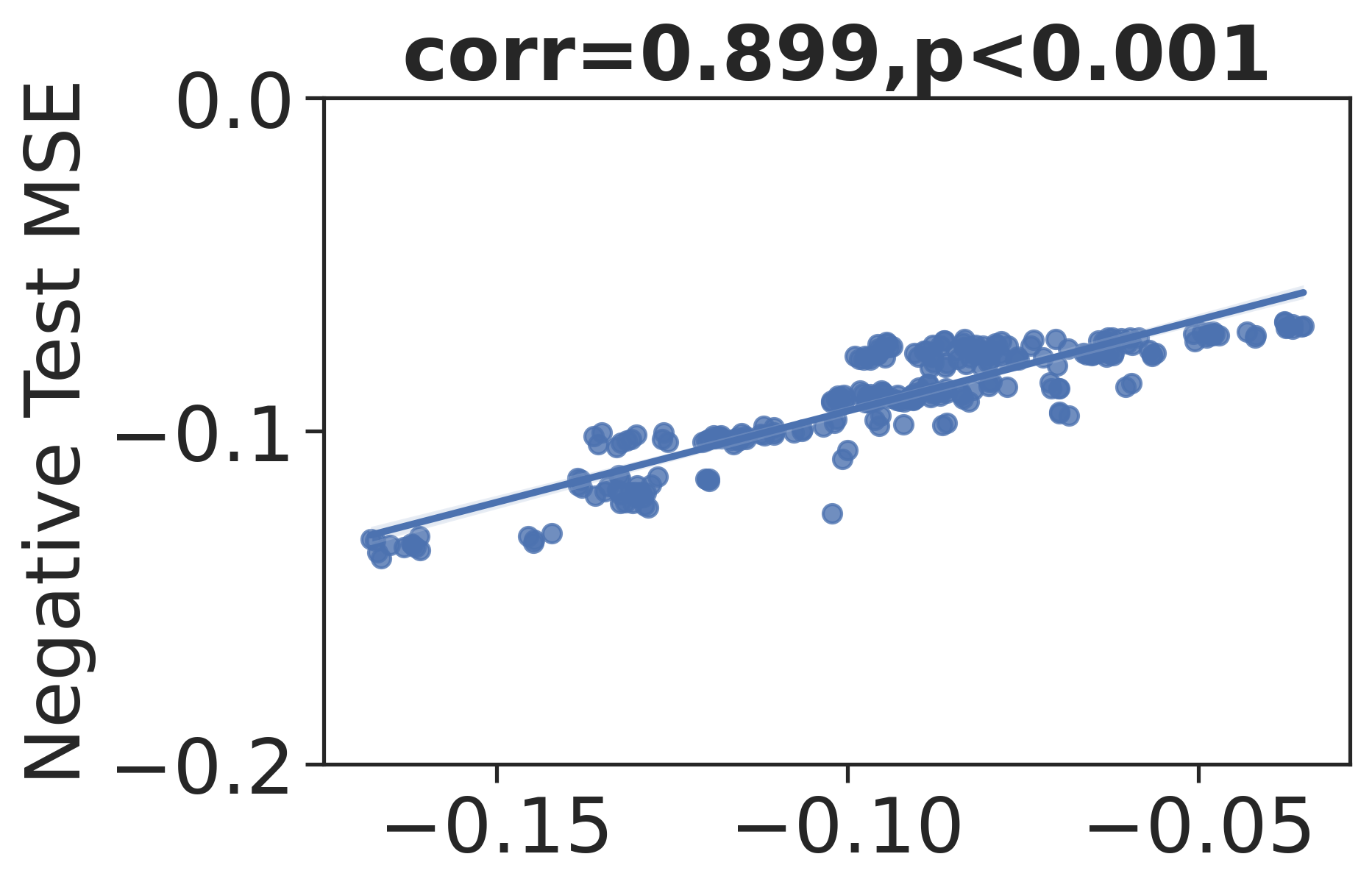}    
    \caption{LinMSE0}
    \end{subfigure}
    {\hskip 4pt}
    \begin{subfigure}[b]{0.23\textwidth}
    \includegraphics[width=\textwidth]{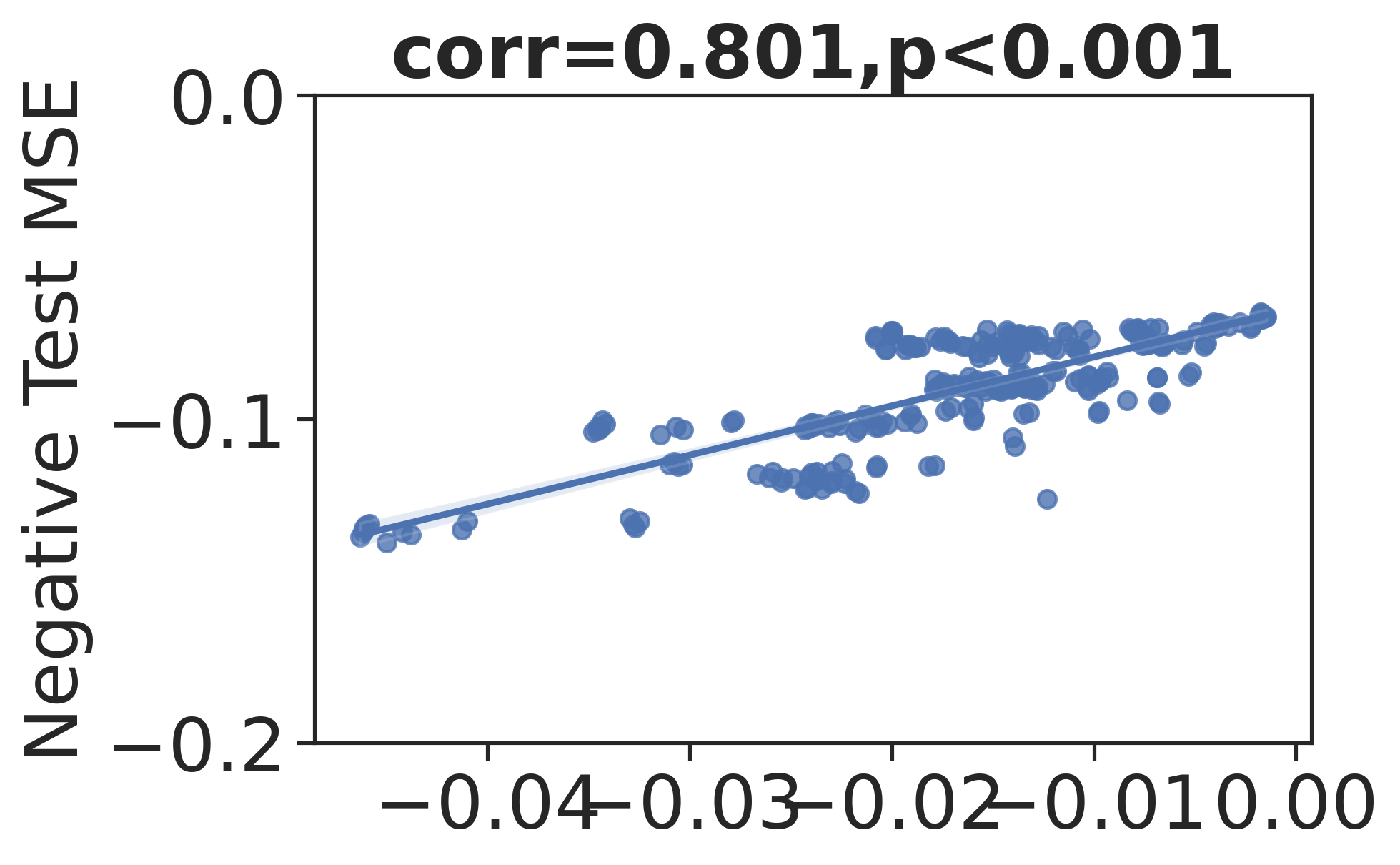}
    \caption{LinMSE1}
    \end{subfigure}
    
    {\vskip 0.2cm}
    
    \begin{subfigure}[b]{0.23\textwidth}
    \includegraphics[width=\textwidth]{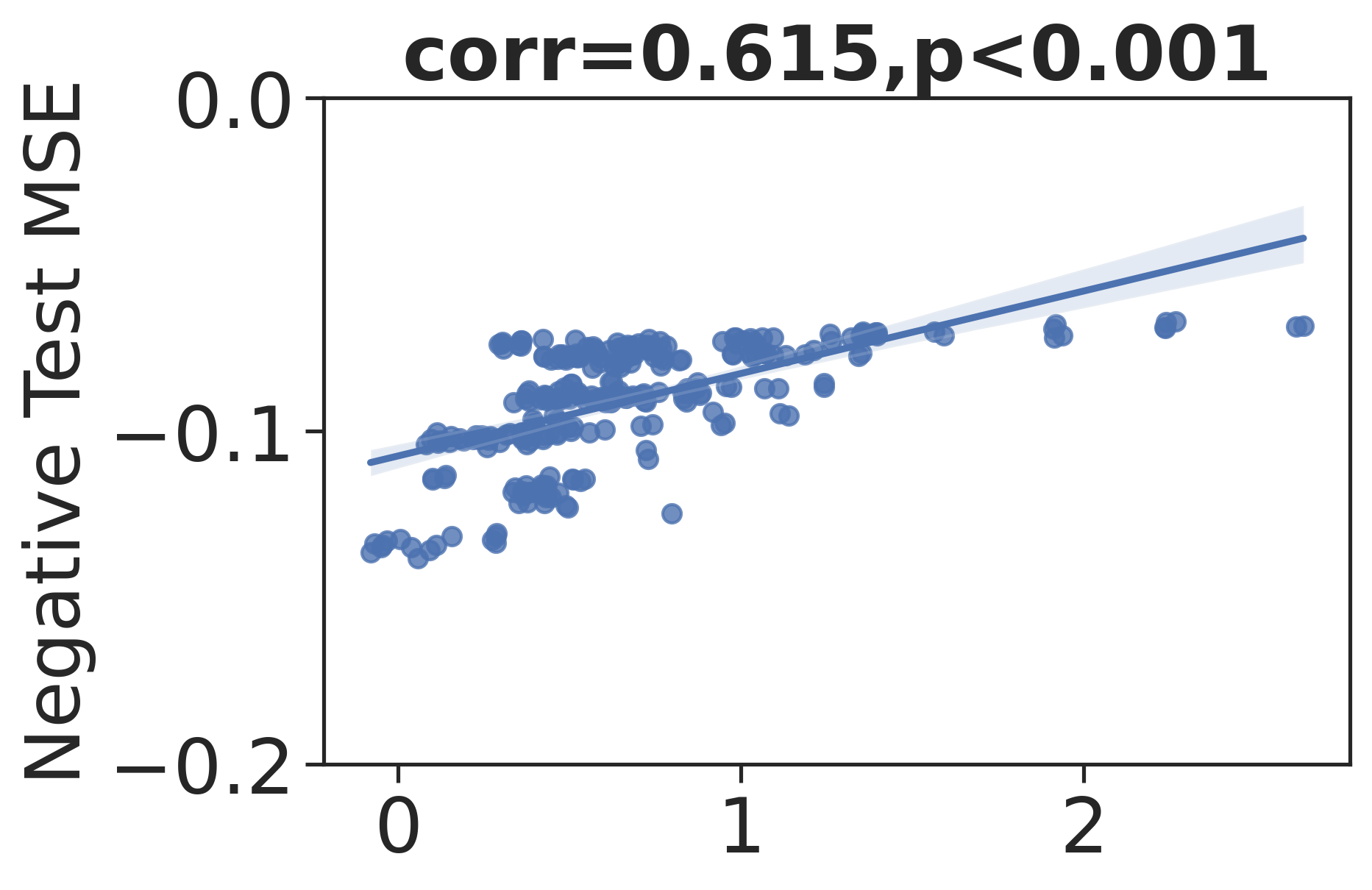}
    \caption{LabLogME}
    \end{subfigure}
    {\hskip 4pt}
    \begin{subfigure}[b]{0.23\textwidth}
    \includegraphics[width=\textwidth]{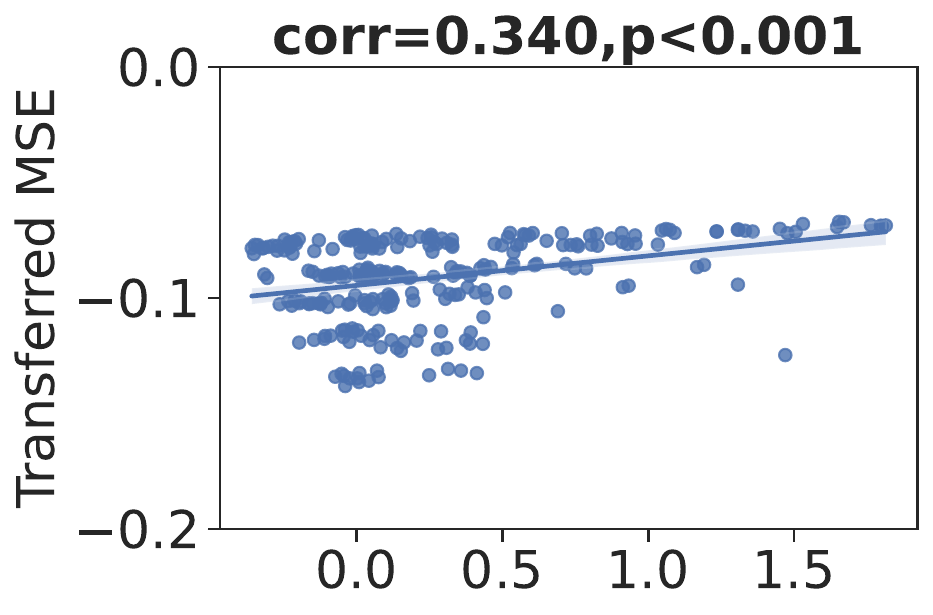}
    \caption{LabTransRate}
    \end{subfigure}
    {\hskip 4pt}
    \begin{subfigure}[b]{0.23\textwidth}
    \includegraphics[width=\textwidth]{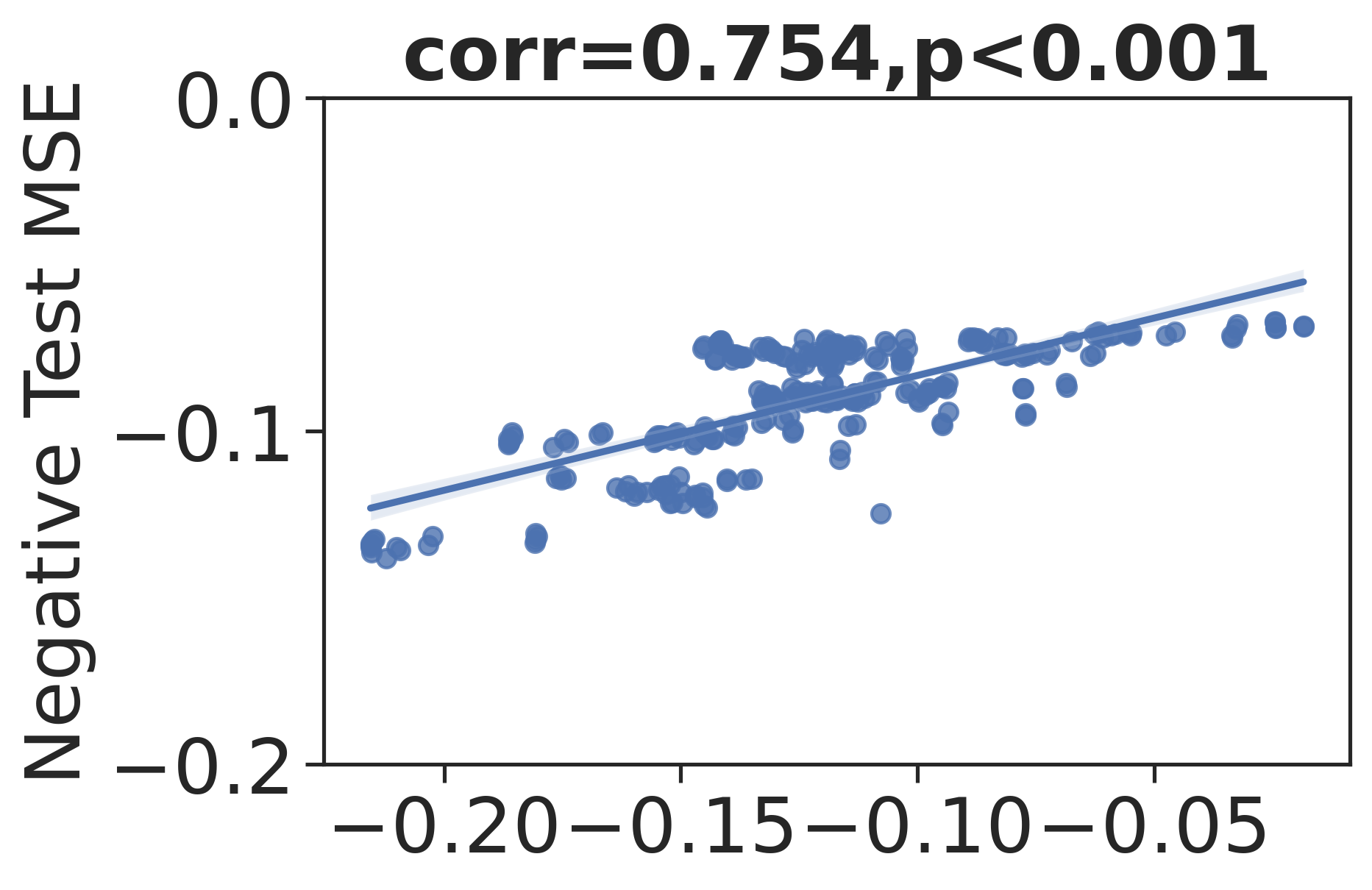}
    \caption{LabMSE0}
    \end{subfigure}
    {\hskip 4pt}
    \begin{subfigure}[b]{0.23\textwidth}
    \includegraphics[width=\textwidth]{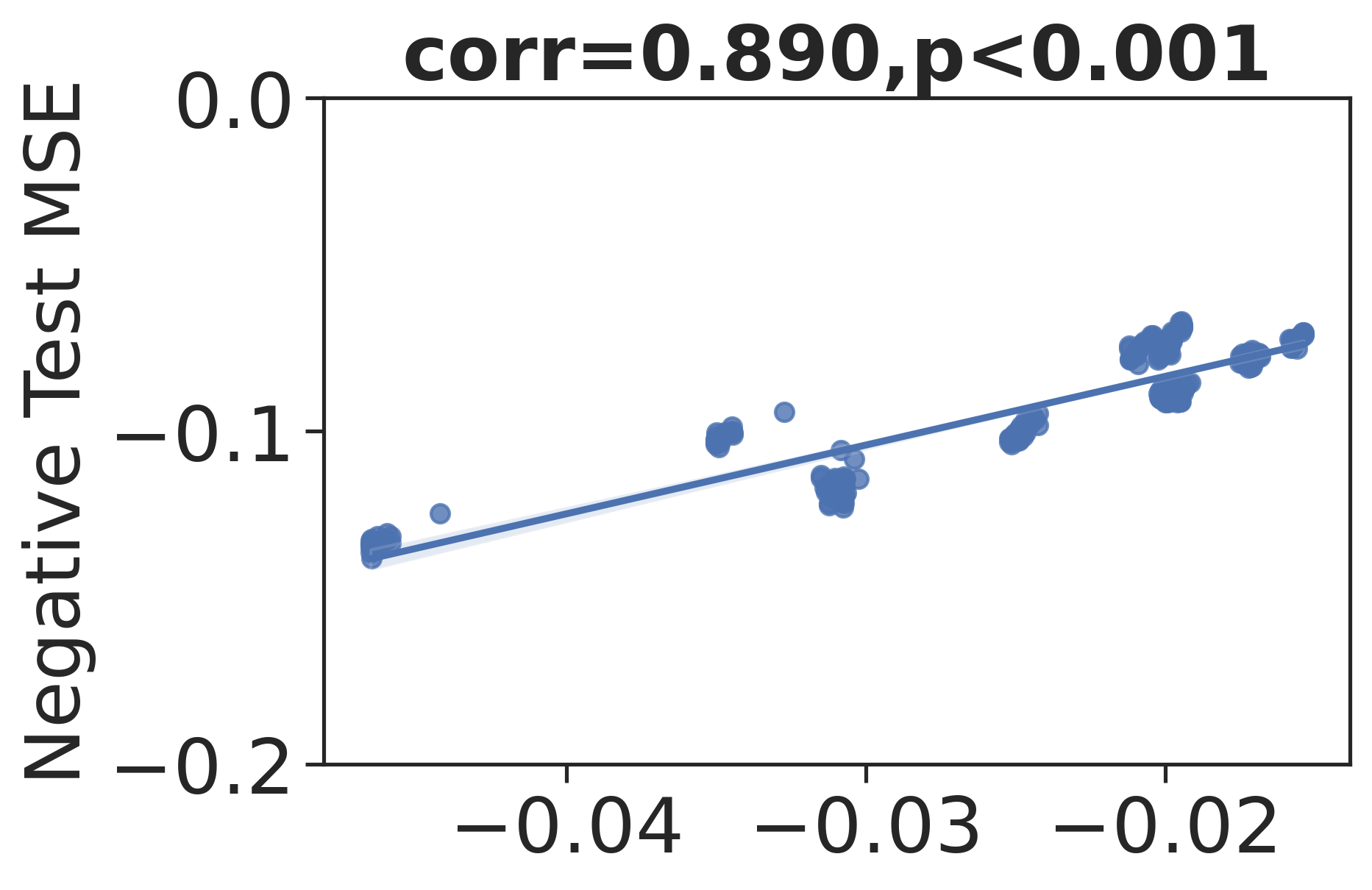}
    \caption{LabMSE1}
    \end{subfigure}
    {\vskip -0.2cm}
    \caption{\textbf{Correlation coefficients and $p$-values between transferability estimators and negative test MSEs} when transferring with half fine-tuning between any two different keypoints (with shared inputs) on OpenMonkey.}
    \label{fig:shared_input_openmonkey_half_ft}
    
    {\vskip 0.4cm}
    
    \begin{subfigure}[b]{0.23\textwidth}
    \includegraphics[width=\textwidth]{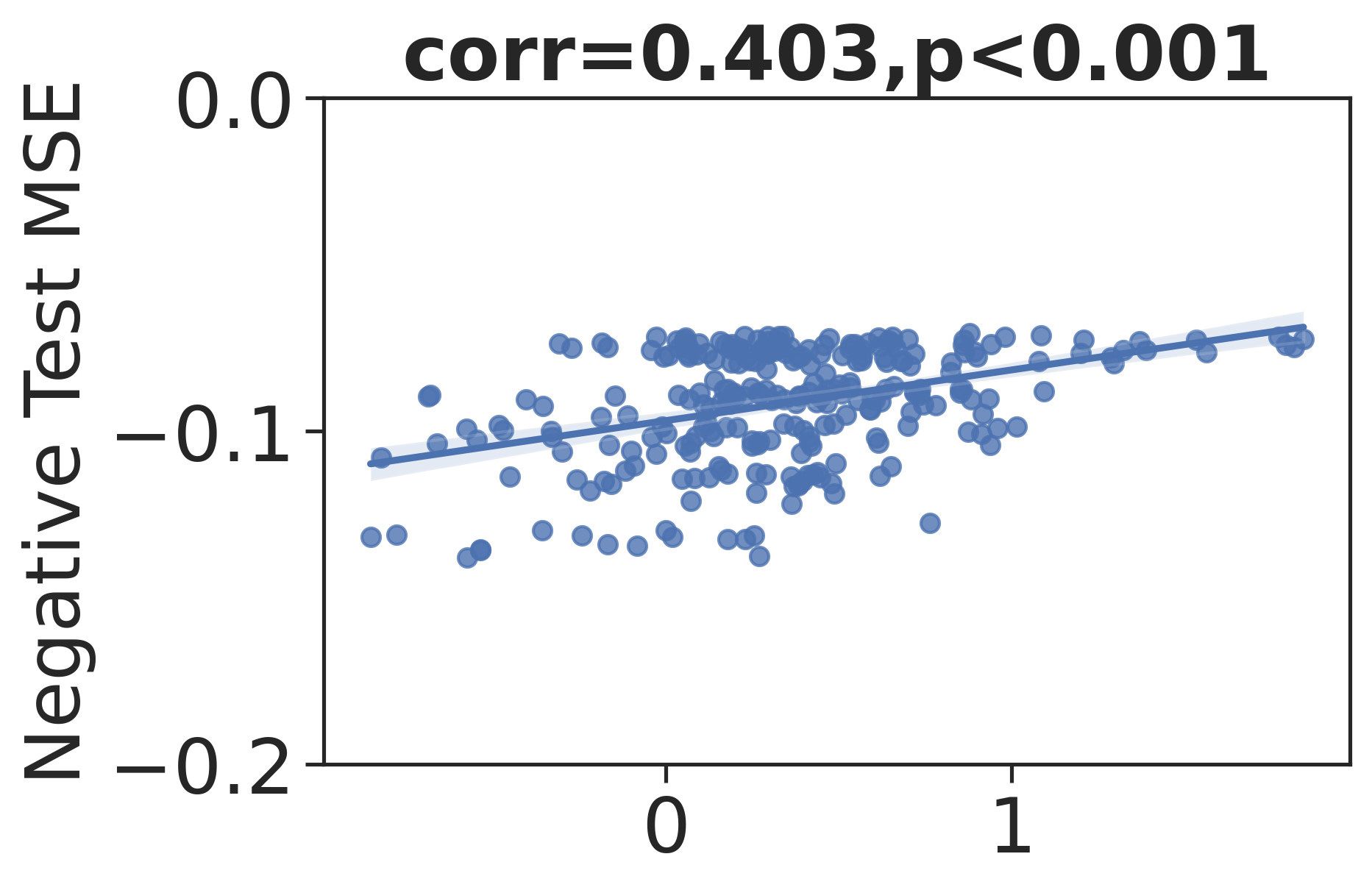}
    \caption{LogME}
    \end{subfigure}
    {\hskip 4pt}
    \begin{subfigure}[b]{0.23\textwidth}
    \includegraphics[width=\textwidth]{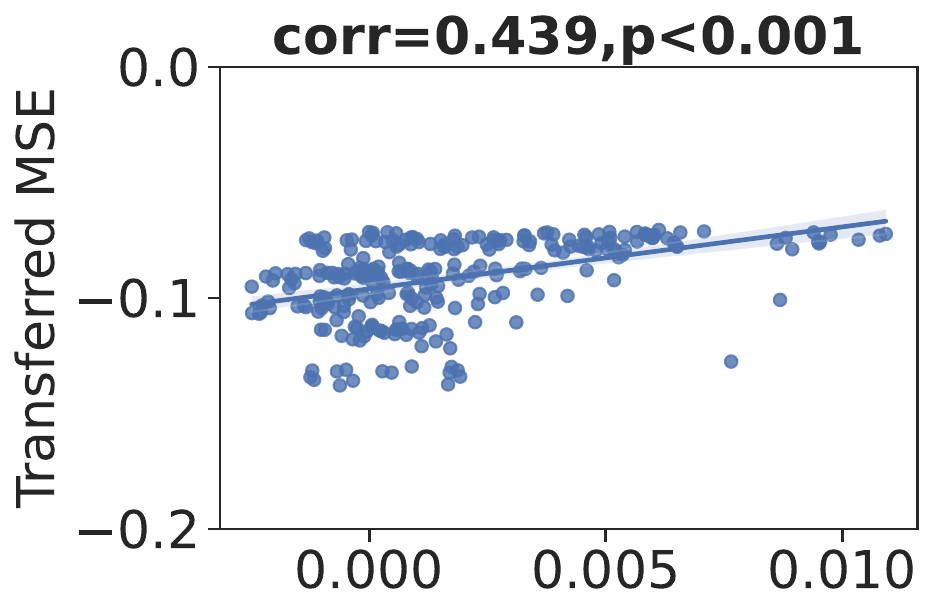}
    \caption{TransRate}
    \end{subfigure}
    {\hskip 4pt}
    \begin{subfigure}[b]{0.23\textwidth}
    \includegraphics[width=\textwidth]{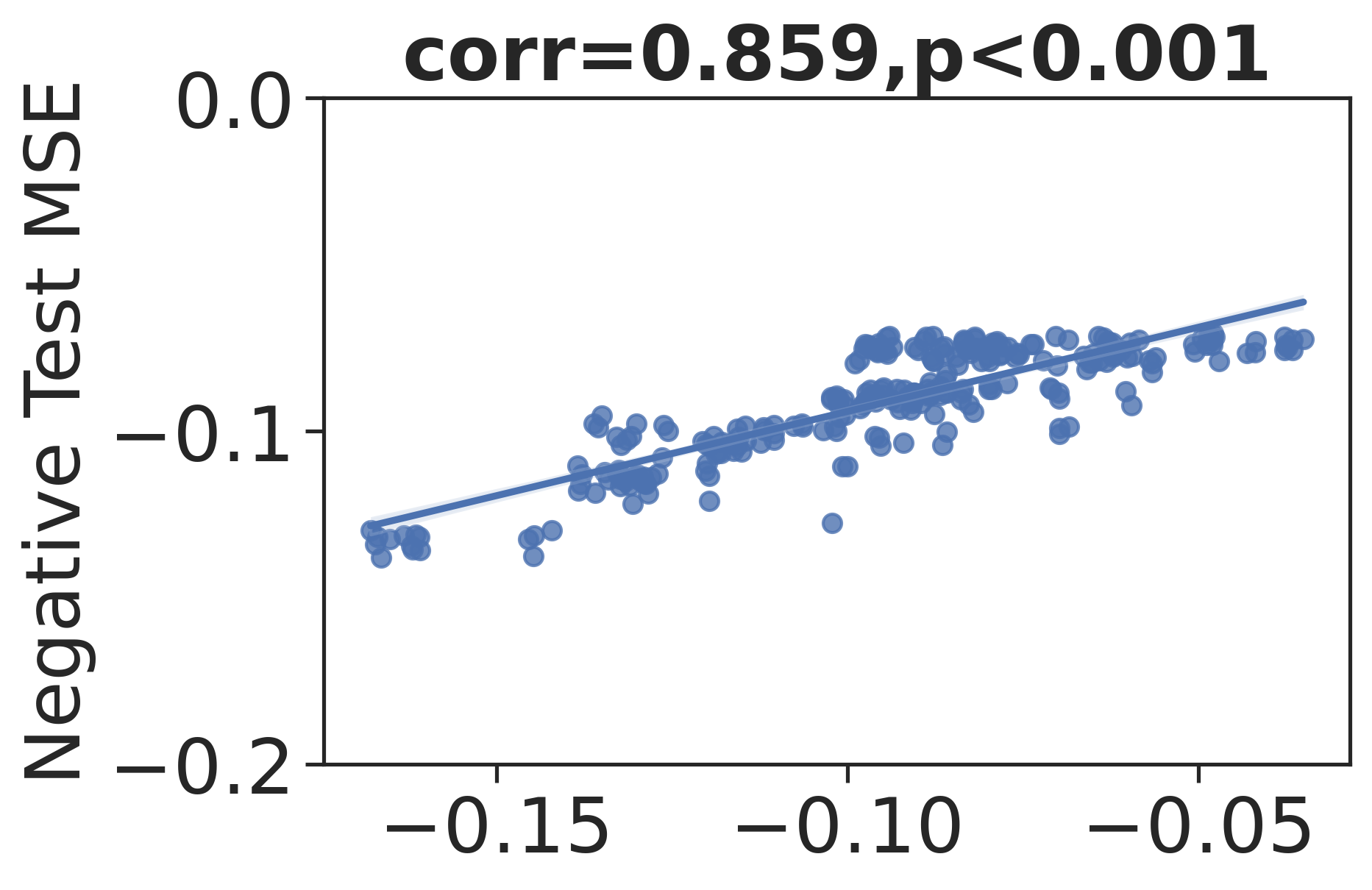}    
    \caption{LinMSE0}
    \end{subfigure}
    {\hskip 4pt}
    \begin{subfigure}[b]{0.23\textwidth}
    \includegraphics[width=\textwidth]{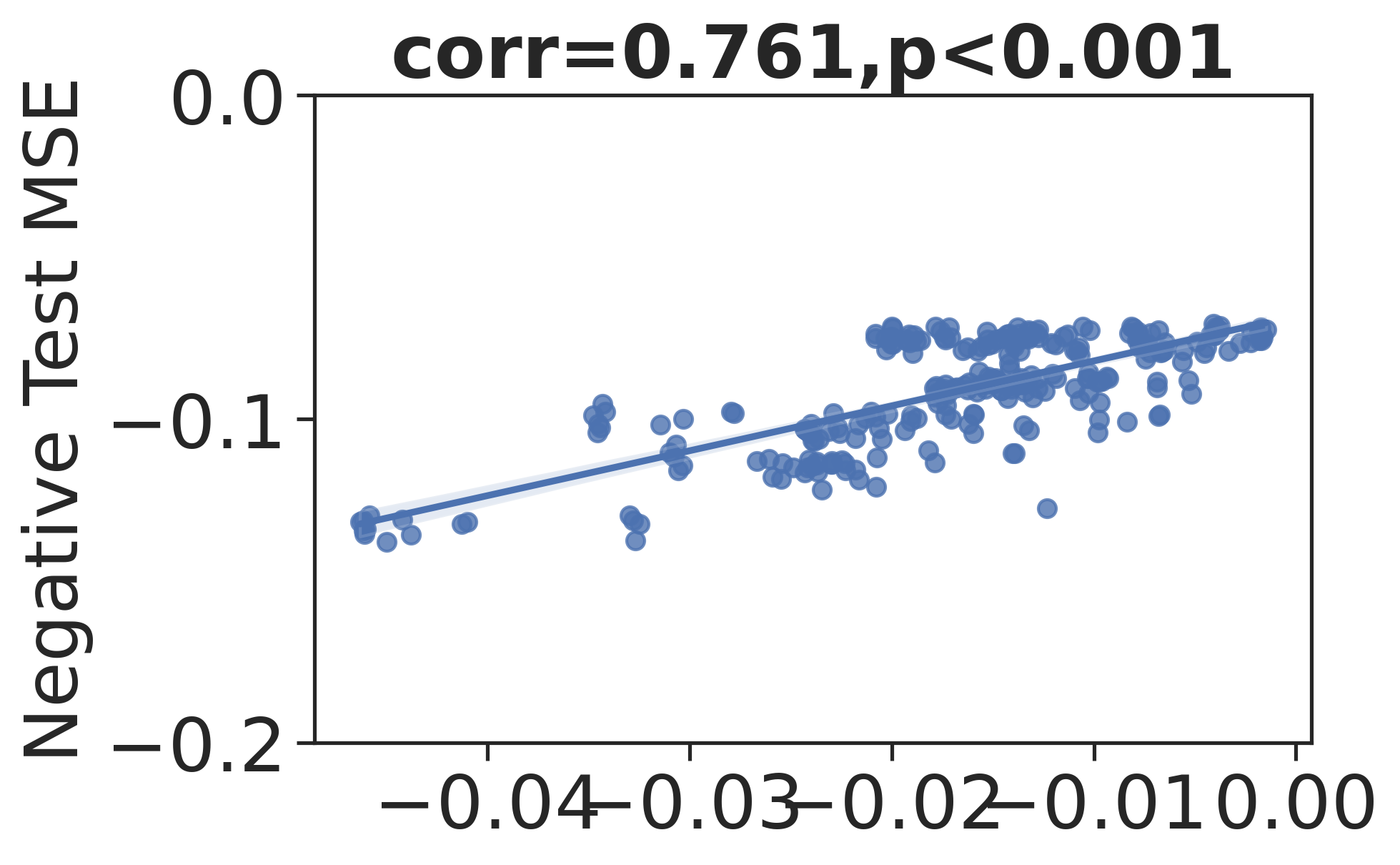}
    \caption{LinMSE1}
    \end{subfigure}
    
    {\vskip 4pt}
    
    \begin{subfigure}[b]{0.23\textwidth}
    \includegraphics[width=\textwidth]{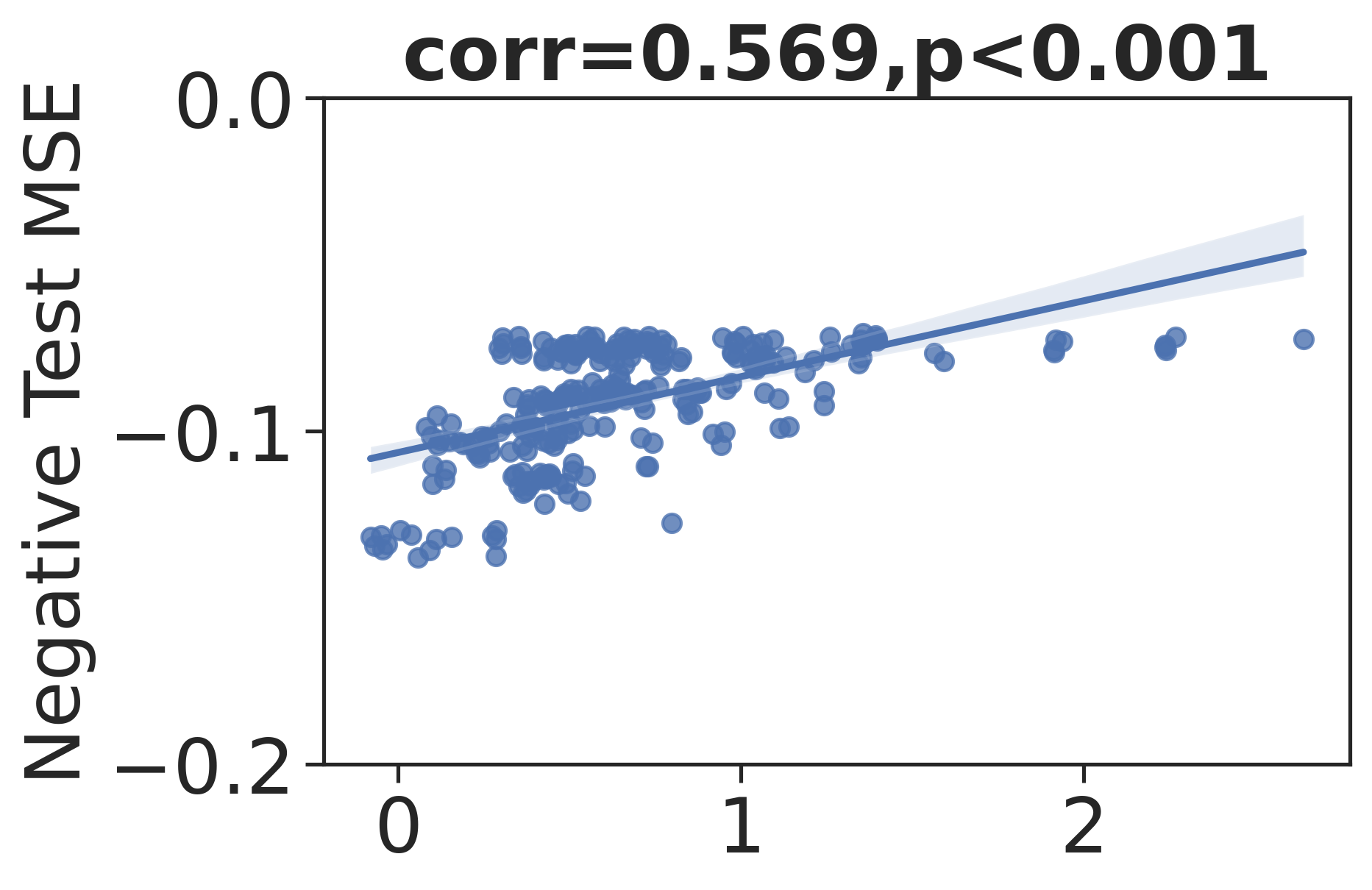}
    \caption{LabLogME}
    \end{subfigure}
    {\hskip 4pt}
    \begin{subfigure}[b]{0.23\textwidth}
    \includegraphics[width=\textwidth]{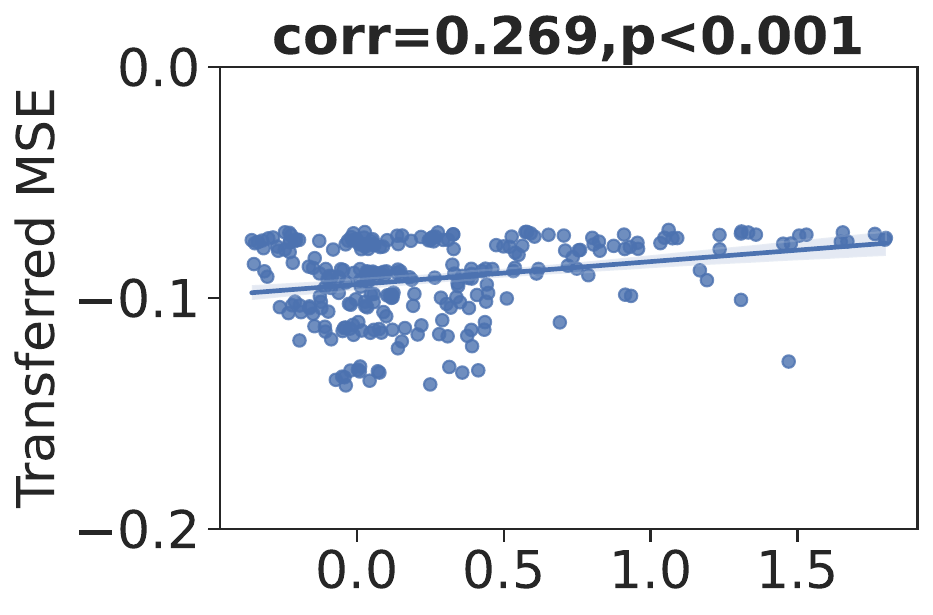}
    \caption{LabTransRate}
    \end{subfigure}
    {\hskip 4pt}
    \begin{subfigure}[b]{0.23\textwidth}
\includegraphics[width=\textwidth]{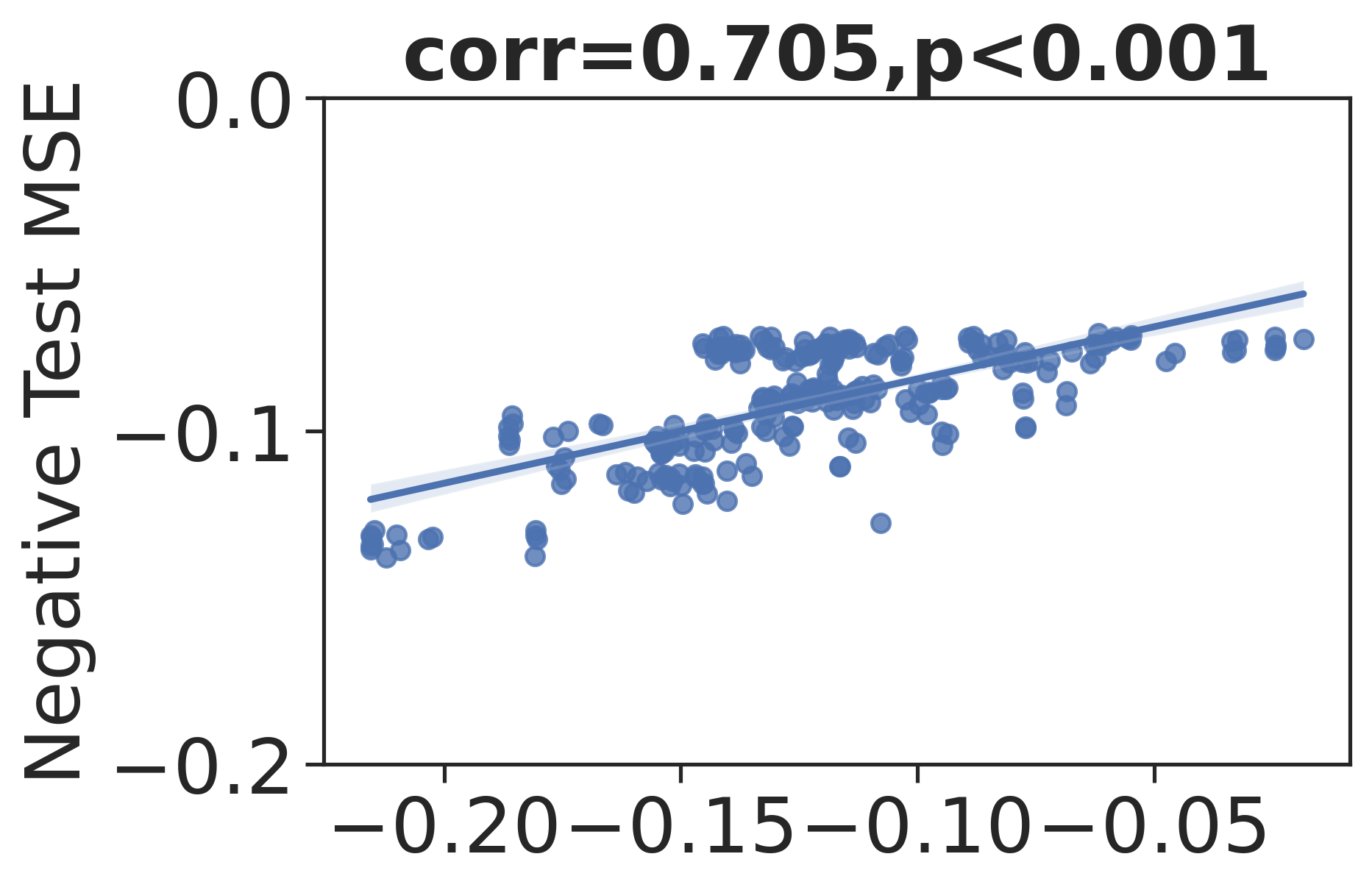}
    \caption{LabMSE0}
    \end{subfigure}
    {\hskip 4pt}
    \begin{subfigure}[b]{0.23\textwidth}
    \includegraphics[width=\textwidth]{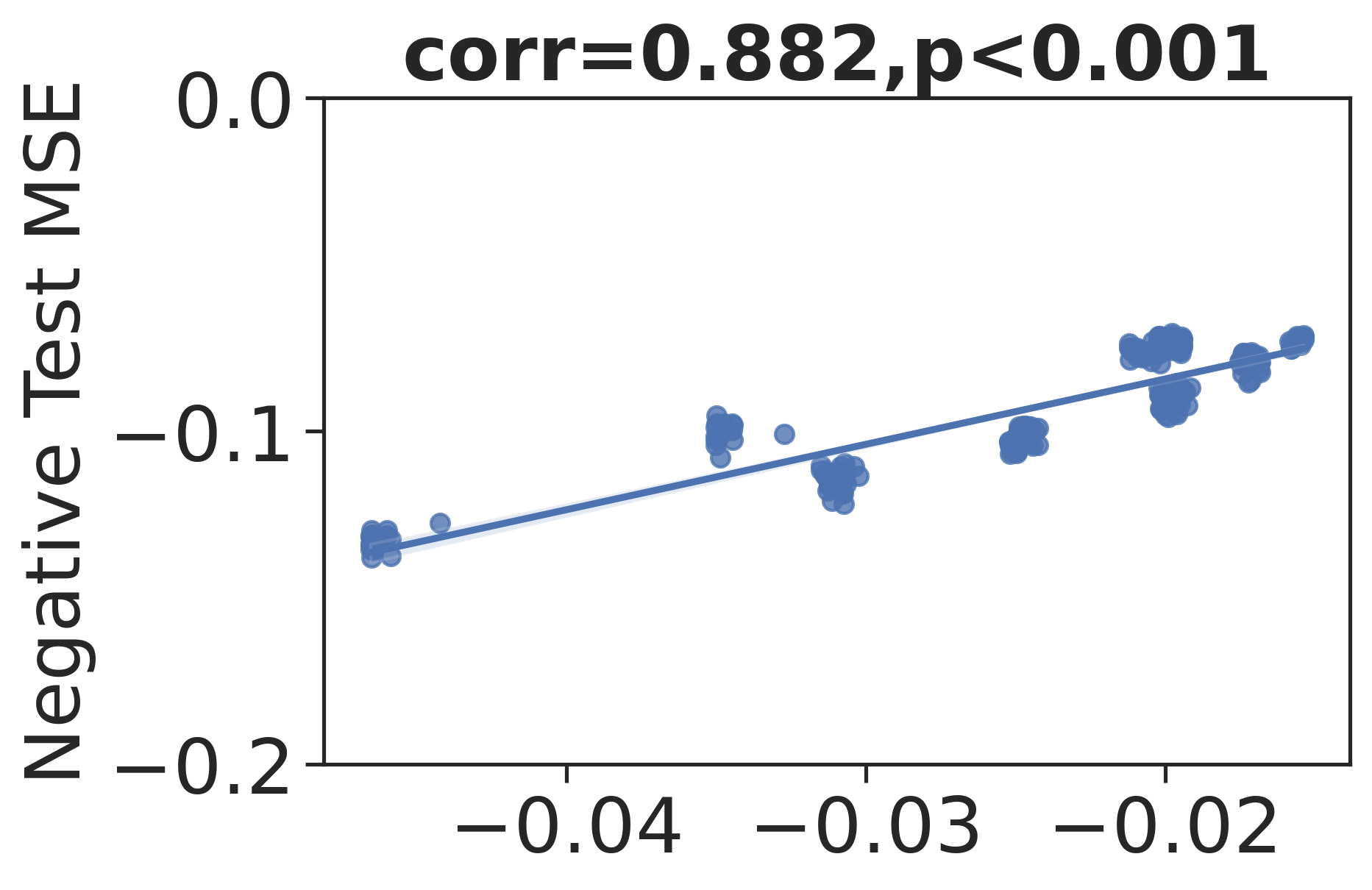}
    \caption{LabMSE1}
    \end{subfigure}
    {\vskip -0.2cm}
    \caption{\textbf{Correlation coefficients and $p$-values between transferability estimators and negative test MSEs} when transferring with full fine-tuning between any two different keypoints (with shared inputs) on OpenMonkey.}
    \label{fig:shared_input_openmonkey_full_ft}
\end{figure*}

\begin{table*}[ht]
\caption{{\bf Correlation coefficients for all source tasks} on CUB-200-2011. Bold numbers indicate best results in each row. Asterisks (*) indicate best results among the corresponding label-based or feature-based methods.}
\resizebox{\textwidth}{!}{%
\centering
\begin{tabular}{cccccccccc}
\toprule
\multirow{3}{*}{\parbox{1.4cm}{Transfer setting}} & \multirow{3}{*}{Source task} & \multicolumn{4}{c}{Label-based method} & \multicolumn{4}{c}{Feature-based method} \\
\cmidrule(lr){3-6} \cmidrule(lr){7-10}
& & LabLogME & LabTransRate & LabMSE0 & LabMSE1 & LogME & TransRate & LinMSE0  & LinMSE1 \\
\midrule
\multirow{15}{*}{\parbox{1.4cm}{Head re-training}} & Back & 0.743 & 0.116 & 0.956 & \textbf{0.966*} & 0.920 & 0.273 & 0.931 & 0.964*\\ 
& Beak & 0.863 & 0.229 & 0.922* & 0.915 & 0.878 & 0.158 & 0.906 & \textbf{0.945*}\\ 
& Belly & 0.892 & 0.097 & 0.970 & \textbf{0.982*} & 0.933 & 0.188 & 0.932 & \textbf{0.982*}\\ 
& Breast & 0.915 & 0.120 & 0.935 & 0.945* & 0.903 & 0.279 & 0.922 & \textbf{0.961*}\\ 
& Crown & 0.917 & 0.041 & 0.962 & 0.966* & 0.913 & 0.251 &0.945 & \textbf{0.979*}\\ 
& Forehead & 0.888 & 0.076 & 0.941* & 0.939 & 0.885 & 0.221 & 0.924 & \textbf{0.966*}\\ 
& Left eye & 0.035 & 0.076 & 0.913 & 0.964* & 0.924 & 0.289 & 0.945 & \textbf{0.969*}\\ 
& Left leg & 0.261 & 0.221 & 0.935 & \textbf{0.975*} & 0.935 & 0.223 & 0.953 & \textbf{0.975*}\\ 
& Left wing & 0.260 & 0.170 & 0.964 & \textbf{0.994*} & 0.980 & 0.173 & \textbf{0.994*} & \textbf{0.994*}\\ 
& Nape& 0.889 & 0.085 & 0.922 & 0.942* & 0.900 & 0.300 & 0.929 & \textbf{0.953}*\\ 
& Right eye & 0.625 & 0.242 & 0.904 & 0.974* & 0.921 & 0.244 & 0.948 & \textbf{0.975*}\\ 
& Right leg & 0.508 & 0.047 & 0.958 & 0.989* & 0.942 & 0.217 & 0.954 & \textbf{0.990*}\\ 
& Right wing & 0.521 & 0.167 & 0.907 & 0.979* & 0.935 & 0.270 & 0.946 & \textbf{0.980*}\\ 
& Tail & 0.591 & 0.392 & 0.900 & \textbf{0.927*} & 0.872 & 0.544 & 0.880 & 0.890*\\ 
& Throat & 0.896 & 0.124 & 0.938 & 0.941* & 0.890 & 0.291 & 0.924 & \textbf{0.956*}\\ 
\midrule
\multirow{15}{*}{\parbox{1.4cm}{Half fine-tuning}} 
& Back & 0.714 & 0.076 & 0.791 & 0.814* & 0.835 & 0.168 & \textbf{0.911*} & 0.873\\ 
& Beak & 0.663 & 0.160 & 0.831* & 0.772 & 0.765 & 0.076 & 0.883 & \textbf{0.899*}\\ 
& Belly & 0.528 & 0.233 & 0.655 & 0.752* &0.758 & 0.309 & \textbf{0.849*} &0.764\\ 
& Breast & 0.730 & 0.100 & 0.802* &0.779 & 0.762 & 0.152 & \textbf{0.867}* & 0.850\\ 
& Crown & 0.644 & 0.068 & 0.752 & 0.776* & 0.714 & 0.165 & \textbf{0.832}* & 0.816\\
& Forehead & 0.654 & 0.032 & 0.804* & 0.786 & 0.727 & 0.120 & 0.859 & \textbf{0.873*} \\ 
& Left eye & 0.420 & 0.046 & \textbf{0.913*} & 0.853 & 0.812 & 0.227 & 0.892* &0.865\\ 
& Left leg & 0.121 & 0.095 & 0.721 & 0.819* & 0.845 & 0.150 & \textbf{0.893}* & 0.832\\ 
& Left wing & 0.352 & 0.150 & \textbf{0.949}* & 0.918 & 0.859 & 0.189 & 0.919* & 0.918 \\ 
& Nape& 0.660 & 0.055 & 0.705 & 0.770* & 0.751& 0.181 & \textbf{0.863*} & 0.802\\ 
& Right eye & 0.561 & 0.221 & \textbf{0.911*} & 0.873 & 0.786 & 0.180 & 0.871 & 0.890*\\ 
& Right leg & 0.268 & 0.125 & 0.690 & 0.804* & 0.810 & 0.069 & \textbf{0.861*} & 0.820\\ 
& Right wing & 0.407 & 0.133 & 0.495 & 0.613* & 0.516 & 0.338 & 0.521 & \textbf{0.617*}\\ 
& Tail & 0.801 & 0.117 & 0.930* & 0.812 & 0.848 & 0.285 & 0.924 & \textbf{0.968*}\\ 
& Throat & 0.767 & 0.013 & 0.870* & 0.810 & 0.811 & 0.253 & \textbf{0.900*} & 0.873\\ 
\midrule
\multirow{15}{*}{\parbox{1.4cm}{Full fine-tuning}} & Back & 0.710 & 0.085 & 0.785 & 0.808* & 0.829 & 0.178 & \textbf{0.906*} & 0.868\\ 
& Beak & 0.659 & 0.161 & 0.826* & 0.780 & 0.758 & 0.073 & 0.877 & \textbf{0.899*}\\ 
& Belly & 0.645 & 0.273 & 0.782 & 0.847* & 0.862 & 0.365 & \textbf{0.926*} & 0.856 \\ 
& Breast & 0.740 & 0.104 & 0.811* & 0.791 & 0.768 & 0.152 & \textbf{0.871*} & 0.859\\ 
& Crown & 0.647 & 0.073 & 0.756 & 0.784* & 0.717 & 0.157 & \textbf{0.834*} & 0.821 \\ 
& Forehead & 0.648 & 0.037 & 0.799* & 0.783 & 0.723 & 0.111 & 0.855 & \textbf{0.869*} \\ 
& Left eye & 0.224 & \textbf{0.456}* & 0.297 & 0.347 & 0.333* &0.246 & 0.282 & 0.326 \\ 
& Left leg & 0.057 & 0.067 & 0.659 & 0.769* & 0.796 & 0.146 & \textbf{0.850}* & 0.783 \\ 
& Left wing & 0.342 & 0.159 & \textbf{0.954}* & 0.915 & 0.860 & 0.195 & 0.920* & 0.914\\ 
& Nape & 0.667 & 0.041 & 0.713 & 0.779* & 0.752 & 0.177 & \textbf{0.864*} & 0.810 \\ 
& Right eye & 0.549 & 0.213 & \textbf{0.915*} & 0.876 & 0.794 & 0.199 & 0.877 & 0.893* \\ 
& Right leg & 0.237 & 0.377 & 0.673 & 0.692* & 0.755 & 0.431 & \textbf{0.766}* &0.693 \\ 
& Right wing& \textbf{0.254}* & 0.046 & 0.237 & 0.223 & 0.225 & 0.093 & 0.227* & 0.220 \\ 
& Tail & 0.803 & 0.122 & 0.930* & 0.818 & 0.846 & 0.288 & 0.923 & \textbf{0.969}*\\ 
& Throat & 0.665 & 0.027 & 0.801* & 0.779 & 0.744 & 0.256 & \textbf{0.850*} & 0.834\\ 
\bottomrule
\end{tabular}
}
\label{tab:full_all_sources_cub}
\end{table*}

\begin{table*}[ht]
\caption{{\bf Correlation coefficients for all source tasks} on OpenMonkey. Bold numbers indicate best results in each row. Asterisks (*) indicate best results among the corresponding label-based or feature-based methods.}
\resizebox{\textwidth}{!}{%
\centering
\begin{tabular}{cccccccccc}
\toprule
\multirow{3}{*}{\parbox{1.4cm}{Transfer setting}} & \multirow{3}{*}{Source task} & \multicolumn{4}{c}{Label-based method} & \multicolumn{4}{c}{Feature-based method} \\
\cmidrule(lr){3-6} \cmidrule(lr){7-10}
& & LabLogME & LabTransRate & LabMSE0 & LabMSE1 & LogME & TransRate & LinMSE0 & LinMSE1 \\
\midrule
\multirow{17}{*}{\parbox{1.4cm}{Head re-training}} 
& Right eye & 0.894 & 0.859 & \textbf{0.986*} & 0.835 & 0.918 & 0.846 & 0.978 & \textbf{0.986*}\\ 
&Left eye& 0.895 & 0.854 &\textbf{0.987}* & 0.838 & 0.868 & 0.858 & 0.981 & \textbf{0.987*}\\ 
& Nose & 0.908 & 0.849 & 0.988* & 0.849 & 0.818 & 0.837 & 0.978 & \textbf{0.989*}\\ 
& Head & 0.941 & 0.881 & \textbf{0.992*} & 0.821 & 0.897 & 0.884 & 0.983* & 0.978\\ 
& Neck & 0.972 & 0.862 & \textbf{0.998}* & 0.887 & 0.932 & 0.839 & 0.982 & 0.987*\\ 
& Right shoulder & 0.977 & 0.837 & \textbf{0.994*} & 0.891 & 0.842 & 0.811 & 0.982* & 0.980\\ 
& Right elbow & 0.963 & 0.529 & \textbf{0.994*} & 0.940 & 0.469 & 0.564 & 0.969 & 0.990*\\ 
& Right wrist & 0.970 & 0.753 & \textbf{0.993}* & 0.939 & 0.615 & 0.446& 0.963 & 0.990*\\ 
& Left shoulder & 0.972 & 0.800 & \textbf{0.997}* & 0.915 & 0.823 & 0.808 & 0.988* & 0.988*\\ 
& Left elbow & 0.960 & 0.546 & \textbf{0.994}* & 0.948 & 0.711 & 0.572 & 0.969 & 0.989*\\ 
& Left wrist & 0.975 & 0.597 & \textbf{0.993*} & 0.951 & 0.964 & 0.544 & 0.963 & \textbf{0.993*}\\ 
& Hip & 0.922 & 0.540 & 0.989* & 0.325 & 0.874 & 0.557 & 0.800 & \textbf{0.991*}\\ 
& Right knee & 0.925 & 0.080 & 0.975* & 0.850 & 0.766 & 0.331 & 0.945 & \textbf{0.993*}\\ 
&Right ankle & 0.931 & 0.411 & 0.989* & 0.770 & 0.737 & 0.371 & 0.930 & \textbf{0.997*}\\ 
&Left knee & 0.923 & 0.160 & 0.978* & 0.848 & 0.692 & 0.209 & 0.936 & \textbf{0.994*}\\ 
&Left ankle & 0.916 & 0.416 & 0.986* & 0.775 & 0.852 & 0.329 & 0.925 & \textbf{0.998*}\\ 
&Tail & 0.936 & 0.712 & \textbf{0.993*} & 0.312 & 0.821 & 0.662 & 0.897 & 0.990*\\ 
\midrule
\multirow{17}{*}{\parbox{1.4cm}{Half fine-tuning}} 
& Right eye & 0.795 & 0.734 & 0.906* & 0.883 & 0.835 & 0.709 & \textbf{0.963*} & 0.923 \\ 
&Left eye & 0.797 & 0.731 & 0.905* & 0.879 & 0.771 & 0.719 & \textbf{0.960*} & 0.918 \\ 
& Nose & 0.829 & 0.736 &0.914* &0.872 & 0.649 & 0.721 & \textbf{0.968*} & 0.916 \\ 
& Head & 0.835 & 0.759 & 0.921* & 0.882 & 0.804 & 0.751 & \textbf{0.964*} & 0.928\\ 
& Neck & 0.902 & 0.793 & 0.929* & 0.871 & 0.745 & 0.765 & \textbf{0.969*} &0.915\\ 
& Right shoulder & 0.887 & 0.725 & 0.924* & 0.890 & 0.751 & 0.758 & \textbf{0.972*} & 0.924\\ 
& Right elbow & 0.764 & 0.250 & 0.806 & 0.914* & 0.048 & 0.602 & \textbf{0.931}* & 0.821\\ 
&Right wrist & 0.806 & 0.501 & 0.823 & 0.903* &0.172 & 0.643 & 
\textbf{0.929*} & 0.819\\ 
& Left shoulder & 0.893 & 0.718 & 0.927* &0.899 & 0.702 & 0.774 & \textbf{0.972*} & 0.930\\ 
&Left elbow& 0.782 & 0.369 & 0.824 & 0.919* & 0.366 & 0.594 & \textbf{0.946*} & 0.839\\ 
&Left wrist & 0.822 & 0.523 & 0.828 & 0.902* & 0.765 & 0.663 & \textbf{0.932*} & 0.824\\ 
&Hip& 0.030 & 0.487 & 0.233 & \textbf{0.910*} & 0.006 & 0.359 & 0.800* & 0.305\\ 
&Right knee& 0.481 & 0.429 & 0.598 & \textbf{0.906*} & 0.186 & 0.067&  0.831* & 0.687\\ 
&Right ankle & 0.357 & 0.275 & 0.534 & \textbf{0.910*} & 0.286 & 0.226 & 0.806* & 0.632\\ 
&Left knee & 0.467 & 0.355 & 0.601 & \textbf{0.899}* & 0.172 & 0.215 & 0.855* & 0.692\\ 
&Left ankle& 0.331 & 0.242 & 0.530 & \textbf{0.904*} & 0.197 & 0.303 & 0.822* & 0.632\\ 
&Tail & 0.231 & 0.196 & 0.434 & \textbf{0.829*} & 0.160 & 0.121 & 0.729* & 0.494\\ 
\midrule
\multirow{17}{*}{\parbox{1.4cm}{Full fine-tuning}}
&Right eye & 0.796 & 0.711 & 0.905* & 0.894 & 0.821&0.694&\textbf{0.959*}&0.927\\ 
&Left eye& 0.790 & 0.734 & 0.904* & 0.882 & 0.763 & 0.714 & \textbf{0.957*} & 0.921\\ 
& Nose & 0.810 & 0.731 & 0.912* & 0.892 & 0.642 & 0.709 & \textbf{0.960*} &0.932\\ 
&Head& 0.801 & 0.737 & 0.900* & 0.892&0.772&0.718&\textbf{0.947*}&0.920\\ 
&Neck&0.893&0.782&0.930*&0.886&0.755&0.743&\textbf{0.962*}&0.926\\ 
&Right shoulder&0.896&0.722&0.936*&0.908&0.759&0.750&\textbf{0.975*}&0.940\\ 
&Right elbow&0.689&0.168&0.736&0.878*&0.047&0.562&\textbf{0.888*}&0.761\\ 
&Right wrist&0.796&0.505&0.805&0.876*&0.199&0.644&\textbf{0.910*}&0.803\\ 
&Left shoulder&0.872&0.690&0.901*&0.882&0.670&0.762&\textbf{0.955*}&0.903\\ 
&Left elbow&0.726&0.282&0.774&0.904*&0.326&0.538&\textbf{0.914*}&0.797\\ 
&Left wrist&0.787&0.488&0.787&0.868*&0.725&0.672&\textbf{0.903*}&0.785\\ 
&Hip&0.016&0.518&0.173&\textbf{0.894*}&0.038&0.382&0.757*&0.238\\ 
&Right knee&0.391&0.518&0.516&\textbf{0.891*}&0.096&0.141&0.763*&0.614\\ 
&Right ankle&0.246&0.396&0.437&\textbf{0.889*}&0.185&0.340&0.726*&0.546\\ 
&Left knee&0.381&0.448&0.521&\textbf{0.891*}&0.149&0.303&0.789*&0.618\\ 
&Left ankle&0.244&0.297&0.444&\textbf{0.871*}&0.098&0.357&0.751*&0.551\\ 
&Tail&0.105&0.299&0.309&\textbf{0.824*}&0.047&0.212&0.628*&0.372\\ 
\bottomrule
\end{tabular}
}
\label{tab:full_all_sources_openmonkey}
\end{table*}

\end{document}